\documentclass[runningheads, envcountsame, a4paper]{llncs}
\usepackage{makeidx}  

\usepackage[square, authoryear, comma, sort&compress]{natbib}  

\usepackage{etoolbox}
\patchcmd{\bibliography}{\chapter*}{\section*}{}{}

\usepackage{glossaries}
\usepackage[utf8]{inputenc} 
\usepackage[T1]{fontenc}    
\usepackage{url}             
\usepackage{booktabs}       
\usepackage{amsfonts}       
\usepackage{nicefrac}       
\usepackage{microtype}      

\usepackage{acronym}
\usepackage{amsmath}
\usepackage{amssymb}
\usepackage{mathbbol}
\usepackage{mathtools}
\usepackage{mathrsfs}
\usepackage{vector}
\usepackage{cleveref}
\usepackage{bm}
\usepackage[dvipsnames]{xcolor}
\usepackage{multirow}
\usepackage{algorithm}
\usepackage{algorithmic}
\newcommand{\argmax}{\operatornamewithlimits{argmax}}

\newacronym{CME}{CME}{conditional mean embedding}
\newacronym{MCE}{MCE}{multiclass conditional embedding}
\newacronym{CEN}{CEN}{conditional embedding network}
\newacronym{RKHS}{RKHS}{reproducing kernel Hilbert space}
\newacronym{SVM}{SVM}{support vector machine}
\newacronym{SVC}{SVC}{support vector classifier}
\newacronym{GP}{GP}{Gaussian process}
\newacronym{GPs}{GPs}{Gaussian processes}
\newacronym{GPC}{GPC}{Gaussian process classifier}
\newacronym{KRR}{KRR}{kernel ridged regressor}
\newacronym{RLSC}{RLSC}{regularized least squares classifier}
\newacronym{ERM}{ERM}{empirical risk minimization}
\newacronym{MMD}{MMD}{maximum mean discrepancy}
\newacronym{RCB}{RCB}{Rademacher complexity bound}
\newacronym{OVA}{OVA}{one versus all}
\newacronym{OVO}{OVO}{one versus one}
\newacronym{CNN}{CNN}{convolutional neural network}
\newacronym{ARD}{ARD}{automatic relevance determination}
\newacronym{HS}{HS}{Hilbert-Schmidt}

\usepackage[pdftex]{graphicx}
\usepackage{epstopdf}
\usepackage[misc]{ifsym}

\title{Hyperparameter Learning \\ for Conditional Kernel Mean Embeddings \\ with Rademacher Complexity Bounds}
\toctitle{Hyperparameter Learning for Conditional Kernel Mean Embeddings with Rademacher Complexity Bounds}

\titlerunning{Kernel Hyperparameter Learning with Rademacher Complexity Bounds}

\author{Kelvin Hsu \inst{1, 2} \and
		Richard Nock \inst{1, 2, 3} \and
		Fabio Ramos \inst{1, 2}}
\authorrunning{K. Hsu, R. Nock, and F. Ramos} 
\tocauthor{Kelvin~Hsu, Richard~Nock, Fabio~Ramos}
\institute{	University of Sydney, Sydney, Australia \\ \and
			Data61, CSIRO, Sydney, Australia \\ \and
			Australian National University, Canberra, Australia}

\begin{document}
	\frontmatter          
	\maketitle \setcounter{footnote}{0}          

	\begin{abstract}
		Conditional kernel mean embeddings are nonparametric models that encode conditional expectations in a reproducing kernel Hilbert space. While they provide a flexible and powerful framework for probabilistic inference, their performance is highly dependent on the choice of kernel and regularization hyperparameters. Nevertheless, current hyperparameter tuning methods predominantly rely on expensive cross validation or heuristics that is not optimized for the inference task. For conditional kernel mean embeddings with categorical targets and arbitrary inputs, we propose a hyperparameter learning framework based on Rademacher complexity bounds to prevent overfitting by balancing data fit against model complexity. Our approach only requires batch updates, allowing scalable kernel hyperparameter tuning without invoking kernel approximations. Experiments demonstrate that our learning framework outperforms competing methods, and can be further extended to incorporate and learn deep neural network weights to improve generalization.\footnote{Source code available at: \texttt{https://github.com/Kelvin-Hsu/cake}}
		
	\end{abstract}
	
	\section{Introduction}
	\label{sec:introduction}
	
		\Glspl{CME} are attractive because they encode conditional expectations in a \gls{RKHS}, bypassing the need for a parametrized distribution \citep{song2013kernel}. They are part of a broader class of techniques known as kernel mean embeddings, where nonparametric probabilistic inference can be carried out entirely within the \gls{RKHS} because difficult marginalization integrals become simple linear algebra \citep{muandet2016kernel}. This very general framework is core to modern kernel probabilistic methods, including kernel two-sample testing \citep{gretton2007kernel}, kernel Bayesian inference  \citep{fukumizu2013kernel}, density estimation \citep{song2008tailoring, kanagawa2014recovering}, component analysis \citep{muandet2013domain}, dimensionality reduction \citep{fukumizu2004dimensionality}, feature discovery \citep{jitkrittum2016interpretable}, and state space filtering \citep{kanagawa2016filtering}.
		
		Nevertheless, like most kernel based models, their performance is highly dependent on the hyperparameters chosen. For these models, the model selection process usually begins by selecting a kernel, whose parameters become part of the model \textit{hyperparameters}, which may further include noise or regularization hyperparameters. Given a set of hyperparameters, training is performed by solving either a convex optimization problem, such as the case in \glspl{SVM} \citep{scholkopf2002learning}, or a set of linear equations, such as the case in \gls{GPs} \citep{rasmussen2006gaussian}, \glspl{RLSC} \citep{rifkin2003regularized}, and \glspl{CME}. Unfortunately, hyperparameter tuning is not straight forward, and often cross validation \citep{song2013kernel} or median length heuristics \citep{muandet2016kernel} remain as the primary approaches for this task. The former can be computationally expensive and sensitive to the selection and number of validation sets, while the latter heuristic only applies to hyperparameters with a length scale interpretation and makes no reference to the conditional inference problem involved as it does not make use of the targets.
		
		One notable success story in this domain are \gls{GPs}, which employ their marginal likelihood as an objective for hyperparameter learning. The marginal likelihood arises from its Bayesian formulation, and exhibits certain desirable properties -- in particular, the ability to automatically balance between data fit and model complexity. On the other hand, \glspl{CME} are not necessarily Bayesian, and hence they do not benefit from a natural marginal likelihood formulation, yet such a balance is critical when generalizing the model beyond known examples.
		
		Can we formulate a learning objective for \glspl{CME} to balance data fit and model complexity, similar to the marginal likelihood of \gls{GPs}? For \glspl{CME} with categorical targets and arbitrary input, we present such a learning objective as our main contribution. In particular, we: (1) derive a data-dependent model complexity measure $r(\theta, \lambda)$ for a \gls{CME} with hyperparameters $(\theta, \lambda)$ based on the Rademacher complexity of a relevant class of \glspl{CME}, (2) propose a novel learning objective based on this complexity measure to control generalization risk by balancing data fit against model complexity, and (3) design a scalable hyperparameter learning algorithm under this objective using stochastic batch gradient updates. We show that this learning objective produces \glspl{CME} that generalize better than that learned from cross validation, \gls{ERM}, and median length heuristics on standard benchmarks, and apply such an algorithm to incorporate and learn neural network weights to improve generalization accuracy.
	
	\section{Background and Related Work}
	\label{sec:background}
	
		\subsection{Conditional Mean Embeddings}
	
			To construct a conditional mean embedding operator $\mathcal{U}_{Y | X}$ corresponding to the distribution $\mathbb{P}_{Y | X}$, where $X : \Omega \to \mathcal{X}$ and $Y: \Omega \to \mathcal{Y}$ are measurable random variables, we first choose a kernel $k : \mathcal{X} \times \mathcal{X} \to \mathbb{R}$ for the input space $\mathcal{X}$ and another kernel $l : \mathcal{Y} \times \mathcal{Y} \to \mathbb{R}$ for the output space $\mathcal{Y}$. These kernels $k$ and $l$ each describe how similarity is measured within their respective domains $\mathcal{X}$ and $\mathcal{Y}$, and are symmetric positive definite such that they uniquely define the \gls{RKHS} $\mathcal{H}_{k}$ and $\mathcal{H}_{l}$. The conditional mean embedding operator $\mathcal{U}_{Y | X}$ is then the operator $\mathcal{U} : \mathcal{H}_{k} \to \mathcal{H}_{l}$ for which $\mu_{Y | X = x} = \mathcal{U} k(x, \cdot)$, where $\mu_{Y | X = x} := \mathbb{E}[l(Y, \cdot) | X = x]$ is the \gls{CME} \citep{song2009hilbert}. In this sense, it sweeps out a family of conditional mean embeddings $\mu_{Y | X = x}$ in $\mathcal{H}_{l}$, each indexed by the input variable $x \in \mathcal{X}$. We then define cross covariance operators $C_{YX} := \mathbb{E}[l(Y, \cdot) \otimes k(X, \cdot)] : \mathcal{H}_{k} \to \mathcal{H}_{l}$ and $C_{XX} := \mathbb{E}[k(X, \cdot) \otimes k(X, \cdot)] : \mathcal{H}_{k} \to \mathcal{H}_{k}$. Alternatively, they can be seen as elements within the tensor product space $C_{YX} \in \mathcal{H}_{l} \otimes \mathcal{H}_{k}$ and $C_{XX} \in \mathcal{H}_{k} \otimes \mathcal{H}_{k}$.
		
			Under the assumption that $k(x, \cdot) \in \mathrm{image}(C_{XX})$, it can be shown that $\mathcal{U}_{Y | X} = C_{YX} C_{XX}^{-1}$. While this assumption is satisfied for finite domains $\mathcal{X}$ with a characteristic kernel $k$, it does not necessarily hold when $\mathcal{X}$ is a continuous domain \citep{fukumizu2004dimensionality}, which is the case for many classification problems. In this case, $C_{YX} C_{XX}^{-1}$ becomes only an approximation to $\mathcal{U}_{Y | X}$, and we instead regularize the inversion and use $\mathcal{U}_{Y | X} = C_{YX} (C_{XX} + \lambda I)^{-1}$, which also serves to avoid overfitting \citep{song2013kernel}. \glspl{CME} are useful for probabilistic inference since conditional expectations of a function $g \in \mathcal{H}_{l}$ can be expressed as inner products with the \gls{CME}, $\mathbb{E}[g(Y) | X = x] = \langle \mu_{Y | X = x}, g \rangle$, provided that $\mathbb{E}[g(Y) | X = \cdot] \in \mathcal{H}_{k}$ \citep[Theorem 4]{song2009hilbert}.
		
			Furthermore, as both $C_{YX}$ and $C_{XX}$ are defined via expectations, we can estimate them with their respective empirical means to derive a nonparametric estimate for $\mathcal{U}_{Y | X}$ based on finite collection of observations $\{x_{i}, y_{i}\} \in \mathcal{X} \times \mathcal{Y}$, $i \in \mathbb{N}_{n} := \{1, \dots, n\}$,
			\begin{equation}
				\hat{\mathcal{U}}_{Y | X} = \Psi (K + n \lambda I)^{-1} \Phi^{T},
			\label{eq:empirical_conditional_embedding}
			\end{equation}
			where $K_{ij} := k(x_{i}, x_{j})$, $\Phi := \begin{bmatrix} \phi(x_{1}) & \dots & \phi(x_{n}) \end{bmatrix}$, $\Psi := \begin{bmatrix} \psi(y_{1}) & \dots & \psi(y_{n}) \end{bmatrix}$, $\phi(x) := k(x, \cdot)$, and $\psi(y) := l(y, \cdot)$ \citep{song2013kernel}. The empirical \gls{CME} defined by $\hat{\mu}_{Y | X = x} := \hat{\mathcal{U}}_{Y | X} k(x, \cdot)$ then stochastically converges to the \gls{CME} $\mu_{Y | X = x}$ in the RKHS norm at a rate of $O_{p}((n \lambda)^{-\frac{1}{2}} + \lambda^{\frac{1}{2}})$, under the assumption that $k(x, \cdot) \in \mathrm{image}(C_{XX})$ \cite[Theorem 6]{song2009hilbert}. This allows us to approximate the conditional expectation with $\langle \hat{\mu}_{Y | X = x}, g \rangle$ instead, 
			\begin{equation}
				\mathbb{E}[g(Y) | X = x] \approx \langle \hat{\mu}_{Y | X = x}, g \rangle = \bvec{g}^{T} (K + n \lambda I)^{-1} \bvec{k}(x),
			\label{eq:empirical_conditional_expectation}
			\end{equation}
			where $\bvec{g} := \{g(y_{i})\}_{i = 1}^{n}$ and $\bvec{k}(x) := \{k(x_{i}, x)\}_{i = 1}^{n}$.
	
		\subsection{Hyperparameter Learning}
	
			Hyperparameter learning for \glspl{CME} is particularly difficult compared to marginal or joint embeddings, since the kernel $k = k_{\theta}$ with hyperparameters $\theta \in \Theta$ is to be learned jointly with a regularization hyperparameter $\lambda \in \Lambda = \mathbb{R}_{+}$. \cite{grunewalder2012conditional} proposed to hold out a validation set $\{k(x_{t_{j}}, \cdot), l(y_{t_{j}}, \cdot)\}_{j = 1}^{J}$ and minimize $\frac{1}{J} \sum_{j = 1}^{J} \big\| l(y_{t_{j}}, \cdot) - \hat{\mathcal{U}}_{Y | X} k(x_{t_{j}}, \cdot) \big\|_{\mathcal{H}_{l}}^{2}$ where $\hat{\mathcal{U}}_{Y | X}$ is estimated from the remaining training set using \eqref{eq:empirical_conditional_embedding}. This could also be repeated over multiple folds for cross validation. \citet[p. 15]{song2013kernel} also uses this cross validation approach, but adds regularization $\lambda \| \mathcal{U} \|^{2}_{HS}$ to the validation objective. Validation sets are necessary for improving generalization to unseen examples. This is because the \gls{CME} is already the solution that minimizes the objective from \cite{grunewalder2012conditional} over the operator space, so further optimization over the hyperparmeters using the same training set would lead to overfitting. Moreoever, the cross validation objective changes depending on the particular split and number of folds. Additionally, by fitting a separate model for each fold during learning, they incur a large computational cost of $O(J n^{3})$ for $J$ folds, and become prohibitive with large datasets. This spells a need for an alternative hyperparameter learning framework using a different objective.
	
			When cross validation is too expensive, length scales can be set by the median heuristic \citep{muandet2016kernel} via $\ell = \mathrm{median}_{i, j}(\| x_{i} - x_{j} \|_{2})$ for many stationary kernels. However, they cannot be used to set hyperparameters other than length scales, such as $\lambda$. In the setting of two sample testing, \cite{gretton2012optimal} note that they can possibly lead to poor performance. In the context of \glspl{CME}, they are also unable to leverage supervision from labels. \cite{flaxman2016bayesian} proposed a Bayesian learning framework for marginal mean embeddings via inducing points, although it is unclear how this can be extended to \glspl{CME}. \cite{fukumizu2009kernel} also investigated the choice of kernel bandwidth for stationary kernels in the setting of binary classification and two sample testing using \gls{MMD}, but has yet to generalize to \glspl{CME} or multiclass settings.
	
		\subsection{Rademacher Complexity}
	
			Rademacher complexity \citep{bartlett2002rademacher} measures the expressiveness of a function class $F$ by its ability to shatter, or fit, noise. They are data-dependent measures, and are thus particularly well suited to learning tasks where generalization is vital, since complexity penalties that are not data dependent cannot be universally effective \citep{kearns1997experimental}. The Rademacher complexity \citep[Definition 2]{bartlett2002rademacher} of a function class $F$ is defined by $\mathcal{R}_{n}(F) := \mathbb{E}[\sup_{f \in F} \| \frac{2}{n} \sum_{i = 1}^{n} \sigma_{i} f(X_{i}) \|]$, where $\{\sigma_{i}\}_{i = 1}^{n}$ are \textit{iid} Rademacher random variables, taking values in $\{-1, 1\}$ with equal probability, and $\{X_{i}\}_{i = 1}^{n}$ are \textit{iid} random variables from the same distribution $\mathbb{P}_{X}$. Since $\{\sigma_{i}\}_{i = 1}^{n}$ are distributed independently without knowledge of $f$, the intuition is to interpret $\{\sigma_{i}\}_{i = 1}^{n}$ as labels that are simply noise. For a given set of inputs $\{X_{i}\}_{i = 1}^{n}$, the term inside the norm is high when the sign of $f(X_{i})$ matches the signs of $\sigma_{i}$ averaged across $i \in \mathbb{N}_{n}$, meaning that $f$ has managed to fit the noise well. We take this as the defining feature of what it means for a model $f$ to be complex. The suprenum then finds the $f$ within $F$ that fits the noise the best, intuitively representing the most complex $f$ within $F$. The final expectation then averages this quantity across realizations of $\{X_{i}\}_{i = 1}^{n}$ from $\mathbb{P}_{X}$.
	
			Rademacher complexities are usually applied in the context where classifiers are trained by minimizing some empirical loss within a class of classifiers whose Rademacher complexity is bounded. In the context of multi-label learning, \cite{yu2014large} used trace norm regularization to bound the Rademacher complexity, achieving tight generalization bounds. \cite{xu2016local} extends the trace norm regularization approach by considering the local Rademacher complexity on a subset of the predictor class, where they instead minimize the tail sum of the predictor singular values. Local Rademacher complexity has also been employed for multiple kernel learning \citep{kloft2011local, cortes2013learning} to learn convex combinations of fixed kernels for \glspl{SVM}. Similarly, \cite{pontil2013excess} also used trace norm regularization to bound the Rademacher complexity and minimize the truncated hinge loss. Nevertheless, while Rademacher complexities have been employed to restrict the function class considered for training weight parameters, they have not been applied to learn kernel hyperparameters itself.

	\section{Multiclass Conditional Embeddings}
	\label{sec:multiclass_conditional_embedding}
	
		In this section we present a particular type of \glspl{CME} that are suitable for prediction tasks with categorical targets. We show that for \glspl{CME} with categorical targets and arbitrary inputs, we can further infer conditional probabilities directly, and not just conditional expectations. As there can be more than two target categories, we refer to these \glspl{CME} as \glspl{MCE}.
	
		For categorical targets, the output label space is finite and discrete, taking values only in $\mathcal{Y} = \mathbb{N}_{m} := \{1, \dots, m\}$. Naturally, we choose the Kronecker delta kernel $\delta : \mathbb{N}_{m} \times \mathbb{N}_{m} \to \{0, 1\}$ as the output kernel $l$, where labels that are the same have unit similarity and labels that are different have no similarity. That is, for all pairs of labels $y_{i}, y_{j} \in \mathcal{Y}$, $\delta(y_{i}, y_{j}) = 1$ only if $y_{i} = y_{j}$ and is $0$ otherwise. As $\delta$ is an integrally strictly positive definite kernel on $\mathbb{N}_{m}$, it is therefore characteristic \citep[Theorem 7]{sriperumbudur2010hilbert}. Therefore, by definition \citep{fukumizu2004dimensionality}, $\delta$ uniquely defines a \gls{RKHS}  $\mathcal{H}_{\delta} = \overline{\mathrm{span}\{\delta(y, \cdot) : y \in \mathcal{Y}\}}$, which is the closure of the span of its kernel induced features \citep{xu2009refinement}. For $\mathcal{Y} = \mathbb{N}_{m}$, this means that any $g : \mathbb{N}_{m} \to \mathbb{R}$ that is bounded on its discrete domain $\mathbb{N}_{m}$ is in the \gls{RKHS} of $\delta$, because we can always write $g = \sum_{y = 1}^{m} g(y) \delta(y, \cdot) \in \mathrm{span}\{\delta(y, \cdot) : y \in \mathcal{Y}\} \subseteq \mathcal{H}_{\delta}$. In particular, indicator functions on $\mathbb{N}_{m}$ are in $\mathcal{H}_{\delta}$, since $\mathbb{1}_{c}(y) := \mathbb{1}_{\{c\}}(y) = \delta(c, y)$, so that $\mathbb{1}_{c} = \delta(c, \cdot)$ are simply the canonical features of $\mathcal{H}_{\delta}$. Such properties do not necessarily hold for continuous target domains in general. For discrete target domains, this convenient property enables consistent estimations of decision probabilities.
	
		Let $p_{c}(x) := \mathbb{P}[Y = c | X = x]$ be the \textit{decision probability function} for class $c \in \mathbb{N}_{m}$, which is the probability of the class label $Y$ being $c$ when the example $X$ is $x$. Importantly, note that there are no restrictions on the input domain $\mathcal{X}$ as long as a kernel $k$ can be defined on it. For example, $\mathcal{X}$ could be the continuous Euclidean space $\mathbb{R}^{d}$, the space of images, or the space of strings. We begin by writing this probability as an expectation of indicator functions,
		\begin{equation}
			p_{c}(x) := \mathbb{P}[Y = c | X = x] = \mathbb{E}[\mathbb{1}_{c}(Y) | X = x].
		\label{eq:decision_probability}
		\end{equation}
		
		With $\mathbb{1}_{c} \in \mathcal{H}_{\delta}$, we let $g = \mathbb{1}_{c}$ in \eqref{eq:empirical_conditional_expectation} and $\bvec{1}_{c} := \{\mathbb{1}_{c}(y_{i})\}_{i = 1}^{n}$ to estimate the right hand side of \eqref{eq:decision_probability} by
		\begin{equation}
			\hat{p}_{c}(x) = f_{c}(x) := \bvec{1}_{c}^{T} (K + n \lambda I)^{-1} \bvec{k}(x).
		\label{eq:empirical_decision_probability}
		\end{equation}
		
		Let $\bvec{Y} := \begin{bmatrix} \bvec{1}_{1} & \bvec{1}_{2} & \cdots & \bvec{1}_{m} \end{bmatrix} \in \{0, 1\}^{n \times m}$ be the one hot encoded labels of $\{y_{i}\}_{i = 1}^{n}$. The vector of empirical decision probabilities over the classes $c \in \mathbb{N}_{m}$ is then
		\begin{equation}
			\hat{\bvec{p}}(x) = \bvec{f}(x) := \bvec{Y}^{T} (K + n \lambda I)^{-1} \bvec{k}(x) \in \mathbb{R}^{m}.
		\label{eq:empirical_decision_probability_vector}
		\end{equation}
		
		Since $\mathcal{U} = \hat{\mathcal{U}}_{Y | X}$ \eqref{eq:empirical_conditional_embedding} is the solution to a regularized least squares problem in the RKHS from $k(x, \cdot) \in \mathcal{H}_{k}$ to $l(y, \cdot) \in \mathcal{H}_{l}$ \citep{grunewalder2012conditional}, \glspl{CME} are essentially \glspl{KRR} with targets in the \gls{RKHS}. In this case, because $\mathcal{Y} = \mathbb{N}_{m}$ is discrete, $\mathcal{H}_{\delta}$ can be identified with $\mathbb{R}^{m}$. As a result, the rows of the \gls{MCE} can also be seen as $m$ \glspl{KRR} \citep{friedman2001elements} on binary $\{0, 1\}$-targets, where they all share the same input kernel $k$. Because they all share the same kernel to form the \gls{MCE}, we can show that the empirical decision probabilities \eqref{eq:empirical_decision_probability} do converge to the population decision probability.
	
		\begin{theorem}[Convergence of Empirical Decision Probability Function]
			\label{thm:probability_convergence_copy}
			Assuming that $k(x, \cdot)$ is in the image of $C_{XX}$, the empirical decision probability function $\hat{p}_{c} : \mathcal{X} \to \mathbb{R}$ \eqref{eq:empirical_decision_probability} converges uniformly to the true decision probability $p_{c} : \mathcal{X} \to [0, 1]$ \eqref{eq:decision_probability} at a stochastic rate of at least $O_{p}((n \lambda)^{-\frac{1}{2}} + \lambda^{\frac{1}{2}})$ for all $c \in \mathcal{Y} = \mathbb{N}_{m}$. See appendix for proof, including for all subsequent theorems.
		\end{theorem}
		
		In particular, the assumption $k(x, \cdot) \in \mathrm{image}(C_{XX})$ is a statement on the input kernel $k$, not the output kernel $l$, which is a Kronecker delta $l = \delta$ for \glspl{MCE}. It is worthwhile to note that this assumption is common for \glspl{CME}, and is not as restrictive as it may first appear, as it can be relaxed through introducing the regularization hyperparameter $\lambda$ \eqref{eq:empirical_conditional_embedding} in practice \citep[p.74-75, Sec. 3 and 3.1 \textit{resp.}]{song2009hilbert, song2013kernel, muandet2016kernel}.
		
		Note that for finite $n$ the probability estimates \eqref{eq:empirical_decision_probability} may not necessarily lie in the range $[0,1]$ nor form a normalized distribution for finite $n$. Nonetheless, \cref{thm:probability_convergence_copy} guarantees that they approach one with increasing sample size. When normalized distributions are required, clip-normalized estimates can be used,
		\begin{equation}
			\tilde{p}_{c}(x) := \frac{\max\{\hat{p}_{c}(x), 0\}}{\sum_{j = 1}^{m} \max\{\hat{p}_{j}(x), 0\}}.
		\label{eq:empirical_decision_probability_clip_normalized}
		\end{equation}
		
		This does not change the resulting prediction, since $\hat{y}(x) = \argmax_{c \in \mathbb{N}_{m}} \hat{p}_{c}(x) = \argmax_{c \in \mathbb{N}_{m}} \tilde{p}_{c}(x)$. \Cref{thm:probability_convergence_copy} also implies that eventually the effect of clip-normalization vanishes, where $\tilde{p}_{c}(x)$ approaches to both $\hat{p}_{c}(x)$ and thus $p_{c}(x)$ with increasing sample sizes.
		
		Importantly, this enables \glspl{MCE} to be naturally applied to perform probabilistic classification in multiclass settings with categorical targets. In contrast, in terms of probabilistic classification, \glspl{SVC} do not output probabilities and probabilistic extensions require difficult calibration, while \glspl{GPC} require posterior approximations. Furthermore, in terms of the multiclass setting, multiclass extensions to \glspl{SVC} and \glspl{GPC} often employ the \gls{OVA} or \gls{OVO} scheme \citep{aly2005survey}, resulting in multiple separately trained binary classifiers with no guarantees of coherence between their outputs. Instead, training a single \gls{MCE} is sufficient for producing consistent multiclass probabilistic estimates.
		
		Similar to \gls{RLSC}, \glspl{MCE} are solutions to a regularized least squares problem in a \gls{RKHS} \citep{grunewalder2012conditional}, resulting in a similar system of linear equations. Nevertheless, \glspl{RLSC} primarily differ in the way they handle the labels, in which binary labels $\{-1, 1\}$ appear directly in the squared loss instead of its kernel feature $\delta(y_{i}, \cdot)$ or, equivalently, its one hot encoded form $\bvec{y}_{i}$. Consequently, multiclass extensions for \gls{RLSC} either require using the \gls{OVA} scheme \citep{rifkin2003regularized} which suffers from computational and coherence issues, or alternatively minimize the total loss across all binarized tasks for the overall least squares problem \citep{pahikkala2012unsupervised}. Although the latter attempts to link the classifiers together through its loss, both approaches still produce separate classifiers for each class. As a result, multiclass \gls{RLSC} does not produce consistent estimates of class probabilities as in \cref{thm:probability_convergence_copy} for \glspl{MCE}.
	
	\section{Hyperparameter Learning with Rademacher Complexity Bounds}
	\label{sec:hyperparameter_learning}
	
		In this section we derive learning theoretic bounds that motivate our proposed hyperparameter learning algorithm, and discuss how it can be extended in various ways to enhance scalability and performance. From here onwards, we denote $\theta$ as the kernel hyperparameters of the kernel $k = k_{\theta}$.
		
		We begin by defining a loss function as a measure for performance. For decision functions of the form $\bvec{f} : \mathcal{X} \to \mathcal{A} = \mathbb{R}^{m}$ whose entries are probability estimates, we employ a modified cross entropy loss,
		\begin{equation}
			\mathcal{L}_{\epsilon}(y, \bvec{f}(x)) := - \log{ [\bvec{y}^{T} \bvec{f}(x)]_{\epsilon}^{1} } = - \log{ [f_{y}(x)]_{\epsilon}^{1} },
		\label{eq:cross_entropy_loss_copy}
		\end{equation}
		to express risk, where we use the notation $[\;\cdot\;]_{\epsilon}^{1} := \min\{\max\{\;\cdot\;, \epsilon\}, 1\}$ for $\epsilon \in (0, 1)$. It is worthwhile to point out that this choice only makes sense due to \cref{thm:probability_convergence_copy}, as it allows us to interpret the outputs of the \gls{CME} as asymptotic probability estimates. Note that we employ the loss on the original probability estimates \eqref{eq:empirical_decision_probability_vector}, not the clip-normalized version \eqref{eq:empirical_decision_probability_clip_normalized}. We employ this loss in virtue of \cref{thm:probability_convergence_copy}, where we expect $\bvec{f}(x)$ \eqref{eq:empirical_decision_probability_vector} to be approximations to the population decision probabilities. In contrast, direct outputs from \glspl{SVC}, \glspl{GPC}, or \glspl{RLSC} are not consistent probability estimates and cannot take advantage of \eqref{eq:cross_entropy_loss_copy} easily.
		
		However, simply minimizing the empirical loss $\frac{1}{n} \sum_{i = 1}^{n} \mathcal{L}_{\epsilon}(y_{i}, \bvec{f}_{\theta, \lambda}(x_{i}))$ over the hyperparameters $(\theta, \lambda)$ could lead to an overfitted model. We therefore employ Rademacher complexity bounds to control the model complexity of \glspl{MCE}.
	
		Let $\Theta$ and $\Lambda$ be a space of kernel and regularization hyperparameters respectively. We define the class of \glspl{MCE} over these hyperparameter spaces by	
		\begin{equation}
			F_{n}(\Theta, \Lambda) := \{ \bvec{f}_{\theta, \lambda}(x) : \theta \in \Theta, \lambda \in \Lambda \}.
		\label{eq:predictor_class_copy}
		\end{equation}
		We denote $W_{\theta, \lambda}^{T} \equiv \hat{\mathcal{U}}^{(\theta, \lambda)}_{Y | X}$ so that $\| W_{\theta, \lambda} \|_{\mathrm{tr}} = \| \hat{\mathcal{U}}^{(\theta, \lambda)}_{Y | X} \|_{HS}$ to reflect the dependence on $(\theta, \lambda)$ and also to emphasize the role it plays as the weights of the decision function. We first restrict the space of hyperparameters by the norms of $W_{\theta, \lambda}$ and $k_{\theta}(x, x)$ to obtain an upper bound to the Rademacher complexity of $F_{n}(\Theta, \Lambda)$.
	
		\begin{theorem}[\gls{MCE} Rademacher Complexity Bound]
			\label{thm:rademacher_complexity_bound_copy}
			Suppose that the trace norm $\| W_{\theta, \lambda} \|_{\mathrm{tr}} \leq \rho$ is bounded for all $\theta \in \Theta, \lambda \in \Lambda$. Further suppose that the canonical feature map is bounded in \gls{RKHS} norm $\| \phi_{\theta}(x) \|_{\mathcal{H}_{k_{\theta}}}^{2} = k_{\theta}(x, x) \leq \alpha^{2}$, $\alpha > 0$, for all $x \in \mathcal{X}, \theta \in \Theta$. For any set of training observations $\{x_{i}, y_{i}\}_{i = 1}^{n}$, the Rademacher complexity of the class of \glspl{MCE} $F_{n}(\Theta, \Lambda)$ \eqref{eq:predictor_class_copy} is bounded by
			\begin{equation}
			\mathcal{R}_{n}(F_{n}(\Theta, \Lambda)) \leq 2 \alpha \rho.
			\label{eq:rademacher_complexity_bound_copy}
			\end{equation}
		\end{theorem}
	
		\cite{bartlett2002rademacher} showed that the expected risk can be bounded with high probability using the empirical risk and the Rademacher complexity of the loss composed with the function class. For a Lipchitz loss, \citet{ledoux2013probability} further showed that the latter quantity can be bounded using the Rademacher complexity of the function class itself. We use these two results to arrive at the following probabilistic upper bound to our expected loss.

		\begin{theorem}[\gls{MCE} $\epsilon$-Specific Expected Risk Bound]
			\label{thm:specific_expected_loss_bound_for_multiclass_conditional_embedding_copy}
			Assume the same assumptions as \cref{thm:rademacher_complexity_bound_copy}. For any integer $n \in \mathbb{N}_{+}$, any $\epsilon \in (0, e^{-1})$, and any set of training observations $\{x_{i}, y_{i}\}_{i = 1}^{n}$, with probability of at least $1 - \beta$ over \textit{iid} samples $\{X_{i}, Y_{i}\}_{i = 1}^{n}$ of length $n$ from $\mathbb{P}_{X Y}$, every $f \in F_{n}(\Theta, \Lambda)$ satisfies
			\begin{equation}
			\begin{aligned}
			\mathbb{E}[\mathcal{L}_{e^{-1}}(Y, f(X))] \leq \frac{1}{n} \sum_{i = 1}^{n} \mathcal{L}_{\epsilon}(Y_{i}, f(X_{i})) + 4 e \; \alpha \rho + \sqrt{\frac{8}{n} \log{\frac{2}{\beta}}}.
			\end{aligned}
			\label{eq:specific_expected_loss_bound_for_multiclass_conditional_embedding_copy}
			\end{equation}
		\end{theorem}
	
		However, for hyperparameter learning, we would require a risk bound for specific choice of hyperparameters, not just for a set of hyperparameters. For some $\tilde{\theta} \in \Theta$ and $\tilde{\lambda} \in \Lambda$, we construct a subset of hyperparameters $\Xi(\tilde{\theta}, \tilde{\lambda}) \subseteq \Theta \times \Lambda$ defined by $\Xi(\tilde{\theta}, \tilde{\lambda}) := \{ (\theta, \lambda) \in \Theta \times \Lambda : \| W_{\theta, \lambda} \|_{\mathrm{tr}} \leq \| W_{\tilde{\theta}, \tilde{\lambda}} \|_{\mathrm{tr}}, \; \sup_{x \in \mathcal{X}} k_{\theta}(x, x) \leq \alpha^{2}(\tilde{\theta}) := \sup_{x \in \mathcal{X}} k_{\tilde{\theta}}(x, x) \}$. Clearly, this subset is non-empty, since $(\tilde{\theta}, \tilde{\lambda}) \in \Xi(\tilde{\theta}, \tilde{\lambda})$ is itself an element of this subset. Thus, we can assert that $\| W_{\theta, \lambda} \|_{\mathrm{tr}} \leq \rho = \| W_{\tilde{\theta}, \tilde{\lambda}} \|_{\mathrm{tr}}$ is bounded for all $(\theta, \lambda) \in \Xi(\tilde{\theta}, \tilde{\lambda})$, and that $\| \phi_{\theta}(x) \|_{\mathcal{H}_{k_{\theta}}}^{2} = k_{\theta}(x, x) \leq \alpha^{2} = \sup_{x \in \mathcal{X}} k_{\tilde{\theta}}(x, x)$ is bounded for all $x \in \mathcal{X}, (\theta, \lambda) \in \Xi(\tilde{\theta}, \tilde{\lambda})$.
	
		We can now choose some arbitrary $\tilde{\theta} \in \Theta$, $\tilde{\lambda} \in \Lambda$ and apply \cref{thm:specific_expected_loss_bound_for_multiclass_conditional_embedding_copy} with $\rho = \| W_{\tilde{\theta}, \tilde{\lambda}} \|_{\mathrm{tr}}$ and $\alpha^{2} = \sup_{x \in \mathcal{X}} k_{\tilde{\theta}}(x, x)$ and by considering only the hyperparameters $(\theta, \lambda) \in \Xi(\tilde{\theta}, \tilde{\lambda})$. The probabilistic statement \eqref{eq:specific_expected_loss_bound_for_multiclass_conditional_embedding_copy} then only holds for $(\theta, \lambda) \in \Xi(\tilde{\theta}, \tilde{\lambda})$. In particular, since $(\tilde{\theta}, \tilde{\lambda}) \in \Xi(\tilde{\theta}, \tilde{\lambda})$, it holds for $(\theta, \lambda) = (\tilde{\theta}, \tilde{\lambda})$. Applying this choice, the only hyperparameters that remain in the statement are $(\tilde{\theta}, \tilde{\lambda})$. We then replace these symbols with $(\theta, \lambda)$ again to avoid cluttered notation. Since they were chosen arbitrarily from $ \Theta \times \Lambda$, we arrive at our final result.
		
		\begin{theorem}[\gls{MCE} Expected Risk Bound for Hyperparameters]
			\label{thm:expected_risk_bound_hyperparameter_learning_copy}
			For any integer $n \in \mathbb{N}_{+}$ and any set of training observations $\{x_{i}, y_{i}\}_{i = 1}^{n}$ used to define $\bvec{f}_{\theta, \lambda}$ \eqref{eq:empirical_decision_probability_vector}, with probability $1 - \beta$ over \textit{iid} samples $\{X_{i}, Y_{i}\}_{i = 1}^{n}$ of length $n$ from $\mathbb{P}_{X Y}$, every $\theta \in \Theta$, $\lambda \in \Lambda$, and $\epsilon \in (0, e^{-1})$ satisfies
			\begin{equation}
			\begin{aligned}
			\mathbb{E}[\mathcal{L}_{e^{-1}}(Y, \bvec{f}_{\theta, \lambda}(X))] \leq \frac{1}{n} \sum_{i = 1}^{n} \mathcal{L}_{\epsilon}(Y_{i}, \bvec{f}_{\theta, \lambda}(X_{i})) + 4 e \; r(\theta, \lambda) + \sqrt{\frac{8}{n} \log{\frac{2}{\beta}}},
			\label{eq:expected_risk_bound_hyperparameter_learning_copy}
			\end{aligned}
			\end{equation}
			where $r(\theta, \lambda) := \sqrt{\mathrm{trace}(V_{\theta, \lambda}^{T} K_{\theta} V_{\theta, \lambda}) \sup_{x \in \mathcal{X}} k_{\theta}(x, x)}$ and $V_{\theta, \lambda} := (K_{\theta} + n \lambda I)^{-1} \bvec{Y}$.	
		\end{theorem}
		
		In particular, $r(\theta, \lambda)$ is an upper bound to the Rademacher complexity of a relevant class of \glspl{MCE} based on the hyperparameters $\Xi(\theta, \lambda)$. We call $r(\theta, \lambda)$ the \gls{RCB} and use it to measure the model complexity of a \gls{MCE} with hyperparameters $(\theta, \lambda)$. Since the training set itself is a sample of length $n$ drawn from $\mathbb{P}_{X Y}$, the inequality \eqref{eq:expected_risk_bound_hyperparameter_learning_copy} holds with probability $1 - \beta$ when the random variables $(X_{i}, Y_{i})$ are realized as the training observations $(x_{i}, y_{i})$. Motivated by this, we employ this upper bound as the learning objective for hyperparameter learning,
		\begin{equation}
			q(\theta, \lambda) := \frac{1}{n} \sum_{i = 1}^{n} \mathcal{L}_{\epsilon}(y_{i}, \bvec{f}_{\theta, \lambda}(x_{i})) + 4 e \; r(\theta, \lambda).
		\label{eq:learning_objective}
		\end{equation}
	
		Importantly, the first term is an empirical risk that measures data fit, and the second term is the \gls{RCB} that measures model complexity. Together, this learning objective achieves a balance between data fit and model complexity, similar to the corresponding property of a negative log marginal likelihood learning objective.
	
		\begin{algorithm}[tb]
			\caption{\gls{MCE} Hyperparameter Learning with Stochastic Gradient Updates}
			\label{alg:multiclass_conditional_embedding_training}
			\begin{algorithmic}[1]
				\STATE {\bfseries Input:} kernel family $k_{\theta} : \mathcal{X} \times \mathcal{X} \to \mathbb{R}$, dataset $\{x_{i}, y_{i}\}_{i = 1}^{n}$, initial kernel hyperparameters $\theta_{0}$, initial regularization hyperparameters $\lambda_{0}$, learning rate $\eta$, cross entropy loss threshold $\epsilon$, batch size $n_{b}$
				\STATE $\theta \leftarrow \theta_{0}$, $\lambda \leftarrow \lambda_{0}$
				\REPEAT
				\STATE Sample the next batch $\mathcal{I}_{b} \subseteq \mathbb{N}_{n}$ s.t. $| \mathcal{I}_{b} | = n_{b}$ 
				\STATE $Y \leftarrow \{\delta(y_{i}, c) : i \in \mathcal{I}_{b}, c \in \mathbb{N}_{m}\} \hspace{\fill} \in \{0, 1\}^{n_{b} \times m}$
				\STATE $K_{\theta} \leftarrow \{k_{\theta}(x_{i}, x_{j}) : i \in \mathcal{I}_{b}, j \in \mathcal{I}_{b}\} \hspace{\fill} \in \mathbb{R}^{n_{b} \times n_{b}}$
				\STATE $L_{\theta, \lambda} \leftarrow \mathrm{cholesky}(K_{\theta} + n_{b} \lambda I_{n_{b}}) \hspace{\fill} \in \mathbb{R}^{n_{b} \times n_{b}}$
				\STATE $V_{\theta, \lambda} \leftarrow L_{\theta, \lambda}^{T} \backslash (L_{\theta, \lambda} \backslash Y) \hspace{\fill} \in \mathbb{R}^{n_{b} \times m}$
				\STATE $P_{\theta, \lambda} \leftarrow K_{\theta} V_{\theta, \lambda} \hspace{\fill} \in \mathbb{R}^{n_{b} \times m}$
				\STATE $r(\theta, \lambda) \leftarrow \alpha(\theta) \sqrt{\mathrm{trace}(V_{\theta, \lambda}^{T} K_{\theta} V_{\theta, \lambda})}$
				\STATE $q(\theta, \lambda) \leftarrow \frac{1}{n_{b}} \sum_{i = 1}^{n_{b}} \mathcal{L}_{\epsilon}((Y)_{i}, (P_{\theta, \lambda})_{i}) + 4 e \; r(\theta, \lambda)$
				\STATE $(\theta, \lambda) \leftarrow \mathrm{GradientBasedUpdate}(q, \theta, \lambda; \eta)$ 
				\UNTIL{maximum iterations reached} 
				\STATE {\bfseries Output:} kernel hyperparameters $\theta$, regularization $\lambda$
			\end{algorithmic}
		\end{algorithm}
		
		\subsection{Extensions}
		\label{sec:extensions}
		
			\paragraph{Batch Stochastic Gradient Update}
			
				Since \cref{thm:expected_risk_bound_hyperparameter_learning_copy} holds for any $n \in \mathbb{N}_{+}$ and any set of data $\{x_{i}, y_{i}\}_{i = 1}^{n}$ from $\mathbb{P}_{X Y}$, the bound \eqref{eq:expected_risk_bound_hyperparameter_learning_copy} also holds with high probability for a batch subset of the training data. However, the batch size cannot be too small, in order to keep the constant $\sqrt{8 \log{(2 / \beta)} / n}$ relatively small. We therefore propose to use only a random batch subset of the data to perform each gradient update. This enables scalable hyperparameter learning through batch stochastic gradient updates, where each gradient update stochastically improves a different probabilistic upper bound of the generalization risk. Note that without \cref{thm:expected_risk_bound_hyperparameter_learning_copy}, it is not straightforward to simply apply stochastic gradient updates to optimize $q$, since $r$ depends on the dataset but is not written in terms of a summation over the data. We present this scalable hyperparameter learning approach via batch stochastic gradient updates in \cref{alg:multiclass_conditional_embedding_training}, reducing the time complexity from $O(n^{3})$ to $O(n_{b}^{3})$, where $n_{b}$ is the batch size. The Cholesky decomposition for the full training set requires $O(n^{3})$ time and is necessary only for inference, instead of once every learning iteration. It can be further avoided by using random Fourier features \citep{rahimi2008random} or kernel herding \citep{chen2010super} to approximate the already learned \gls{MCE}. All further inference takes $O(n^{2})$ time, or potentially less with approximation, using back substitution. 
	
			\paragraph{Batch Validation}
			
				While we simply instantiated $(X_{i}, Y_{i})$ to be the training observations in \cref{thm:expected_risk_bound_hyperparameter_learning_copy} to obtain \eqref{eq:learning_objective}, this does not have to be the case for batch updates. Instead, in each learning iteration, we could further split the batch into two sub-batches -- one for training and one for validation. The training batch is used to form the \gls{MCE} $\bvec{f}_{\theta, \lambda}$ and \gls{RCB} $r(\theta, \lambda)$, while we evaluate the empirical risk on the validation batch,
				\begin{equation}
					q^{(V)}(\theta, \lambda) := \frac{1}{n_{(V)}} \sum_{i = 1}^{n_{(V)}} \mathcal{L}_{\epsilon}(y^{(V)}_{i}, \bvec{f}^{(T)}_{\theta, \lambda}(x^{(V)}_{i})) + \tau \; r^{(T)}(\theta, \lambda),
				\label{eq:validation_learning_objective}
				\end{equation}
				where $(T)$ and $(V)$ denotes training and validation. Importantly, in contrast to standard cross validation, not all data is required for each update due to the presence of the \gls{RCB}. Furthermore, although the multiplier on the \gls{RCB} is $4 e$, experiments show that generalization performance can improve if we use a smaller multiplier $\tau < 4 e$, suggesting an upper bound tighter than \eqref{eq:expected_risk_bound_hyperparameter_learning_copy} may exist. In practice, these two extensions work well together. Intuitively, by introducing a validation batch to measure empirical data fit, a smaller weight on the complexity penalty is required. 
	
			\paragraph{Conditional Embedding Network}
			
				Our learning algorithm does not restrict the way the kernel $k_{\theta}$ is constructed from its hyperparameters $\theta \in \Theta$. One particularly useful type of \glspl{MCE} are those where the input kernel $k_{\theta}(x, x') = \langle \varphi_{\theta}(x), \varphi_{\theta}(x') \rangle$ is constructed from neural networks $\varphi_{\theta} : \mathcal{X} \to \mathbb{R}^{p}$ explicitly. We refer to \glspl{MCE} constructed this way as \glspl{CEN}. In these cases, the weights and biases of the neural network become the kernel hyperparameters $\theta$ of the \glspl{CEN}. We can therefore learn network weights and biases jointly under \eqref{eq:learning_objective}. \glspl{CEN} can also scale easily, since the $n \times n$ Cholesky decomposition required for full gradient updates in \cref{alg:multiclass_conditional_embedding_training} can be transformed into a $p \times p$ decomposition by the Woodbury matrix inversion identity \citep{higham2002accuracy}, reducing the time complexity to $O(p^{3} + np^{2})$. This allows scalable learning for $n >> p$ even without using batch gradient updates. For inference, standard map reduce methods can be used. We direct the reader to appendix C for detailed discussion on the various \gls{MCE} architectures and their implementation as compared to \cref{alg:multiclass_conditional_embedding_training}. 
				
	\section{Experiments}
	\label{sec:experiments}
	
		\paragraph{Toy Example}
	
			\begin{figure*}[t]
				\centering
				\includegraphics[width=0.32\linewidth]{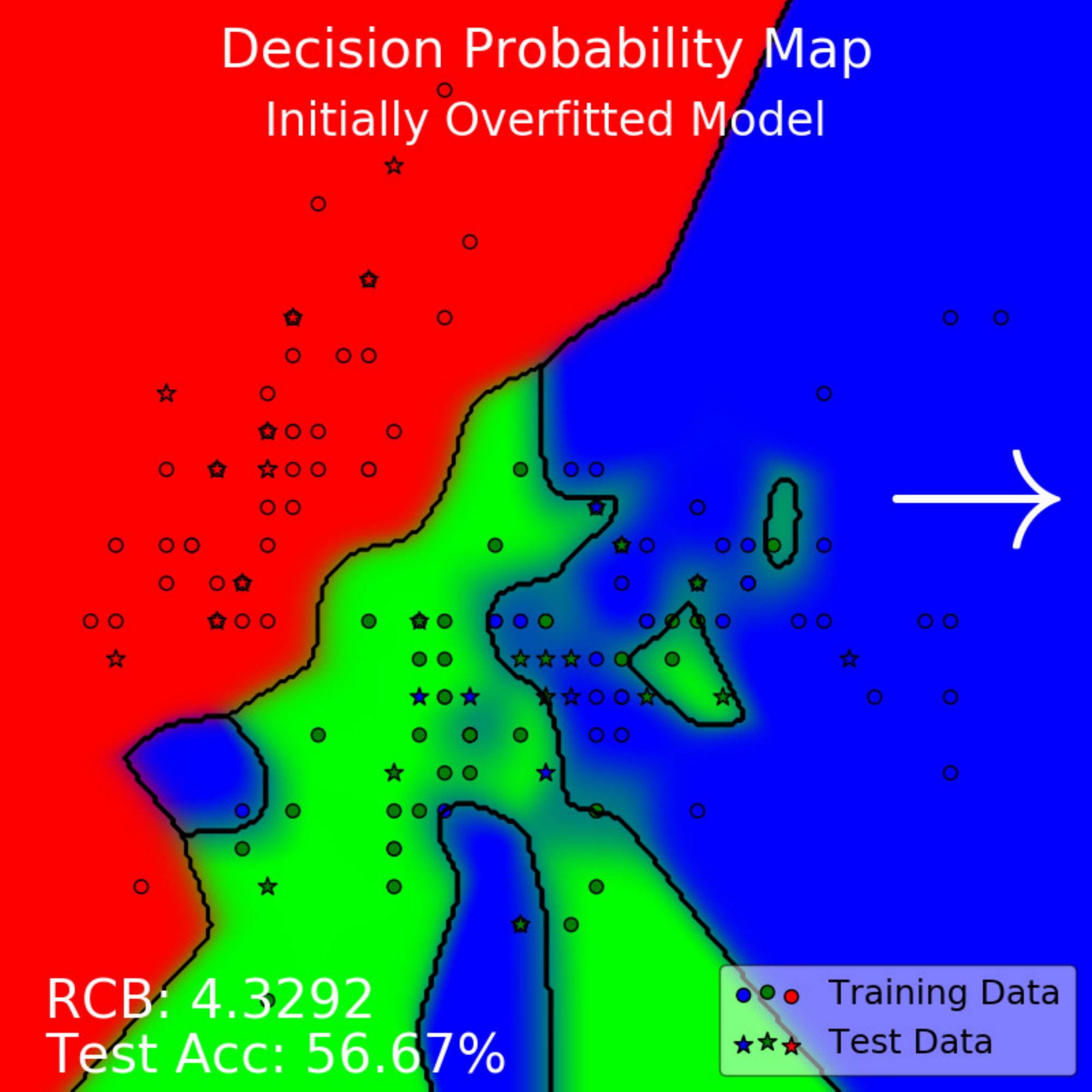}
				\includegraphics[width=0.32\linewidth]{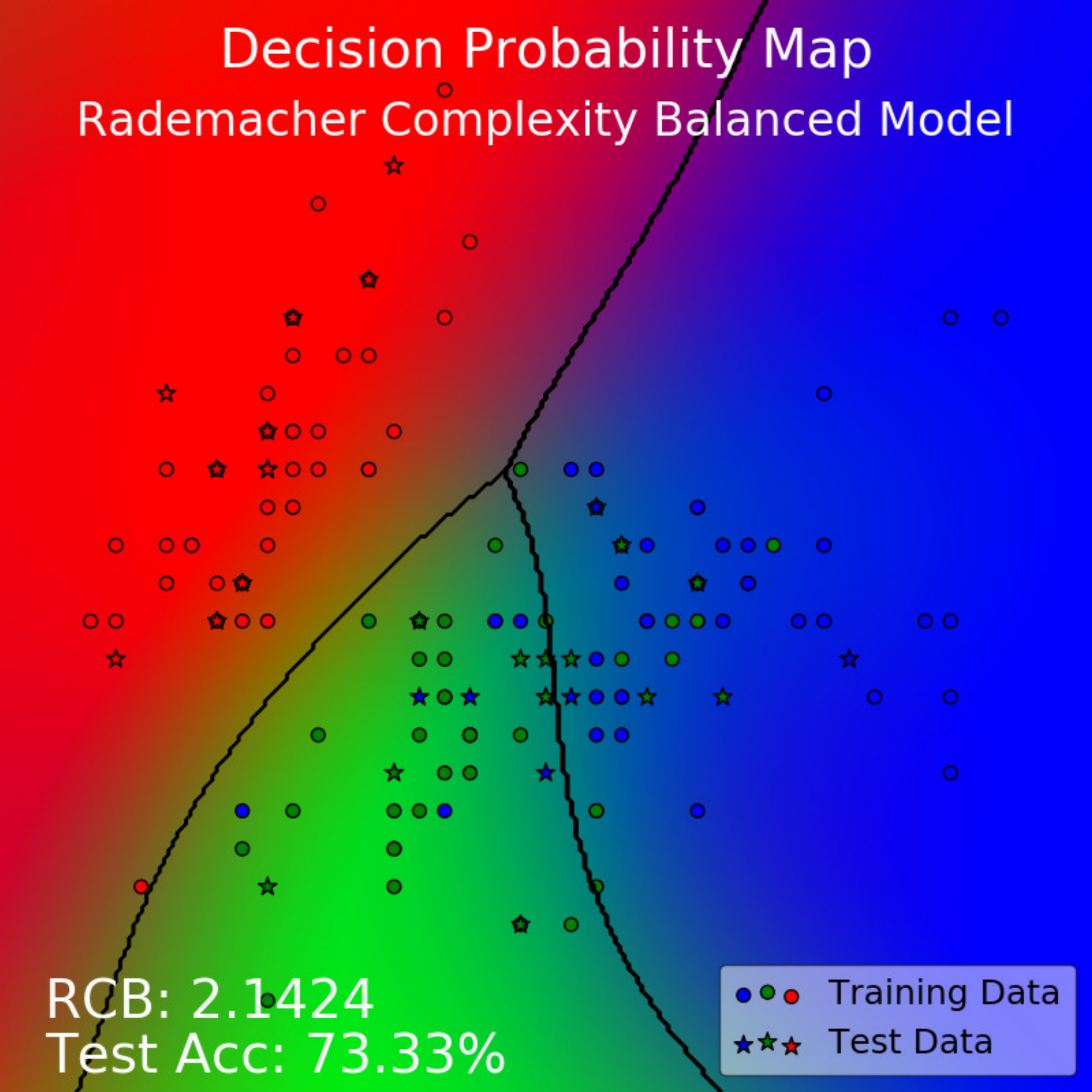}
				\includegraphics[width=0.32\linewidth]{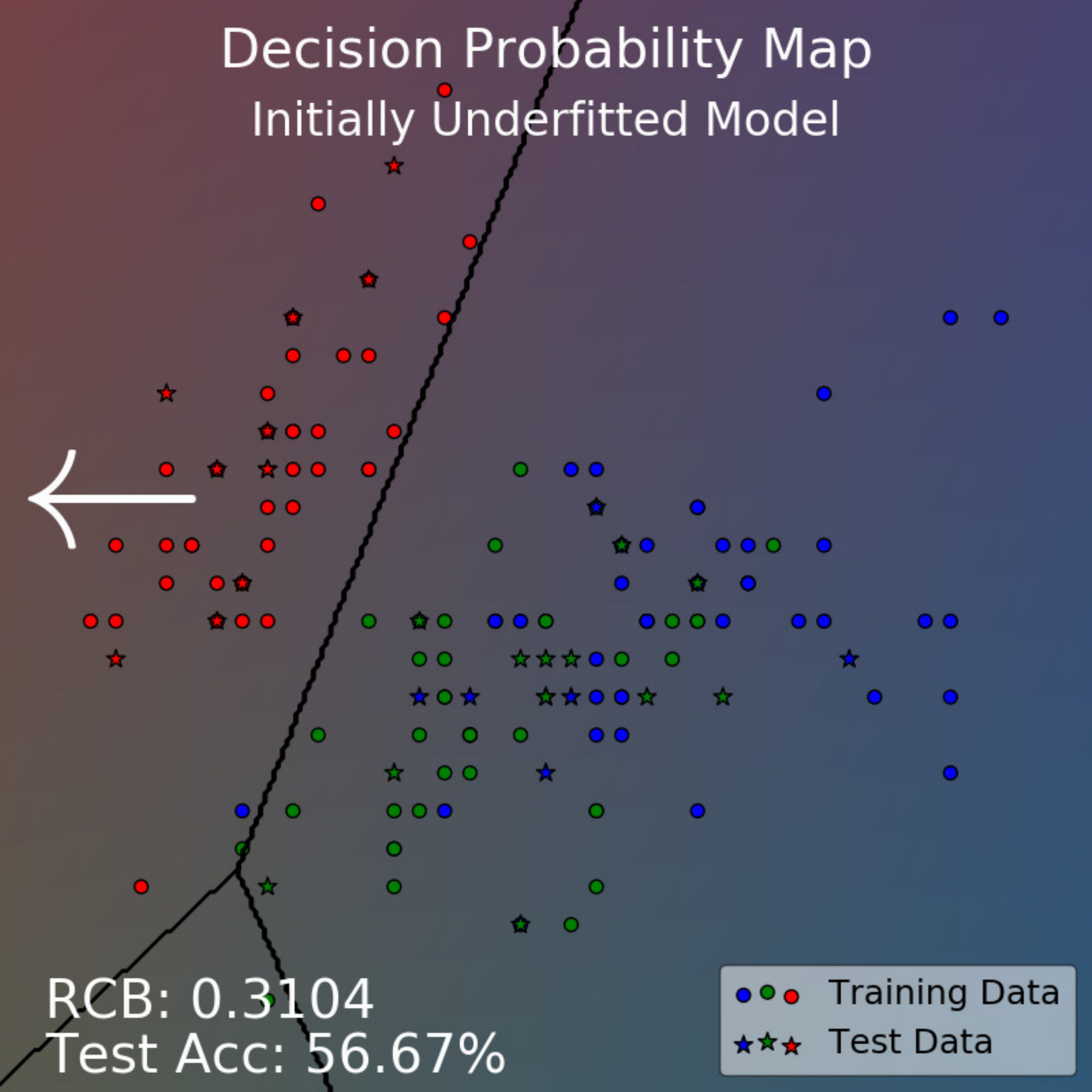}
				\includegraphics[width=0.48\linewidth]{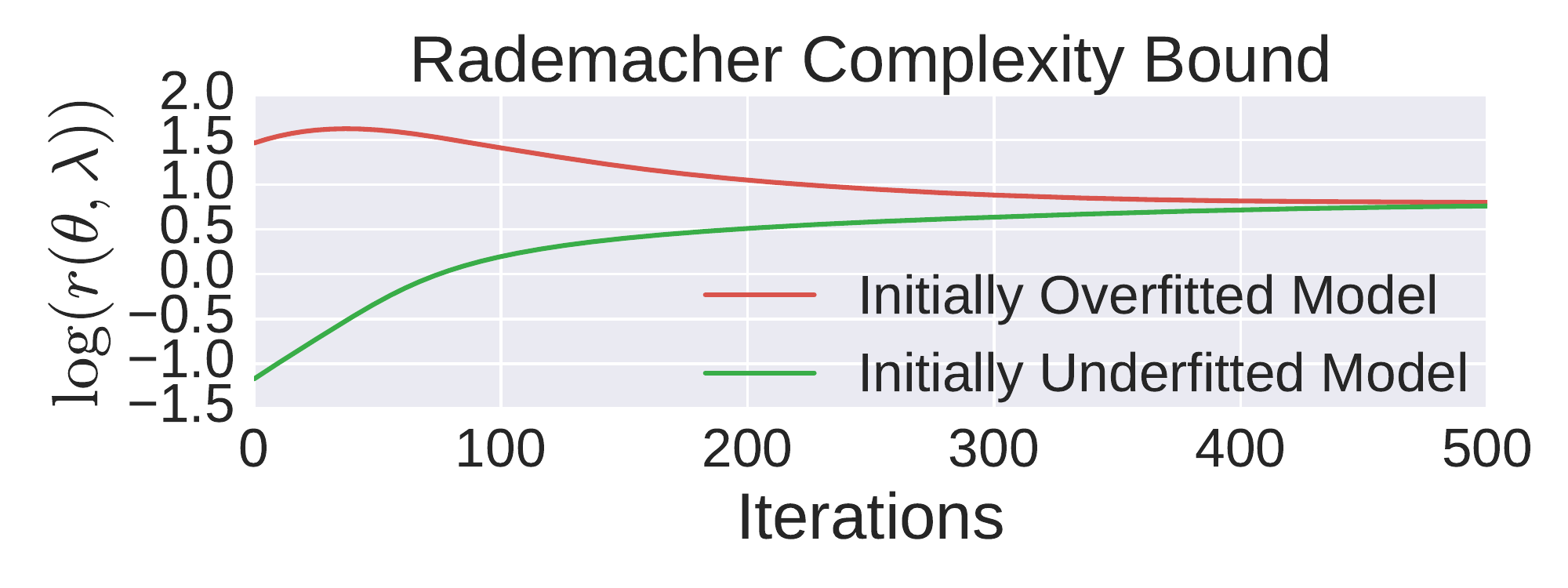}
				\includegraphics[width=0.48\linewidth]{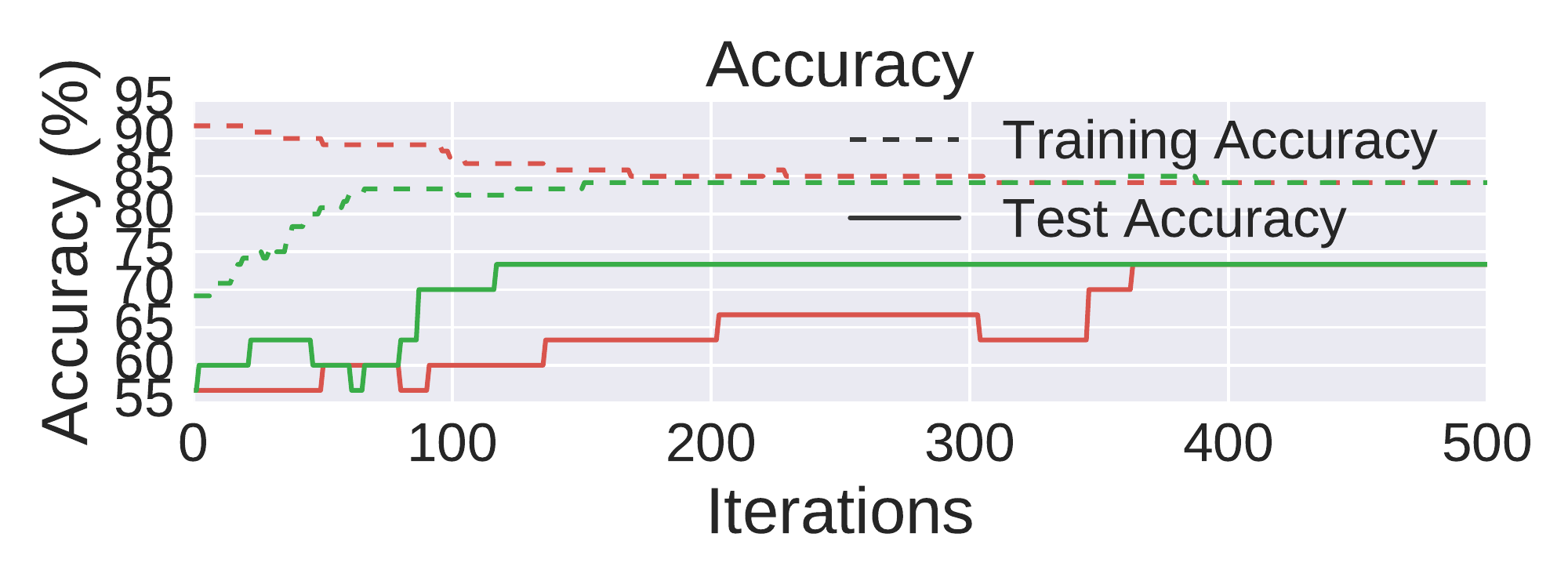}
				\caption{Rademacher complexity balanced learning of hyperparameters for an isotropic Gaussian \gls{MCE} using the first two attributes of the iris dataset}
				\label{fig:iris}
			\end{figure*}
				
			The first two of four total attributes of the iris dataset \citep{fisher1936use} are known to have class labels that are non-separable by any means, in that the same example $x \in \mathbb{R}^{2}$ may be assigned different output labels $y \in \mathbb{N}_{3} := \{1, 2, 3\}$. In these difficult scenarios, the notion of model complexity is extremely important, and the success of a learning algorithm greatly depends on how it balances training performance and model complexity to avoid both underfitting and overfitting. 
			
			\Cref{fig:iris} demonstrates \cref{alg:multiclass_conditional_embedding_training} with full gradient updates ($n_{b} = n$) to learn hyperparameters of the \gls{MCE} on the two attribute iris dataset. The kernel used is isotropic Gaussian with diagonal length scales $\Sigma = \ell^{2} I_{2}$ and sensitivity $\alpha = \sigma_{f}$, so that the hyperparameters are $\theta = (\alpha, \ell)$ and $\lambda$. We evaluate the performance of the learning algorithm on a withheld test set using 20\% of the available 150 data samples. Attributes are scaled into the unit range $[0, 1]$ and decision probability maps are plotted for the region $[-0.5, 1.05]^{2}$, where the red, green, and blue color channels represent the clip-normalized decision probability \eqref{eq:empirical_decision_probability_clip_normalized} for classes $c = 1, 2, 3$. We begin from two initial sets of hyperparameters, one originally overfitting and another underfitting the training data. Initially, both models perform sub-optimally with a test accuracy of 56.67\%. We see that the \gls{RCB} $r(\theta, \lambda)$ appropriately measures the amount of overfitting with high (resp. low) values for the overfitted (resp. underfitted) model. We then learn hyperparameters with \cref{alg:multiclass_conditional_embedding_training} for 500 iterations from both initializations at rate $\eta = 0.01$, where both models converges to a balanced model with a moderate \gls{RCB} and an improved test accuracy of 73.33\%. In particular, the initially overfitted model learns a simpler model at the expense of lower training performance, emphasizing the benefits of complexity based regularization, without which the learning would only maximize training performance at the cost of further overfitting. Meanwhile, the initially underfitted model learns to increase complexity to improve the sub-optimal performance on the training set.
	
		\paragraph{UCI Datasets}

			\begin{table*}[t]
				\caption{Test accuracy (\%) on UCI datasets}
				\label{tab:uci_experiments}
				\centering
				\begin{tabular}{lcccccc}	
					Method         & \texttt{banknote}       & \texttt{ecoli}          & \texttt{robot}          & \texttt{segment}        & \texttt{wine}           & \texttt{yeast}          \\
					\midrule
					G\gls{MCE}     & $\mathbf{99.9 \pm 0.2}$ & $\mathbf{87.5 \pm 4.4}$ & $\mathbf{96.7 \pm 0.9}$ & $\mathbf{98.4 \pm 0.8}$ & $\mathbf{97.2 \pm 3.7}$ & $52.5 \pm 2.1$          \\
					G\gls{MCE}-SGD & $98.8 \pm 0.9$          & $84.5 \pm 5.0$          & $95.5 \pm 0.9$          & $96.1 \pm 1.5$          & $93.3 \pm 6.0$          & $\mathbf{60.3 \pm 4.4}$ \\
					\gls{CEN}-1    & $99.5 \pm 1.0$          & $\mathbf{87.5 \pm 3.2}$ & $82.3 \pm 7.1$          & $94.6 \pm 1.6$          & $96.1 \pm 5.0$          & $55.8 \pm 5.0$          \\
					\gls{CEN}-2    & $99.4 \pm 0.9$          & $86.3 \pm 6.0$          & $94.5 \pm 0.8$          & $96.7 \pm 1.1$          & $97.2 \pm 5.1$          & $59.6 \pm 4.0$          \\
					ERM            & $\mathbf{99.9 \pm 0.2}$ & $72.1 \pm 20.5$         & $91.0 \pm 3.7$          & $98.1 \pm 1.1$          & $93.9 \pm 5.2$          & $45.9 \pm 6.4$          \\
					CV             & $\mathbf{99.9 \pm 0.2}$ & $73.8 \pm 23.8$         & $90.9 \pm 3.4$          & $98.3 \pm 1.3$          & $93.3 \pm 7.4$          & $58.0 \pm 5.8$          \\
					MED            & $92.0 \pm 4.3$          & $42.1 \pm 47.7$         & $81.1 \pm 6.2$          & $27.3 \pm 26.4$         & $93.3 \pm 7.8$          & $31.2 \pm 14.1$         \\
					Others         & 99.78\textsuperscript{a}& 81.1\textsuperscript{b} & 97.59\textsuperscript{c}& 96.83\textsuperscript{d}& 100\textsuperscript{e}  & 55.0\textsuperscript{b}
				\end{tabular}
			\end{table*}
		
			We demonstrate the average performance of learning anisotropic Gaussian kernels and kernels constructed from neural networks on standard UCI datasets \citep{bache2013uci}, summarized in \cref{tab:uci_experiments}. The former has a shallow but wide model architecture, while the latter has a deeper but narrower model architecture. The Gaussian kernel is learned with both full (G\gls{MCE}) and batch stochastic gradient updates (G\gls{MCE}-SGD) using a tenth ($n_{b} \approx \frac{n}{10}$) of the training set each training iteration, with sensitivity and length scales initialized to $1$. For \glspl{CEN}, we randomly select two simple fully connected architectures with 16-32-8 (\gls{CEN}-1) and 96-32 (\gls{CEN}-2) hidden units respectively, and learn the conditional mean embedding without dropout under ReLU activation. Biases and standard deviations of zero mean truncated normal distributed weights are initialized to $0.1$, and are to be learned with full gradient updates. For all experiments, $\lambda$ is initialized to $1$ and is learned jointly with the kernel. Optimization is performed with the Adam optimizer \citep{kingma2014adam} in TensorFlow \citep{abadi2016tensorflow} with a rate of $\eta = 0.1$ and $\epsilon = 10^{-15}$ under the learning objective $q(\theta, \lambda)$ \eqref{eq:learning_objective}. Learning is run for 1000 epochs to allow direct comparison. All attributes are scaled to the unit range. Each model is trained on 9 out of 10 folds and tested on the remaining fold, which are shuffled over all 10 combinations to obtain the test accuracy average and deviation. We compare our results to \glspl{MCE} whose hyperparameters are tuned by \gls{ERM} (without the \gls{RCB} term in \eqref{eq:learning_objective}), cross validation (CV), and the median heuristic (MED), as well as to other approaches using neural networks \citep[a; c]{kaya2016banknote, freire2009short}, probabilistic binary trees \citep[b]{horton1996probabilistic}, decision trees \citep[d]{zhou2004size}, and regularized discriminant analysis \citep[e]{aeberhard1992comparison}. 
	
			\Cref{tab:uci_experiments} shows that our learning algorithm outperforms other hyperparameter tuning algorithms, and performs similarly to competing methods. Our method achieves this without any case specific tuning or heuristics, but by simply placing a conditional mean embedding on training data and applying a complexity bound based learning algorithm. The stochastic gradient approach for Gaussian kernels performs similarly to the full gradient approach, supporting the claim of \cref{thm:expected_risk_bound_hyperparameter_learning_copy} for $n = n_{b}$. For \glspl{CEN}, we did not attempt to choose an optimal architecture for each dataset. The learning algorithm is tasked to train the same simple network for different datasets using 1000 epochs to achieve comparable performance. 
	
		\paragraph{Learning pixel relevance}
	
			We apply \cref{alg:multiclass_conditional_embedding_training} to learn length scales of anisotropic Gaussian, or \gls{ARD}, kernels on pixels of the MNIST digits dataset \citep{lecun1998gradient}. In the top left plot of \cref{fig:mnist_experiments}, we train on datasets of varying sizes, from 50 to 5000 images, and show the accuracy on the standard test set of 10000 images. All hyperparameters are initialized to 1 before learning. We train both \glspl{SVC} and \glspl{GPC} under the \gls{OVA} scheme, and use a Laplace approximation for the \gls{GPC} posterior. In all cases \glspl{MCE} outperform \glspl{SVC} as it cannot learn hyperparameters without expensive cross validation. \glspl{MCE} also outperform \glspl{GPC} as more data becomes available. Under the \gls{OVA} scheme, the \gls{GPC} approach learns a set of kernel hyperparameters for each class, while our approach learns a consistent set of hyperparameters for all classes. Consequently, for 5000 data points, the computational time required for hyperparameter learning of \glspl{GPC} is on the order of days even for isotropic Gaussian kernels, while \cref{alg:multiclass_conditional_embedding_training} is on the order of hours for anistropic Gaussian kernels even without batch updates. We also compare hyperparameter learning with and without the \gls{RCB}. For small $n$ below 750 samples, the latter outperforms the former (e.g. 86.69\% and 86.96\% for $n = 500$), while for large $n$ the former outperforms the latter (e.g. 96.05\% and 95.3\% for $n= 5000$). This verifies that complexity based regularization becomes especially important as data size grows, when overfitting starts to decrease generalization performance. The images at the bottom of \cref{fig:mnist_experiments} show the pixel length scales learned through batch stochastic gradient updates ($n_{b} = 1200$) over all available training images the groups of digits shown, demonstrating the most discriminative regions.
	
		\paragraph{Learning convolutional layers}
	
			\begin{figure*}[t]
				\centering 
				\includegraphics[width=0.49\linewidth]{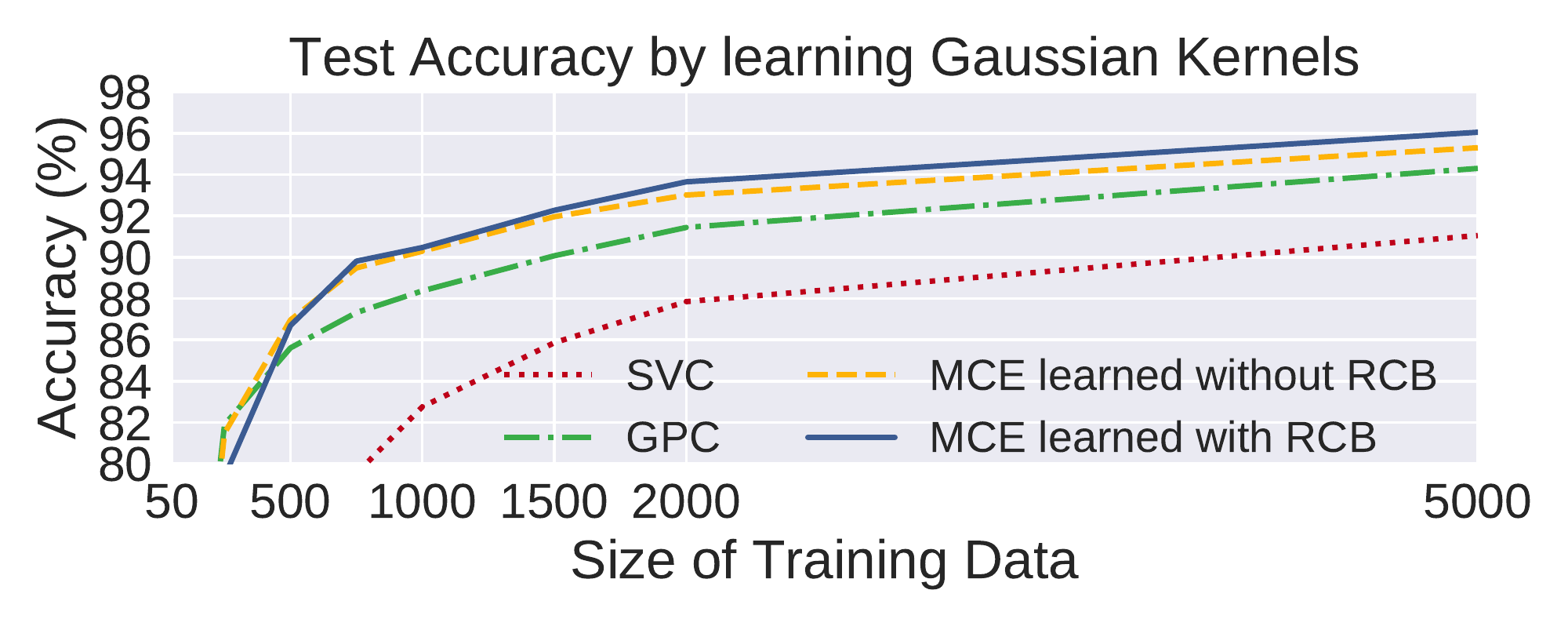}
				\includegraphics[width=0.49\linewidth]{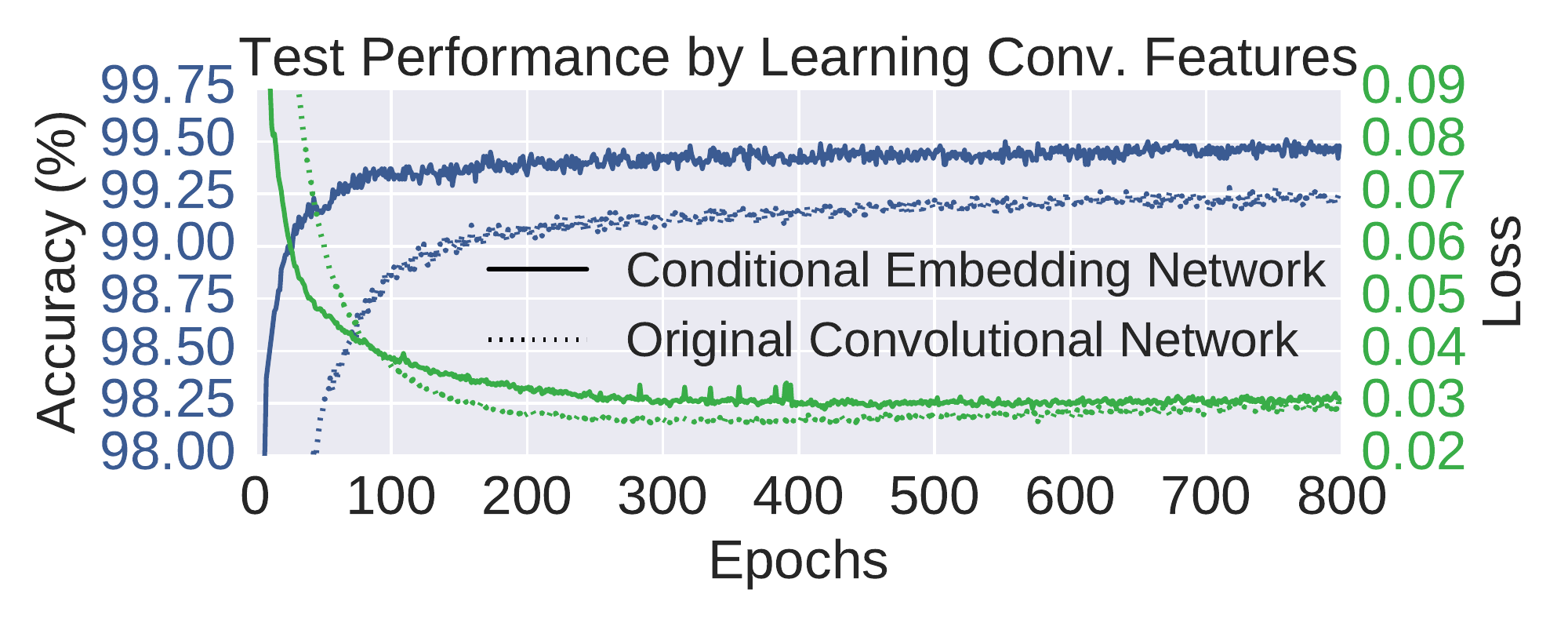}
				\includegraphics[width=0.1\linewidth]{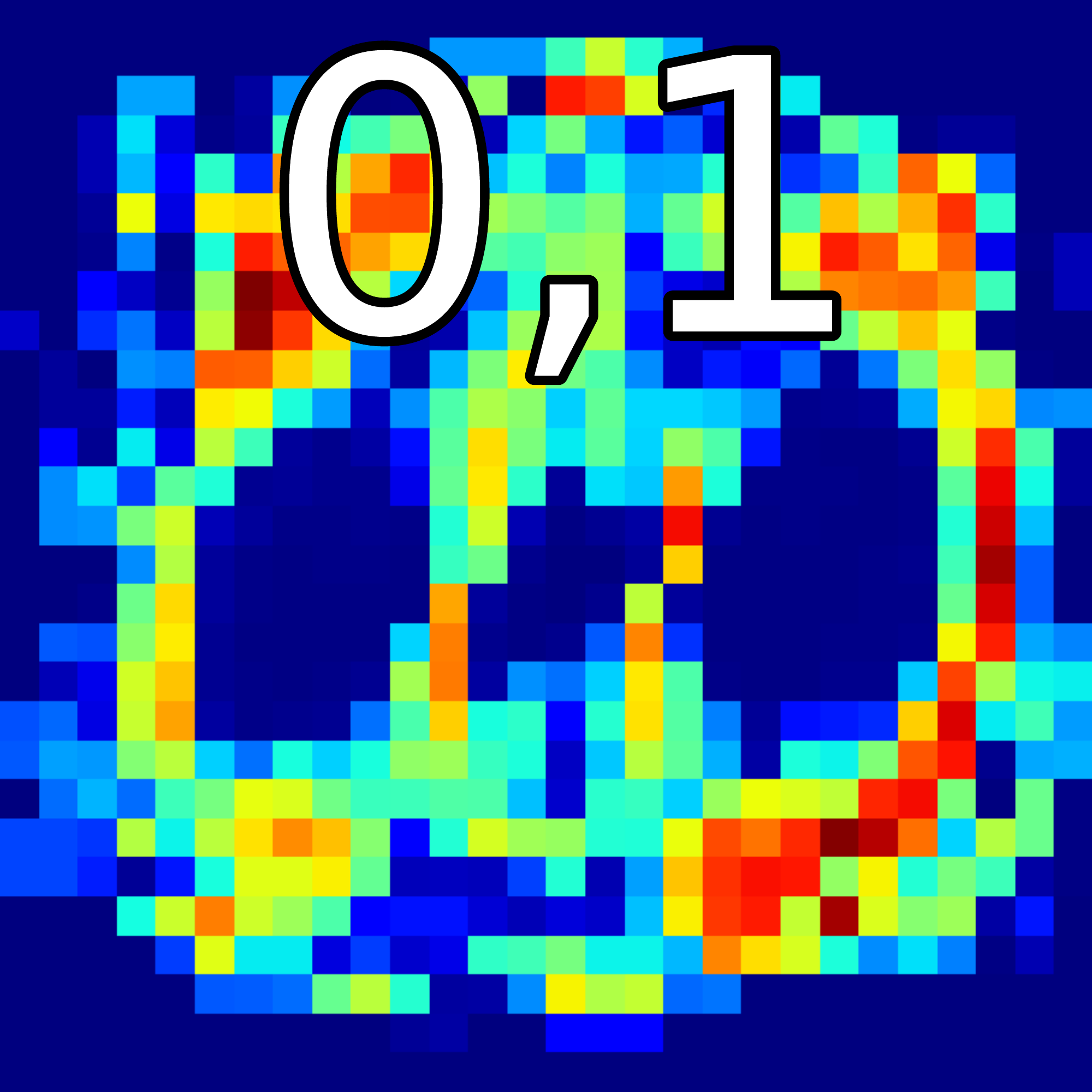}
				\includegraphics[width=0.1\linewidth]{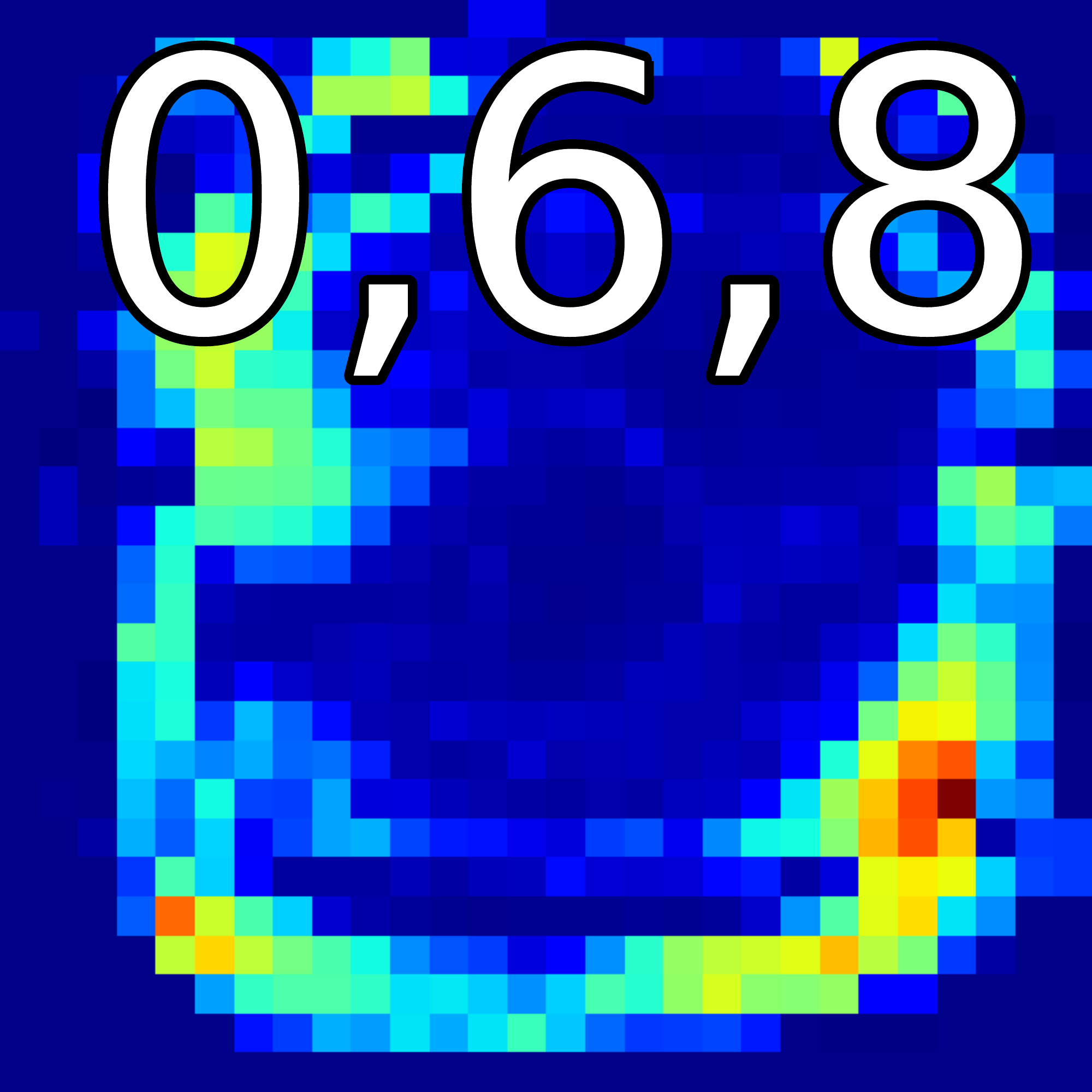}
				\includegraphics[width=0.1\linewidth]{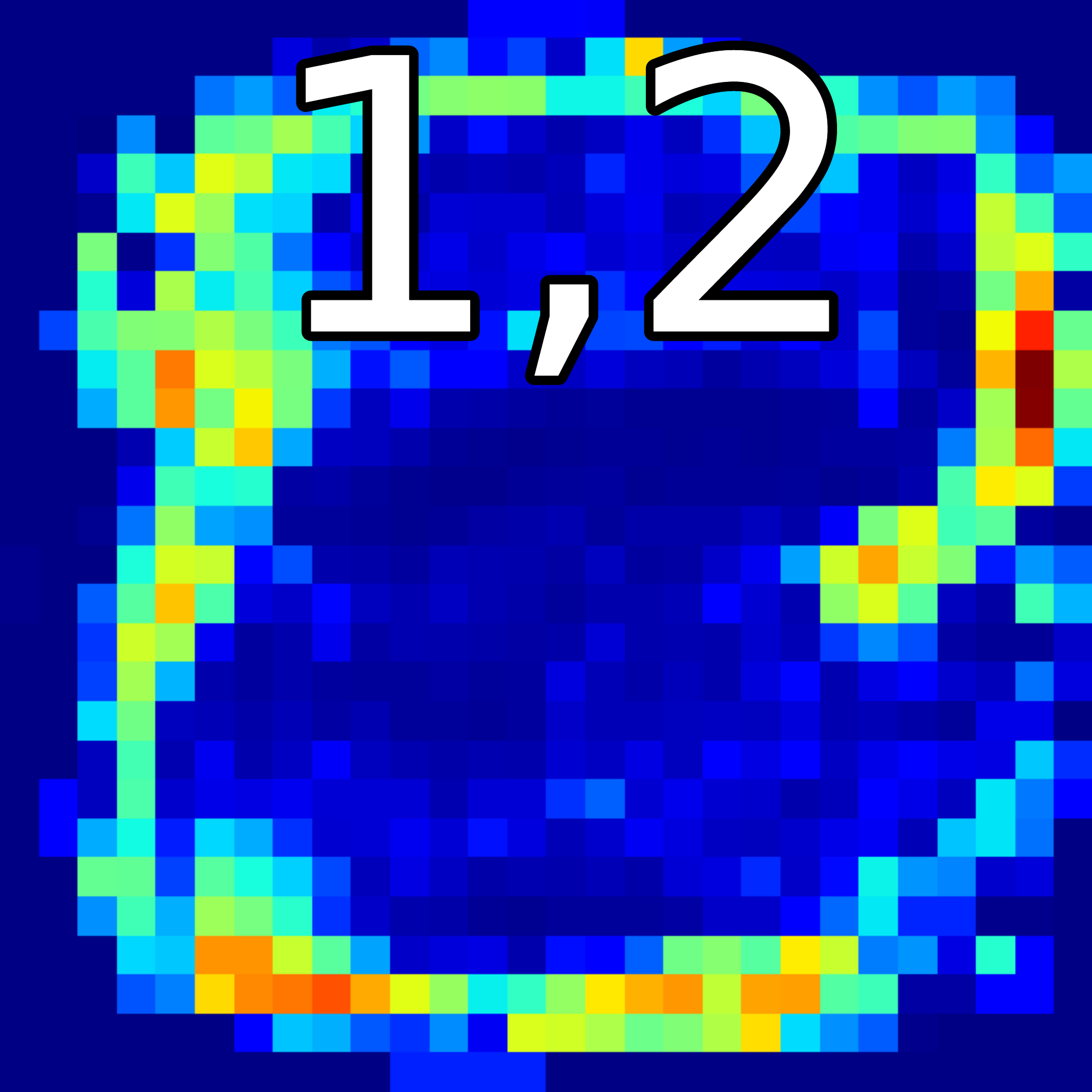}
				\includegraphics[width=0.1\linewidth]{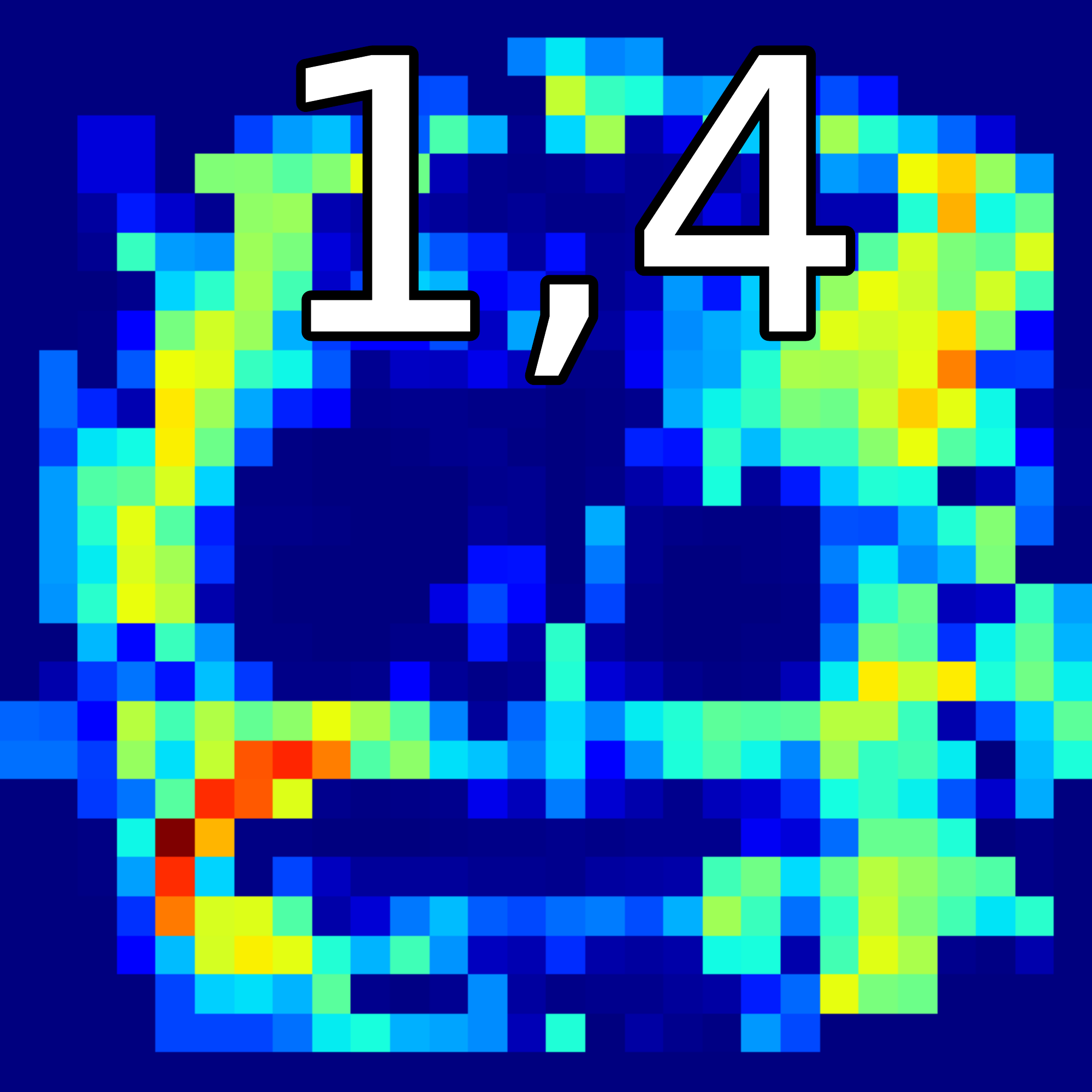}
				\includegraphics[width=0.1\linewidth]{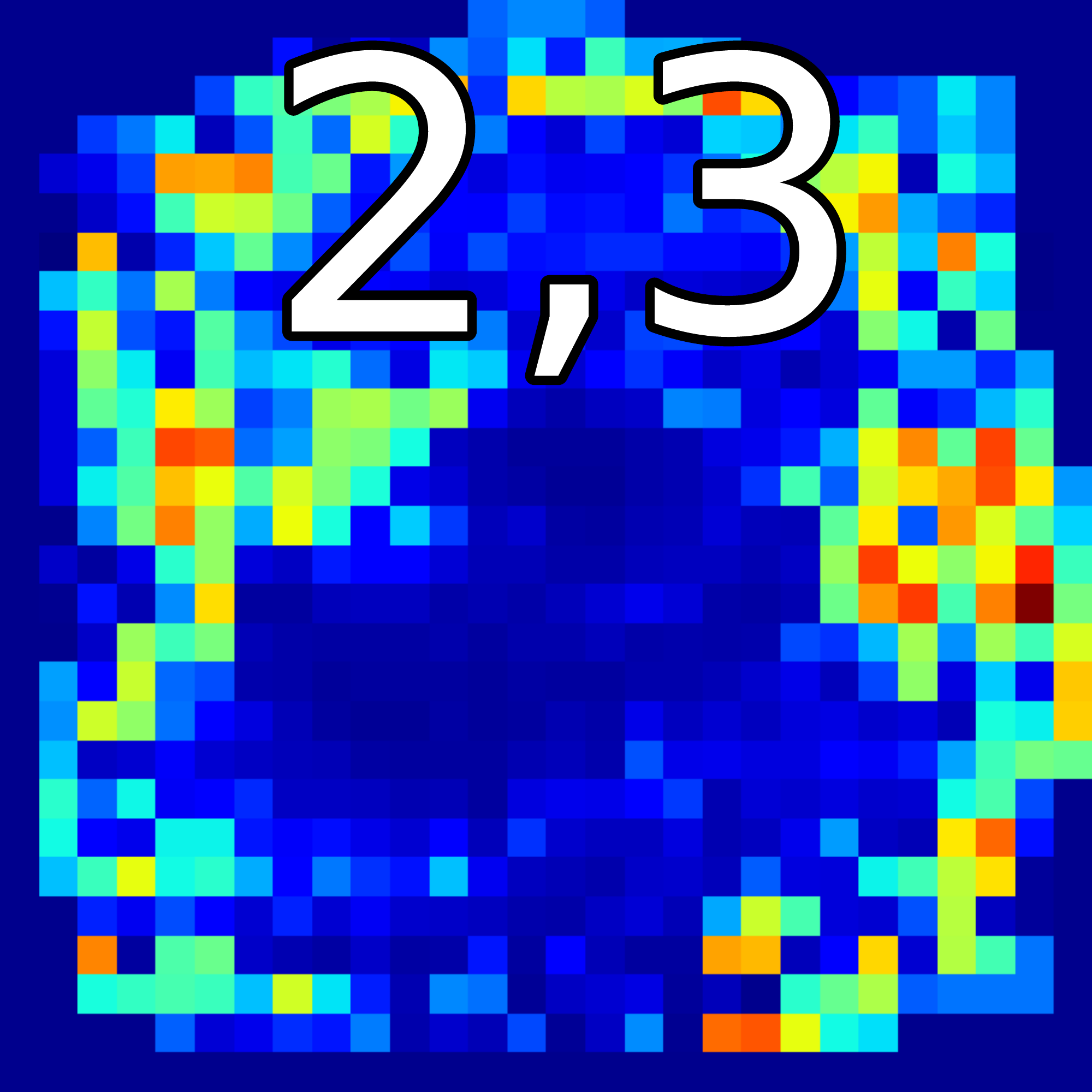}
				\includegraphics[width=0.1\linewidth]{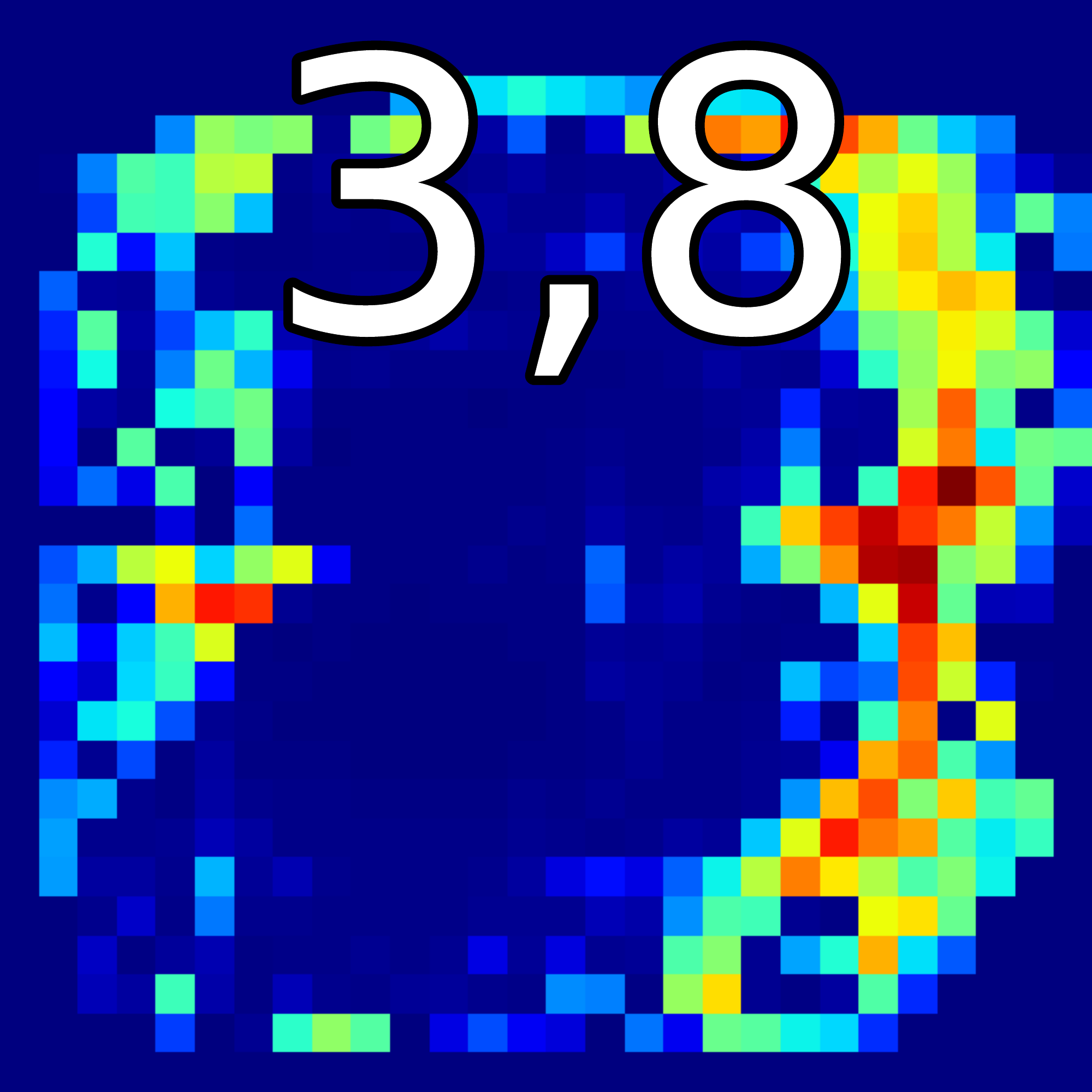}
				\includegraphics[width=0.1\linewidth]{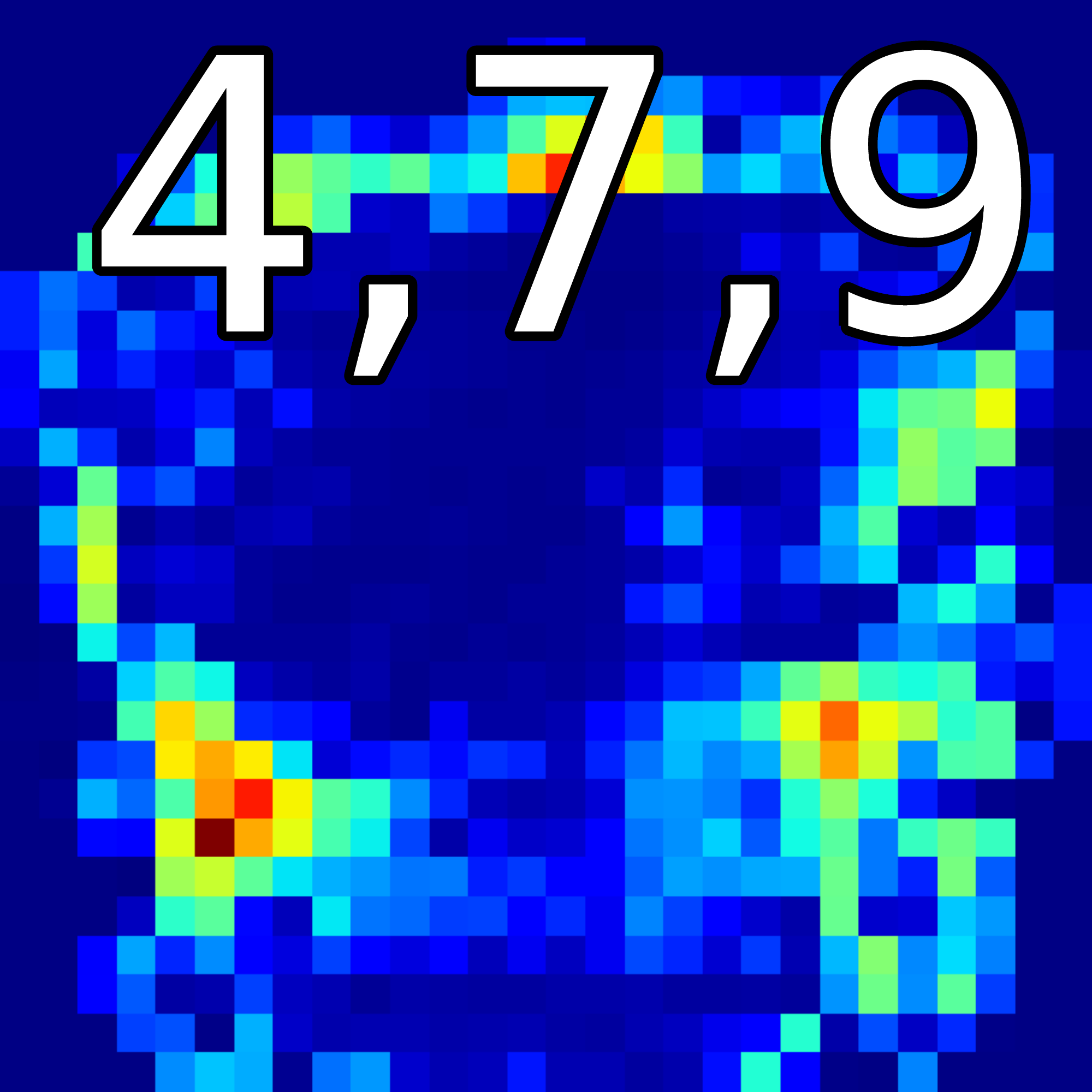}
				\includegraphics[width=0.1\linewidth]{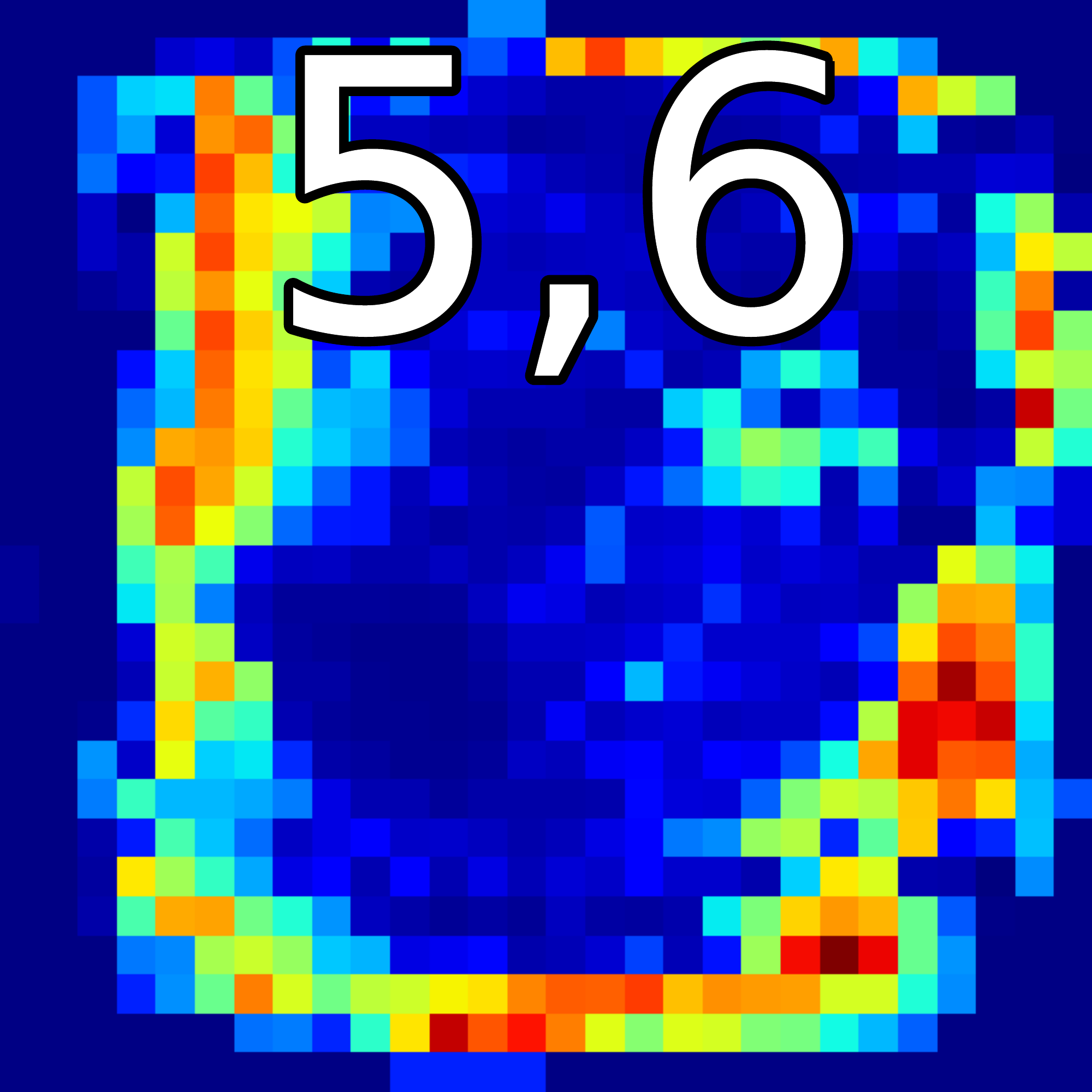}
				\includegraphics[width=0.1\linewidth]{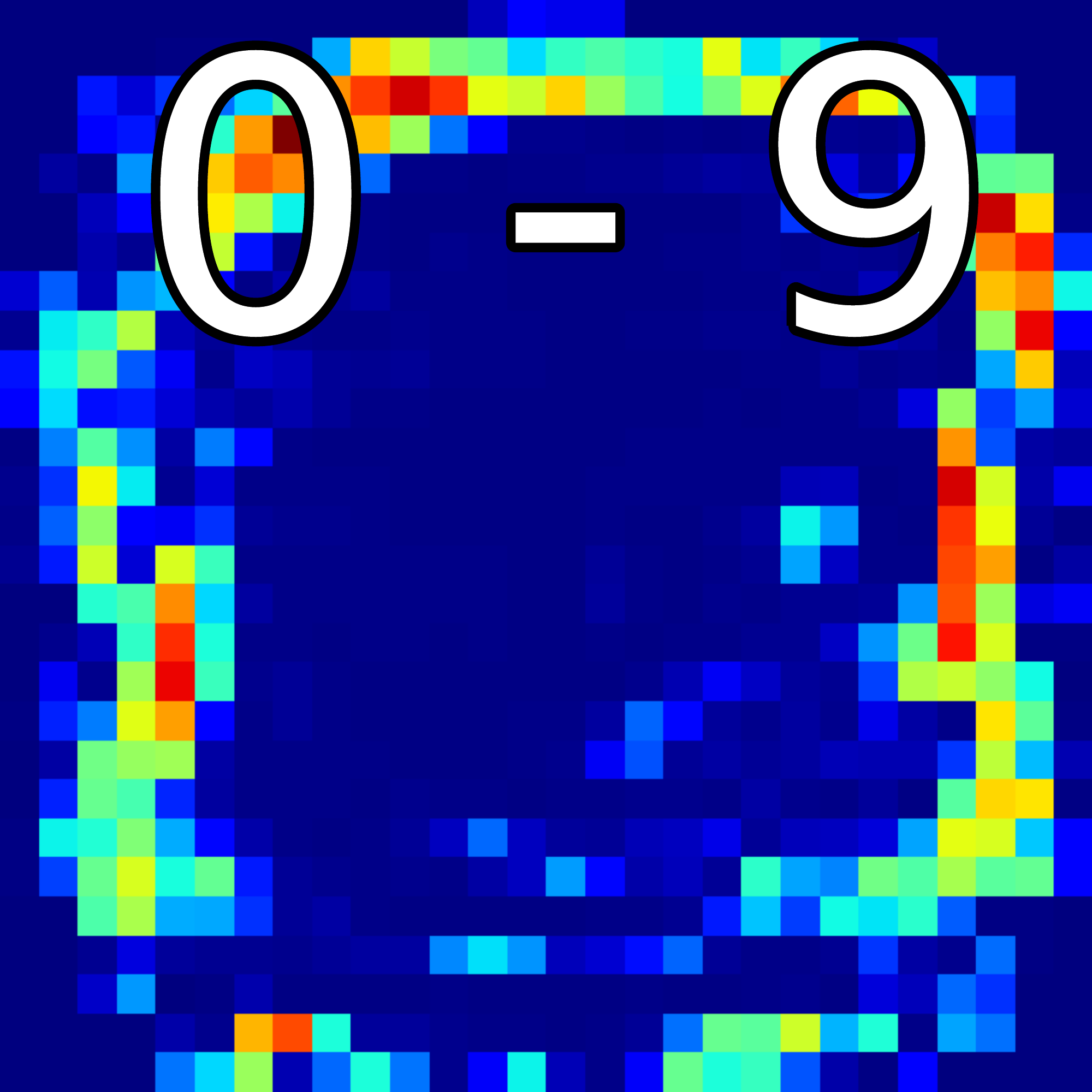}
				\caption{Top: Test accuracy by learning Gaussian kernels (left) and deep convolutional features (right); Bottom: Learned pixel length scales under \gls{ARD} kernels}
				\label{fig:mnist_experiments}
			\end{figure*}

			We now apply \cref{alg:multiclass_conditional_embedding_training} to train a \gls{CEN} with convolutional layers on MNIST. We employ an example architecture from the TensorFlow tutorial on deep MNIST classification \citep{abadi2016tensorflow}. This ReLU activated \gls{CNN} uses two convolutional layers, each with max pooling, followed by a fully connected layer with a drop out probability of 0.5. The original \gls{CNN} then employs a final softmax regressor on the last hidden layer for classification. The \gls{CEN} instead employs a linear kernel on the last hidden layer to construct the conditional mean embedding. We then train both networks from the same initialization using batch updates of $n_{b} = 6000$ images for 800 epochs, with learning rate $\eta = 0.01$. All biases and weight standard deviations are initialized to 0.1. The network weights and biases of the \gls{CEN} are learned jointly with the regularization hyperparameter, initialized to $\lambda = 10$, under our learning objective \eqref{eq:learning_objective}, while the original \gls{CNN} is trained under its usual cross entropy loss. The fully connected layer is trained with a drop out probability of 0.5 for both cases to allow direct comparison. The top right plot in \cref{fig:mnist_experiments} shows that \glspl{CEN} learn at a much faster rate, maintaining a higher test accuracy at all epochs. After 800 epochs, \gls{CEN} reaches a test accuracy of 99.48\%, compared to 99.26\% from the original \gls{CNN}. This demonstrates that our learning algorithm can perform end-to-end learning with convolutional layers from scratch, by simply replacing the softmax layer with a \gls{MCE}. The resulting \gls{CEN} can outperform the original \gls{CNN} in both convergence rate and accuracy.
	
	\section{Conclusion and Future Work}
	
		We developed a scalable hyperparameter learning framework for \glspl{CME} with categorical targets based on Rademacher complexity bounds. These bounds reveal a novel data-dependent quantity $r(\theta, \lambda)$ that reflect its model complexity. We use this measure as an regularization term in addition to the empirical loss for hyperparameter learning. In parallel light to the case with regularized least squares, it remains to be established what type of prior, if any, could correspond to such a regularizer. This would lead to a Bayesian interpretation of our framework. We also envision that such a quantity could potentially be generalized to \glspl{CME} with arbitrary targets, which would enable hyperparameter learning for general conditional mean embeddings in a way that is optimized for the prediction task.
	
\bibliographystyle{apalike}

{\small{\bibliography{references}}}

\begin{thebibliography}{}

\bibitem[Abadi et~al., 2016]{abadi2016tensorflow}
Abadi, M., Barham, P., Chen, J., Chen, Z., Davis, A., Dean, J., Devin, M.,
  Ghemawat, S., Irving, G., Isard, M., et~al. (2016).
\newblock {TensorFlow: A system for large-scale machine learning}.
\newblock In {\em {Proceedings of the 12th USENIX Symposium on Operating
  Systems Design and Implementation (OSDI). Savannah, Georgia, USA}}.

\bibitem[Aeberhard et~al., 1992]{aeberhard1992comparison}
Aeberhard, S., Coomans, D., and De~Vel, O. (1992).
\newblock {Comparison of classifiers in high dimensional settings}.
\newblock {\em Dept. Math. Statist., James Cook Univ., North Queensland,
  Australia, Tech. Rep}, (92-02).

\bibitem[Aly, 2005]{aly2005survey}
Aly, M. (2005).
\newblock {Survey on multiclass classification methods}.
\newblock {\em Neural Netw}, pages 1--9.

\bibitem[Bache and Lichman, 2013]{bache2013uci}
Bache, K. and Lichman, M. (2013).
\newblock {UCI machine learning repository}.

\bibitem[Bartlett and Mendelson, 2002]{bartlett2002rademacher}
Bartlett, P.~L. and Mendelson, S. (2002).
\newblock {Rademacher and Gaussian complexities: Risk bounds and structural
  results}.
\newblock {\em Journal of Machine Learning Research}, 3(Nov):463--482.

\bibitem[Chen et~al., 2010]{chen2010super}
Chen, Y., Welling, M., and Smola, A. (2010).
\newblock {Super-samples from kernel herding}.
\newblock In {\em {The Twenty-Sixth Conference Annual Conference on Uncertainty
  in Artificial Intelligence (UAI-10)}}, pages 109--116. AUAI Press.

\bibitem[Cortes et~al., 2013]{cortes2013learning}
Cortes, C., Kloft, M., and Mohri, M. (2013).
\newblock {Learning kernels using local Rademacher complexity}.
\newblock In {\em {Advances in neural information processing systems}}, pages
  2760--2768.

\bibitem[Fisher, 1936]{fisher1936use}
Fisher, R.~A. (1936).
\newblock {The use of multiple measurements in taxonomic problems}.
\newblock {\em Annals of eugenics}, 7(2):179--188.

\bibitem[Flaxman et~al., 2016]{flaxman2016bayesian}
Flaxman, S., Sejdinovic, D., Cunningham, J.~P., and Filippi, S. (2016).
\newblock {Bayesian learning of kernel embeddings}.
\newblock In {\em {Proceedings of the Thirty-Second Conference on Uncertainty
  in Artificial Intelligence}}, pages 182--191. AUAI Press.

\bibitem[Freire et~al., 2009]{freire2009short}
Freire, A.~L., Barreto, G.~A., Veloso, M., and Varela, A.~T. (2009).
\newblock {Short-term memory mechanisms in neural network learning of robot
  navigation tasks: A case study}.
\newblock In {\em {Robotics Symposium (LARS), 2009 6th Latin American}}, pages
  1--6. IEEE.

\bibitem[Friedman et~al., 2001]{friedman2001elements}
Friedman, J., Hastie, T., and Tibshirani, R. (2001).
\newblock {\em {The elements of statistical learning}}, volume~1.
\newblock Springer series in statistics New York.

\bibitem[Fukumizu et~al., 2004]{fukumizu2004dimensionality}
Fukumizu, K., Bach, F.~R., and Jordan, M.~I. (2004).
\newblock {Dimensionality reduction for supervised learning with reproducing
  kernel Hilbert spaces}.
\newblock {\em Journal of Machine Learning Research}, 5(Jan):73--99.

\bibitem[Fukumizu et~al., 2009]{fukumizu2009kernel}
Fukumizu, K., Gretton, A., Lanckriet, G.~R., Sch{\"o}lkopf, B., and
  Sriperumbudur, B.~K. (2009).
\newblock Kernel choice and classifiability for rkhs embeddings of probability
  distributions.
\newblock In {\em Advances in neural information processing systems}, pages
  1750--1758.

\bibitem[Fukumizu et~al., 2013]{fukumizu2013kernel}
Fukumizu, K., Song, L., and Gretton, A. (2013).
\newblock {Kernel Bayes' rule: Bayesian inference with positive definite
  kernels.}
\newblock {\em Journal of Machine Learning Research}, 14(1):3753--3783.

\bibitem[Genton, 2001]{genton2001classes}
Genton, M.~G. (2001).
\newblock {Classes of kernels for machine learning: a statistics perspective}.
\newblock {\em Journal of machine learning research}, 2(Dec):299--312.

\bibitem[Gretton et~al., 2007]{gretton2007kernel}
Gretton, A., Borgwardt, K.~M., Rasch, M., Sch{\"o}lkopf, B., and Smola, A.~J.
  (2007).
\newblock {A kernel method for the two-sample-problem}.
\newblock In {\em {Advances in neural information processing systems}}, pages
  513--520.

\bibitem[Gretton et~al., 2012]{gretton2012optimal}
Gretton, A., Sejdinovic, D., Strathmann, H., Balakrishnan, S., Pontil, M.,
  Fukumizu, K., and Sriperumbudur, B.~K. (2012).
\newblock Optimal kernel choice for large-scale two-sample tests.
\newblock In {\em Advances in neural information processing systems}, pages
  1205--1213.

\bibitem[Gr{\"u}new{\"a}lder et~al., 2012]{grunewalder2012conditional}
Gr{\"u}new{\"a}lder, S., Lever, G., Baldassarre, L., Patterson, S., Gretton,
  A., and Pontil, M. (2012).
\newblock Conditional mean embeddings as regressors.
\newblock In {\em Proceedings of the 29th International Conference on Machine
  Learning, ICML 2012}, volume~2, pages 1823--1830.

\bibitem[Higham, 2002]{higham2002accuracy}
Higham, N.~J. (2002).
\newblock {\em {Accuracy and stability of numerical algorithms}}.
\newblock SIAM.

\bibitem[Horton and Nakai, 1996]{horton1996probabilistic}
Horton, P. and Nakai, K. (1996).
\newblock {A probabilistic classification system for predicting the cellular
  localization sites of proteins.}
\newblock In {\em {Ismb}}, volume~4, pages 109--115.

\bibitem[Jitkrittum et~al., 2016]{jitkrittum2016interpretable}
Jitkrittum, W., Szab{\'o}, Z., Chwialkowski, K.~P., and Gretton, A. (2016).
\newblock {Interpretable Distribution Features with Maximum Testing Power}.
\newblock In {\em {Advances In Neural Information Processing Systems}}, pages
  181--189.

\bibitem[Kanagawa and Fukumizu, 2014]{kanagawa2014recovering}
Kanagawa, M. and Fukumizu, K. (2014).
\newblock {Recovering Distributions from Gaussian RKHS Embeddings.}
\newblock In {\em {AISTATS}}, pages 457--465.

\bibitem[Kanagawa et~al., 2016]{kanagawa2016filtering}
Kanagawa, M., Nishiyama, Y., Gretton, A., and Fukumizu, K. (2016).
\newblock {Filtering with state-observation examples via kernel monte carlo
  filter}.
\newblock {\em Neural computation}, 28(2):382--444.

\bibitem[Kaya et~al., 2016]{kaya2016banknote}
Kaya, E., Yasar, A., and Saritas, I. (2016).
\newblock {Banknote Classification Using Artificial Neural Network Approach}.
\newblock {\em International Journal of Intelligent Systems and Applications in
  Engineering}, 4(1):16--19.

\bibitem[Kearns et~al., 1997]{kearns1997experimental}
Kearns, M., Mansour, Y., Ng, A.~Y., and Ron, D. (1997).
\newblock {An experimental and theoretical comparison of model selection
  methods}.
\newblock {\em Machine Learning}, 27(1):7--50.

\bibitem[Kimeldorf and Wahba, 1971]{kimeldorf1971some}
Kimeldorf, G. and Wahba, G. (1971).
\newblock {Some results on Tchebycheffian spline functions}.
\newblock {\em Journal of mathematical analysis and applications},
  33(1):82--95.

\bibitem[Kingma and Ba, 2016]{kingma2014adam}
Kingma, D. and Ba, J. (2016).
\newblock {Adam: A method for stochastic optimization}.
\newblock {\em The International Conference on Learning Representations
  (ICLR)}.

\bibitem[Kloft and Blanchard, 2011]{kloft2011local}
Kloft, M. and Blanchard, G. (2011).
\newblock {The local Rademacher complexity of lp-norm multiple kernel
  learning}.
\newblock In {\em {Advances in Neural Information Processing Systems}}, pages
  2438--2446.

\bibitem[LeCun et~al., 1998]{lecun1998gradient}
LeCun, Y., Bottou, L., Bengio, Y., and Haffner, P. (1998).
\newblock {Gradient-based learning applied to document recognition}.
\newblock {\em Proceedings of the IEEE}, 86(11):2278--2324.

\bibitem[Ledoux and Talagrand, 2013]{ledoux2013probability}
Ledoux, M. and Talagrand, M. (2013).
\newblock {\em {Probability in Banach Spaces: isoperimetry and processes}}.
\newblock Springer Science \& Business Media.

\bibitem[Muandet et~al., 2013]{muandet2013domain}
Muandet, K., Balduzzi, D., and Sch{\"o}lkopf, B. (2013).
\newblock {Domain Generalization via Invariant Feature Representation.}
\newblock In {\em {ICML (1)}}, pages 10--18.

\bibitem[Muandet et~al., 2016]{muandet2016kernel}
Muandet, K., Fukumizu, K., Sriperumbudur, B., and Sch{\"o}lkopf, B. (2016).
\newblock Kernel mean embedding of distributions: A review and beyonds.
\newblock {\em stat}, 1050:31.

\bibitem[Pahikkala et~al., 2012]{pahikkala2012unsupervised}
Pahikkala, T., Airola, A., Gieseke, F., and Kramer, O. (2012).
\newblock Unsupervised multi-class regularized least-squares classification.
\newblock In {\em Data Mining (ICDM), 2012 IEEE 12th International Conference
  on}, pages 585--594. IEEE.

\bibitem[Pontil and Maurer, 2013]{pontil2013excess}
Pontil, M. and Maurer, A. (2013).
\newblock Excess risk bounds for multitask learning with trace norm
  regularization.
\newblock In {\em Conference on Learning Theory}, pages 55--76.

\bibitem[Rahimi and Recht, 2008]{rahimi2008random}
Rahimi, A. and Recht, B. (2008).
\newblock Random features for large-scale kernel machines.
\newblock In {\em Advances in neural information processing systems}, pages
  1177--1184.

\bibitem[Rasmussen and Williams, 2006]{rasmussen2006gaussian}
Rasmussen, C.~E. and Williams, C. K.~I. (2006).
\newblock {\em {Gaussian processes for machine learning}}.
\newblock The MIT Press.

\bibitem[Rifkin et~al., 2003]{rifkin2003regularized}
Rifkin, R., Yeo, G., Poggio, T., et~al. (2003).
\newblock Regularized least-squares classification.
\newblock {\em Nato Science Series Sub Series III Computer and Systems
  Sciences}, 190:131--154.

\bibitem[Sch{\"o}lkopf and Smola, 2002]{scholkopf2002learning}
Sch{\"o}lkopf, B. and Smola, A.~J. (2002).
\newblock {\em {Learning with kernels: support vector machines, regularization,
  optimization, and beyond}}.
\newblock MIT press.

\bibitem[Song et~al., 2013]{song2013kernel}
Song, L., Fukumizu, K., and Gretton, A. (2013).
\newblock {Kernel embeddings of conditional distributions: A unified kernel
  framework for nonparametric inference in graphical models}.
\newblock {\em IEEE Signal Processing Magazine}, 30(4):98--111.

\bibitem[Song et~al., 2009]{song2009hilbert}
Song, L., Huang, J., Smola, A., and Fukumizu, K. (2009).
\newblock {Hilbert space embeddings of conditional distributions with
  applications to dynamical systems}.
\newblock In {\em {Proceedings of the 26th Annual International Conference on
  Machine Learning}}, pages 961--968. ACM.

\bibitem[Song et~al., 2008]{song2008tailoring}
Song, L., Zhang, X., Smola, A., Gretton, A., and Sch{\"o}lkopf, B. (2008).
\newblock {Tailoring density estimation via reproducing kernel moment
  matching}.
\newblock In {\em {Proceedings of the 25th international conference on Machine
  learning}}, pages 992--999. ACM.

\bibitem[Sriperumbudur et~al., 2010]{sriperumbudur2010hilbert}
Sriperumbudur, B.~K., Gretton, A., Fukumizu, K., Sch{\"o}lkopf, B., and
  Lanckriet, G.~R. (2010).
\newblock {Hilbert space embeddings and metrics on probability measures}.
\newblock {\em Journal of Machine Learning Research}, 11(Apr):1517--1561.

\bibitem[Xu et~al., 2016]{xu2016local}
Xu, C., Liu, T., Tao, D., and Xu, C. (2016).
\newblock {Local Rademacher complexity for multi-label learning}.
\newblock {\em IEEE Transactions on Image Processing}, 25(3):1495--1507.

\bibitem[Xu and Zhang, 2009]{xu2009refinement}
Xu, Y. and Zhang, H. (2009).
\newblock {Refinement of reproducing kernels}.
\newblock {\em Journal of Machine Learning Research}, 10(Jan):107--140.

\bibitem[Yu et~al., 2014]{yu2014large}
Yu, H.-f., Jain, P., Kar, P., and Dhillon, I. (2014).
\newblock {Large-scale Multi-label Learning with Missing Labels}.
\newblock In {\em {Proceedings of the 31st International Conference on Machine
  Learning (ICML-14)}}, pages 593--601.

\bibitem[Zhou et~al., 2004]{zhou2004size}
Zhou, Z.-H., Wei, D., Li, G., and Dai, H. (2004).
\newblock {On the size of training set and the benefit from ensemble}.
\newblock In {\em {Pacific-Asia Conference on Knowledge Discovery and Data
  Mining}}, pages 298--307. Springer.

\end{thebibliography}

\newpage
\appendix

	\section{Convergence Theorems}
	\label{app:convergence_theorems}
	
		In this section we provide theorems and derivations that establish convergence properties of \glspl{MCE}. Most of the convergence results hold due to \glspl{MCE} being special cases of \glspl{CME}, whose empirical estimates are known to converge. We include this section for completeness.
		
		Suppose $\{X_{i}, Y_{i}\} \sim \mathbb{P}_{X Y}$ are \textit{iid} for all $i \in \mathbb{N}_{n}$, with $X_{i} : \Omega \to \mathcal{X}$ and $Y_{i} : \Omega \to \mathcal{Y}$. We wish to estimate some target function $f : \mathcal{X} \to \mathbb{R}$ by $\hat{f} : \mathcal{X} \to \mathbb{R}$ empirically with a dataset $\{X_{i}, Y_{i}\}_{i = 1}^{n}$ of size $n \in \mathbb{N}_{+}$. Since $\hat{f}$ is empirically estimated, it is a random function over the possible data observation events $\omega \in \Omega$. The aim is to provide a sense of the stochastic convergence of $\hat{f}$ to $f$ by providing an upper bound of their absolute pointwise difference $| \hat{f}(x) - f(x) |$, and show that such an upper bound converges to zero at some stochastic rate. Such an upper bound is provided by the convergence properties of \glspl{CME}. In particular, the empirical \gls{CME} stochastically converges to the \gls{CME} at rate $O_{p}((n \lambda)^{-\frac{1}{2}} + \lambda^{\frac{1}{2}})$, under the assumption that $k(x, \cdot) \in \mathrm{image}(C_{XX})$ \cite[Theorem 6]{song2009hilbert}. That is,
		\begin{equation}
			\begin{aligned}
				&\forall x \in \mathcal{X}, \; \forall \epsilon > 0, \; \exists M_{\epsilon} > 0 \quad s.t.  \\
				&\mathbb{P}\Big[\big\| \hat{\mu}_{Y | X = x} - \mu_{Y | X = x} \big\|_{\mathcal{H}_{l}} > M_{\epsilon} \Big((n \lambda)^{-\frac{1}{2}} + \lambda^{\frac{1}{2}}\Big)\Big] < \epsilon.
			\end{aligned}
		\label{eq:empirical_conditional_embedding_stochastic_convergence}
		\end{equation}
		
		In practice, the assumption that $k(x, \cdot) \in \mathrm{image}(C_{XX})$ can be relaxed by replacing $\mathcal{U}_{Y | X} = C_{YX} C_{XX}^{-1}$ with $\mathcal{U}_{Y | X} = C_{YX} (C_{XX} + \lambda I)^{-1}$ \citep{song2013kernel}. This will apply to all subsequent theorems in this section.
		\begin{theorem}[Pointwise and Uniform Convergence of Conditional Mean Embedding Estimators]
			\label{thm:pointwise_uniform_convergence}
			Suppose that $k(x, \cdot)$ is in the image of $C_{XX}$ and that there exists $0 \leq \gamma(x) < \infty$ such that for some estimator function $\hat{f} : \mathcal{X} \to \mathbb{R}$ and target function  $f : \mathcal{X} \to \mathbb{R}$,
			\begin{equation}
				| \hat{f}(x) - f(x) | \leq \gamma(x) \big\| \hat{\mu}_{Y | X = x} - \mu_{Y | X = x} \big\|_{\mathcal{H}_{l}}, \forall x \in \mathcal{X},
			\label{eq:estimator_error_bound}
			\end{equation}
			then the estimator $\hat{f}$ converges pointwise to the target $f$ at a stochastic rate of at least $O_{p}((n \lambda)^{-\frac{1}{2}} + \lambda^{\frac{1}{2}})$. Further, if $\gamma(x) = \gamma$ is independent of $x \in \mathcal{X}$, then this convergence is uniform.
		\end{theorem}
		
		\begin{proof}
			Suppose that there exists $0 \leq \gamma(x) < \infty$ such that \eqref{eq:estimator_error_bound} is satisfied. That is, the inequality \eqref{eq:estimator_error_bound} holds for all possible data observations $\{X_{i}, Y_{i}\}_{i = 1}^{n}$ where $X_{i} : \Omega \to \mathcal{X}$,  $Y_{i} : \Omega \to \mathcal{Y}$ for all $i \in \mathbb{N}_{n}$. For any constant $C$, the implication statement $\big\| \hat{\mu}_{Y | X = x} - \mu_{Y | X = x} \big\|_{\mathcal{H}_{\delta}} \leq C \implies | \hat{f}(x) - f(x) | \leq C \gamma(x)$ holds for all possible observation events $\omega \in \Omega$. Writing this explicitly in event space translates this to a probability statement,
			\begin{equation}
			\begin{aligned}
				\{\omega \in \Omega : \big\| \hat{\mu}_{Y | X = x} - \mu_{Y | X = x} \big\|_{\mathcal{H}_{l}} \leq C\} &\subseteq \{\omega \in \Omega : | \hat{f}(x) - f(x) | \leq C \gamma(x)\} \\
				\implies \mathbb{P}\Big[\big\| \hat{\mu}_{Y | X = x} - \mu_{Y | X = x} \big\|_{\mathcal{H}_{l}} \leq C\Big] &\leq \mathbb{P}\Big[| \hat{f}(x) - f(x) | \leq C \gamma(x) \Big].
			\label{eq:probability_statement}
			\end{aligned}
			\end{equation}
			
			Since we assume that $k(x, \cdot) \in \mathrm{image}(C_{XX})$, statement \eqref{eq:empirical_conditional_embedding_stochastic_convergence} is valid. By letting $C = M_{\epsilon} ((n \lambda)^{-\frac{1}{2}} + \lambda^{\frac{1}{2}})$ in \eqref{eq:probability_statement}, we immediately have that the probability inequality in statement \eqref{eq:empirical_conditional_embedding_stochastic_convergence} is also true if we replace $\| \hat{\mu}_{Y | X = x} - \mu_{Y | X = x} \|$ with $| \hat{f}(x) - f(x) |$ and $M_{\epsilon}$ with $\gamma(x) M_{\epsilon}$,	
			\begin{equation}
			\begin{aligned}
				\mathbb{P}\Big[\big\| \hat{\mu}_{Y | X = x} - \mu_{Y | X = x} \big\|_{\mathcal{H}_{l}} > M_{\epsilon} \Big((n \lambda)^{-\frac{1}{2}} + \lambda^{\frac{1}{2}}\Big)\Big] &< \epsilon \\
				\implies 1 - \mathbb{P}\Big[\big\| \hat{\mu}_{Y | X = x} - \mu_{Y | X = x} \big\|_{\mathcal{H}_{l}} \leq M_{\epsilon} \Big((n \lambda)^{-\frac{1}{2}} + \lambda^{\frac{1}{2}}\Big)\Big] &< \epsilon \\
				\implies \mathbb{P}\Big[\big\| \hat{\mu}_{Y | X = x} - \mu_{Y | X = x} \big\|_{\mathcal{H}_{l}} \leq M_{\epsilon} \Big((n \lambda)^{-\frac{1}{2}} + \lambda^{\frac{1}{2}}\Big)\Big] &> 1 -  \epsilon \\
				\implies \mathbb{P}\Big[| \hat{f}(x) - f(x) | \leq \gamma(x) M_{\epsilon} \Big((n \lambda)^{-\frac{1}{2}} + \lambda^{\frac{1}{2}}\Big)\Big] &> 1 -  \epsilon \\
				\implies 1 - \mathbb{P}\Big[| \hat{f}(x) - f(x) | \leq \gamma(x) M_{\epsilon} \Big((n \lambda)^{-\frac{1}{2}} + \lambda^{\frac{1}{2}}\Big)\Big] &< \epsilon \\
				\implies \mathbb{P}\Big[| \hat{f}(x) - f(x) | > \gamma(x) M_{\epsilon} \Big((n \lambda)^{-\frac{1}{2}} + \lambda^{\frac{1}{2}}\Big)\Big] &< \epsilon,
			\end{aligned}	
			\end{equation}
			where we employed statement \eqref{eq:probability_statement} between the third and fourth line for $C = M_{\epsilon} ((n \lambda)^{-\frac{1}{2}} + \lambda^{\frac{1}{2}})$. Therefore, since $M_{\epsilon}$ is arbitrary, define $\tilde{M}_{\epsilon}(x) := \gamma(x) M_{\epsilon}$ so that, with the above result, the statement \eqref{eq:empirical_conditional_embedding_stochastic_convergence} implies the following,
			\begin{equation}
				\forall x \in \mathcal{X}, \; \epsilon > 0, \; \exists \tilde{M}_{\epsilon}(x) > 0 \quad s.t. \quad \mathbb{P}\Big[\big| \hat{f}(x) - f(x) \big| > \tilde{M}_{\epsilon}(x) \Big((n \lambda)^{-\frac{1}{2}} + \lambda^{\frac{1}{2}}\Big)\Big] < \epsilon.
			\end{equation}
			
			In other words, the function $\hat{f}$ stochastically converges pointwise to $f$ with a rate of at least $O_{p}((n \lambda)^{-\frac{1}{2}} + \lambda^{\frac{1}{2}})$. The convergence is pointwise as the constant $\tilde{M}_{\epsilon}(x)$ may be different for each point $x \in \mathcal{X}$. If $\gamma(x) = \gamma$ such that $\tilde{M}_{\epsilon}(x) = \tilde{M}_{\epsilon}$ does not depend on $x \in \mathcal{X}$, then this stochastic convergence is uniform in its domain $\mathcal{X}$.
			\qed
		\end{proof}
	
		With \cref{thm:pointwise_uniform_convergence}, we can now show the convergence of various estimators based on the conditional mean embedding, as long as we can show that their estimator error is upper bounded by a multiple of the conditional mean embedding error in the \gls{RKHS} norm. As such, we turn to the convergence of the empirical decision probability function \eqref{eq:empirical_decision_probability} below.
		
		\begin{theorem}[Convergence of Empirical Decision Probability Function]
			\label{thm:probability_convergence}
			Assuming that $k(x, \cdot)$ is in the image of $C_{XX}$, the empirical decision probability function $\hat{p}_{c} : \mathcal{X} \to \mathbb{R}$ \eqref{eq:empirical_decision_probability} converges uniformly to the true decision probability $p_{c} : \mathcal{X} \to [0, 1]$ \eqref{eq:decision_probability} at a stochastic rate of at least $O_{p}((n \lambda)^{-\frac{1}{2}} + \lambda^{\frac{1}{2}})$ for all $c \in \mathcal{Y} = \mathbb{N}_{m}$.
		\end{theorem}
		
		\begin{proof}
			Consider the pointwise absolute difference between the decision probability and its empirical estimate,			
			\begin{equation}
			\begin{aligned}
				| \hat{p}_{c}(x) - p_{c}(x) | &= | \langle \hat{\mu}_{Y | X = x}, \mathbb{1}_{c} \rangle - \langle \mu_{Y | X = x}, \mathbb{1}_{c} \rangle | \\
				&= | \langle \hat{\mu}_{Y | X = x} - \mu_{Y | X = x}, \mathbb{1}_{c} \rangle | \\
				&\leq \big\| \hat{\mu}_{Y | X = x} - \mu_{Y | X = x} \big\|_{\mathcal{H}_{\delta}} \big\| \mathbb{1}_{c} \big\|_{\mathcal{H}_{\delta}},
			\label{eq:decision_probability_error_upper_bound}
			\end{aligned}
			\end{equation}
			where the last inequality follows from the Cauchy Schwarz inequality in a Hilbert space.
			
			Since $\mathbb{1}_{c} = \delta(c, \cdot)$ and using the fact that $\delta$ is a reproducing kernel, we have that for all $c \in \mathcal{Y} = \mathbb{N}_{m}$.
			\begin{equation}
			\begin{aligned}
				\big\| \mathbb{1}_{c} \big\|_{\mathcal{H}_{\delta}}^{2} = \langle \mathbb{1}_{c}, \mathbb{1}_{c} \rangle = \langle \delta(c, \cdot), \delta(c, \cdot) \rangle = \delta(c, c) = 1.
				\label{eq:indicator_RKHS_norm}
			\end{aligned}
			\end{equation}
			
			Therefore, by \cref{thm:pointwise_uniform_convergence} with $\gamma(x) = 1$ independent of $x \in \mathcal{X}$, $\hat{p}_{c}$ converges uniformly to $p_{c}$ at a stochastic rate of at least $O_{p}((n \lambda)^{-\frac{1}{2}} + \lambda^{\frac{1}{2}})$ for all $c \in \mathcal{Y} = \mathbb{N}_{m}$.
			\qed
		\end{proof}
	
		The above proof is for uniform convergence over all $x \in \mathcal{X}$ at the stochastic rate of at least $O_{p}((n \lambda)^{-\frac{1}{2}} + \lambda^{\frac{1}{2}})$. Intuitively, however, for stationary zero-centered kernels like the Gaussian kernel, the convergence rate may be higher at regions of high data density, since the kernel effects, being centered around the training data, are stronger at these regions. The worse case convergence rate described here in the theorem would be a tight lower bound for regions in $\mathcal{X}$ with lower data density, where the kernel effects have decayed and most empirical probabilities are smaller and further from summing up to one.
		
		Because the label space $\mathcal{Y} = \mathbb{N}_{m}$ is discrete and finite, \textit{bounded} functions $g \in \mathcal{H}_{\delta}$ in the \gls{RKHS} are equivalent to their vector representations $\bvec{g} := \{g(c)\}_{c = 1}^{m}$, because one can always write $g = \sum_{c = 1}^{m} g(c) \delta(c, \cdot)$. In other words, there is an isomorphism between $\mathbb{H}_{\delta}$ and $\mathbb{R}^{m}$. A convenient consequence is that inner products in the \gls{RKHS} are simply the usual dot products in a Euclidean space, since
		\begin{equation}
		\begin{aligned}
			\langle g_{1}, g_{2} \rangle_{\mathcal{H}_{\delta}} &= \bigg\langle \sum_{c = 1}^{m} g_{1}(c) \delta(c, \cdot), \sum_{c' = 1}^{m} g_{2}(c') \delta(c', \cdot)  \bigg\rangle_{\mathcal{H}_{\delta}} \\
			&= \sum_{c = 1}^{m} \sum_{c' = 1}^{m} g_{1}(c) g_{2}(c') \langle \delta(c, \cdot), \delta(c', \cdot) \rangle_{\mathcal{H}_{\delta}} \\
			&= \sum_{c = 1}^{m} g_{1}(c) g_{2}(c) \\
			&= \bvec{g}_{1} \cdot \bvec{g}_{2}.
		\end{aligned}
		\end{equation}
	
		Consequently, the \gls{RKHS} norm for bounded functions $g \in \mathcal{H}_{\delta}$ is simply the $\ell_{2}$-norm of its vector representation $\bvec{g}$,
		\begin{equation}
			\lVert g \rVert_{\mathcal{H}_{\delta}} = \lVert \bvec{g} \rVert_{\ell_{2}}.
		\label{eq:RKHS_norm_is_l2_norm}
		\end{equation}
		
		A special and convenient result that arises due to this discrete and finite label space is that the decision probabilities and its empirical estimate are simply the conditional mean embeddings and its empirical estimate.
		
		\begin{lemma}[Decision Probabilities are Conditional Mean Embeddings]
			\label{thm:probability_is_embedding}
			The decision probability for class $c \in \mathbb{N}_{m}$ given an example $x \in \mathcal{X}$ is the conditional mean embedding with $l = \delta$ conditioned at example $x$ evaluated at label $c$,
			\begin{equation}
				p_{c}(x) := \mathbb{P}[Y = c | X = x] = \mu_{Y | X = x}(c).
			\end{equation}
			
			Therefore, $\bvec{p}(x) \equiv \mu_{Y | X = x}$.
		\end{lemma}
	
		\begin{proof} 
			Since indicator functions are the canonical features of the label \gls{RKHS} $\mathcal{H}_{\delta}$, we employ the fact that expectations of indicator functions are probabilities to prove this claim,
			\begin{equation}
			\begin{aligned}
				\mu_{Y | X = x}(c) :=& \mathbb{E}[l(Y, c) | X = x ]= \mathbb{E}[\delta(Y, c) | X = x] \\
				=& \mathbb{E}[\mathbb{1}_{c}(Y) | X = x] = \mathbb{P}[Y \in \{c\} | X = x] \\
				=& \mathbb{P}[Y = c | X = x] =: p_{c}(x).
			\end{aligned}
			\end{equation}
			\qed
		\end{proof}

		\begin{lemma}[Empirical Decision Probabilities are Empirical Conditional Mean Embeddings]
			\label{thm:empirical_probability_is_embedding}
			The empirical decision probability \eqref{eq:empirical_decision_probability} for class $c \in \mathbb{N}_{m}$ given an example $x \in \mathcal{X}$ is the empirical conditional mean embedding with $l = \delta$ conditioned at example $x$ evaluated at label $c$,
			\begin{equation}
			\hat{p}_{c}(x) = \hat{\mu}_{Y | X = x}(c).
			\end{equation}
			
			Therefore, $\hat{\bvec{p}}(x) \equiv \hat{\mu}_{Y | X = x}$.
		\end{lemma}
		
		\begin{proof}
			Let the canonical feature maps of $\mathcal{X}$ and $\mathcal{Y}$ be $\phi(x) = k(x, \cdot)$ and $\psi(y) = l(y, \cdot) = \delta(y, \cdot)$, then the empirical conditional mean embedding is defined by
			\begin{equation}
				\hat{\mu}_{Y | X = x} := \hat{\mathcal{U}}_{Y | X} \phi(x).
			\end{equation}
			
			By the reproducing property, the evaluation of $\hat{\mu}_{Y | X = x} \in \mathcal{H}_{l}$ is given by a dot product,
			\begin{equation}
			\begin{aligned}
				\hat{\mu}_{Y | X = x}(c) &= \langle l(c, \cdot), \hat{\mu}_{Y | X = x} \rangle \\
				&= \langle \psi(c), \hat{\mu}_{Y | X = x} \rangle \\
				&= \psi(c)^{T} \hat{\mu}_{Y | X = x} \\
				&= \psi(c)^{T} \hat{\mathcal{U}}_{Y | X} \phi(x) \\
				&= \psi(c)^{T} \Psi (K + n \lambda I)^{-1} \Phi^{T} \phi(x) \\
				&= \bvec{l}_{c}^{T} (K + n \lambda I)^{-1} \bvec{k}(x),
			\end{aligned}
			\end{equation}
			where $\bvec{l}_{c} := \{l(y_{i}, c)\}_{i = 1}^{n}$ and $\bvec{k}_{x} := \{k(x_{i}, x)\}_{i = 1}^{n}$. While the notation $\bvec{l}_{c}$ is usually avoided due do its similarity to $\bvec{1}_{c}$, in this context they happen to represent equal quantities,
			\begin{equation}
				\bvec{l}_{c} := \{l(y_{i}, c)\}_{i = 1}^{n} = \{\delta(y_{i}, c)\}_{i = 1}^{n} = \{\mathbb{1}_{c}(y_{i})\}_{i = 1}^{n} =: \bvec{1}_{c}.
			\end{equation}
			
			The claim then immediately follows by the definition of our decision probability estimator,
			\begin{equation}
				\hat{\mu}_{Y | X = x}(c) = \bvec{1}_{c}^{T} (K + n \lambda I)^{-1} \bvec{k}(x) =: \hat{p}_{c}(x).
			\end{equation}
			\qed
		\end{proof}
		
		\Cref{thm:empirical_probability_is_embedding} shows that the decision function $\bvec{f}(x)$ \eqref{eq:empirical_decision_probability_vector} of a \gls{MCE} is no more than the empirical conditional mean embedding estimated from the data.
		
		Since we have identified the equivalence of decision probabilities and the conditional mean embedding, we can now also show that the empirical decision probability vector also converges to the true decision probability vector.
		
		\begin{lemma}[Uniform Convergence of Empirical Decision Probability Vector Function in $\ell_{1}$ and $\ell_{2}$]
			\label{thm:probability_vector_convergence} 
			Assuming that $k(x, \cdot)$ is in the image of $C_{XX}$, the empirical decision probability vector function $\hat{\bvec{p}} : \mathcal{X} \to \mathbb{R}^{m}$ \eqref{eq:empirical_decision_probability_vector} converges uniformly to the true decision probability vector function $\bvec{p} : \mathcal{X} \to [0, 1]^{m}$ in the $\ell_{1}$-norm and $\ell_{2}$-norm, where $\bvec{p}(x) := \{p_{c}(x)\}_{c = 1}^{m}$, at a stochastic rate of at least $O_{p}((n \lambda)^{-\frac{1}{2}} + \lambda^{\frac{1}{2}})$ for all $c \in \mathcal{Y} = \mathbb{N}_{m}$.
		\end{lemma}
		
		\begin{proof}
			For convergence in $\ell_{1}$, we simply extend \cref{thm:probability_convergence}, which proved that each entry of $\hat{\bvec{p}}(x)$ converges pointwise uniformly at a rate of $O_{p}((n \lambda)^{-\frac{1}{2}} + \lambda^{\frac{1}{2}})$ to the corresponding entry of $\bvec{p}(x)$. Since each entry converges stochastically at a rate of $O_{p}((n \lambda)^{-\frac{1}{2}} + \lambda^{\frac{1}{2}})$, then so does the entire vector. More formally, from \eqref{eq:decision_probability_error_upper_bound} and \eqref{eq:indicator_RKHS_norm}, the $\ell_{1}$-norm of the difference can be bounded,
			\begin{equation}
			\begin{aligned}
				{\lVert \hat{\bvec{p}}(x) - \bvec{p}(x) \rVert}_{\ell_{1}} :=& \sum_{c = 1}^{m} | \hat{p}_{c}(x) - p_{c}(x) | \\
				\leq& \sum_{c = 1}^{m} \big\| \hat{\mu}_{Y | X = x} - \mu_{Y | X = x} \big\|_{\mathcal{H}_{\delta}} \\
				=& m \big\| \hat{\mu}_{Y | X = x} - \mu_{Y | X = x} \big\|_{\mathcal{H}_{\delta}}.
			\end{aligned}
			\end{equation}
			
			Therefore, by \cref{thm:pointwise_uniform_convergence} with $\gamma(x) = m$ independent of $x \in \mathcal{X}$, we have uniform convergence in $\ell_{1}$ where we replace all instances of $| \hat{f}(x) - f(x) |$ in the proof of \cref{thm:pointwise_uniform_convergence} with ${\lVert \hat{\bvec{p}}(x) - \bvec{p}(x) \rVert}_{\ell_{1}}$.
			
			For convergence in $\ell_{2}$, we show that the $\ell_{2}$-norm of the difference between the true and empirical decision probability vector functions is the same as the \gls{RKHS} norm of the difference between the true and empirical conditional mean embedding, which converges to zero at a stochastic rate of at least $O_{p}((n \lambda)^{-\frac{1}{2}} + \lambda^{\frac{1}{2}})$ for all $x \in \mathcal{X}$ and $c \in \mathcal{Y} = \mathbb{N}_{m}$ by \eqref{eq:empirical_conditional_embedding_stochastic_convergence}. To this end, we use \cref{thm:probability_is_embedding} and \cref{thm:empirical_probability_is_embedding} and write 
			\begin{equation}
			\begin{aligned}
				{\lVert \hat{\bvec{p}}(x)  - \bvec{p}(x) \rVert}_{\ell_{2}} &= {\lVert \{ \hat{p}_{c}(x) \}_{c = 1}^{m} - \{ p_{c}(x) \}_{c = 1}^{m} \rVert}_{\ell_{2}} \\
				&= {\lVert \{ \hat{p}_{c}(x) - p_{c}(x) \}_{c = 1}^{m} \rVert}_{\ell_{2}} \\
				&= {\lVert \{ \hat{\mu}_{Y | X = x}(c) - \mu_{Y | X = x}(c)\}_{c = 1}^{m} \rVert}_{\ell_{2}} \\
				&= \| \hat{\bm{\mu}}_{Y | X = x} - \bm{\mu}_{Y | X = x}  \|_{\ell_{2}} \\
				&= {\lVert \hat{\mu}_{Y | X = x} - \mu_{Y | X = x}\rVert}_{\mathcal{H}_{\delta}},
			\end{aligned}
			\end{equation}
			where the last equality comes from \eqref{eq:RKHS_norm_is_l2_norm} and the fact that the empirical and true conditional mean embeddings are bounded functions in the \gls{RKHS}. Again, by \cref{thm:pointwise_uniform_convergence} with $\gamma(x) = 1$ independent of $x \in \mathcal{X}$, we have uniform convergence in $\ell_{2}$.
			\qed
		\end{proof}

	\subsection{Information Entropy of \glspl{MCE}}
	\label{app:information_entropy}
	
		The \gls{MCE} provides decision probabilities instead of just a single label prediction. Such a probabilistic classifier allows us to quantify the uncertainty of its predictions for any given example $x \in \mathcal{X}$ through the information entropy. This is ideal for detecting the decision boundaries of the classifier and areas of low data density.
		
		We present two main approaches for inferring the information entropy from the classifier. Specifically, we infer estimates for the information entropy of the possible labels $Y$ for a given example $X = x$,
		\begin{equation}
			h(x) := \mathbb{H}[Y | X = x] = - \sum_{c = 1}^{m} p_{c}(x) \log{p_{c}(x)}.
		\end{equation}
		
		The first approach is straight forward, which involves simply computing the information entropy with the clip normalized probabilities \eqref{eq:empirical_decision_probability_clip_normalized}, at the query point $x \in \mathcal{X}$,
		\begin{equation}
			\tilde{h}(x) := - \sum_{c = 1}^{m} \tilde{p}_{c}(x) \log{\tilde{p}_{c}(x)}.
		\label{eq:entropy_clip_normalize}
		\end{equation}
		
		We call \eqref{eq:entropy_clip_normalize} the \textit{clip-normalized information entropy}. Since $\tilde{p}_{c}(x)$ converges pointwise to $p_{c}(x)$ with increasing data, $\tilde{h}(x)$ also converges pointwise to $h(x)$.
		
		Just as decision probabilities can be expressed as an expectation of indicator functions, information entropy can be expressed as expected information,
		\begin{equation}
		\begin{aligned}
			\mathbb{H}[Y | X = x] &= - \sum_{c = 1}^{m} \mathbb{P}[Y = c| X = x] \log{\mathbb{P}[Y = c | X = x]} \\
			&= \mathbb{E}[- \log{\mathbb{P}[Y | X = x]} | X = x] \\
			&= \mathbb{E}[u_{x}(Y) | X = x],
		\end{aligned}
		\end{equation}
		where $u_{x}(y) := - \log{\mathbb{P}[Y = y | X = x]}$ is the \textit{information} (in nats) we would gain when we discover that example $x$ actually has label $y$. Note that while $\mathbb{P}[Y = c | X = x]$ is a constant, we employ the shorthand notation $\mathbb{P}[Y| X = x]$ for the random variable $g(Y)$ where $g(y) := \mathbb{P}[Y = y | X = x]$. If $u_{x} : \mathbb{N}_{m} \to \mathbb{R}$ is in the \gls{RKHS} $\mathcal{H}_{\delta}$, then we know that this expectation can also be approximated by $\langle \hat{\mu}_{Y | X = x}, u_{x} \rangle$. This is the basis of our second approach.
		
		Assuming that $\mathbb{P}[Y = y | X = x]$ is never exactly zero for all labels $y \in \mathcal{Y}$ and examples $x \in \mathcal{X}$, then $u_{x}(y)$ is bounded on its discrete domain $\mathbb{N}_{m}$. We can thus write $u_{x} = \sum_{c = 1}^{m} - \log{\mathbb{P}[Y = c | X = x]} \delta(c, \cdot)$ which shows that $u_{x}$ is in the span of the canonical kernel features and is thus in the \gls{RKHS}. Hence, similar to the case with decision probabilities, with $u_{x} \in \mathcal{H}_{\delta}$ and $\bvec{u}_{x} := \{u_{x}(y_{i})\}_{i = 1}^{n}$ we let $g = u_{x}$ in \eqref{eq:empirical_conditional_expectation} and estimate $h(x)$ by
		\begin{equation}
			\langle \hat{\mu}_{Y | X = x}, u_{x} \rangle = \bvec{u}_{x}^{T} (K + n \lambda I)^{-1} \bvec{k}(x).
		\end{equation}
		
		Unfortunately, $u_{x}$ is not known exactly, since $\mathbb{P}[Y = y | X = x]$ is not known exactly. Instead, since $\hat{p}_{c}(x)$ is a consistent estimate for $\mathbb{P}[Y = c | X = x]$ by \cref{thm:probability_convergence}, we propose to replace $u_{x}(y)$ with the information of $\hat{p}_{y}(x)$. However, we cannot simply take the log of this estimator, as $\hat{p}_{y}(x)$ may produce non-positive estimates to the prediction probabilities. The straight forward way to mitigate this problem is to clip $\hat{p}_{y}(x)$ from the bottom by a very small number, before taking the log. However, experiments show that this produces non-smooth estimates over $\mathcal{X}$ and the degree of smoothness varies drastically between different choices of that small number. Instead, in virtue of the fact that $\lim_{p \to 0} - p \log{p} = 0$ even though $\lim_{p \to 0} - \log{p} = \infty$, we simply define the information estimate $\hat{u}_{x}(y)$ as zero if the empirical decision probability is non-positive,
		\begin{equation}
			\hat{u}_{x}(y) := \begin{cases}
			- \log{\hat{p}_{y}(x)} & \mathrm{if } \quad \hat{p}_{y}(x) > 0, \\
			0 & \mathrm{otherwise}. \end{cases}
		\label{eq:empirical_information}
		\end{equation}
		
		It remains to show that $\hat{u}_{x} \in \mathcal{H}_{\delta}$. Indeed, the identity $\hat{u}_{x} = \sum_{c = 1}^{m} \hat{u}_{x}(c) \delta(c, \cdot)$ holds and thus $\hat{u}_{x}$ is in the span of the kernel canonical features. We then arrive at the following estimate for $h(x)$,
		\begin{equation}
			\hat{h}(x) := \langle \hat{\mu}_{Y | X = x}, \hat{u}_{x} \rangle = \hat{\bvec{u}}_{x}^{T} (K + n \lambda I)^{-1} \bvec{k}(x),
		\label{eq:empirical_information_entropy}
		\end{equation}
		where $\hat{\bvec{u}}_{x} := \{\hat{u}_{x}(y_{i})\}_{i = 1}^{n}$. Similar to the case with decision probabilities \eqref{eq:decision_probability}, the information entropy estimate \eqref{eq:empirical_information_entropy} is not guaranteed to be non-negative. However, in practice these negative values are close to zero. Furthermore, negative estimated information entropy implies that the model is very confident about its prediction, and it suffices to simply clip the entropy at zero if strict information entropy is required. 
		
		Since this estimator is now based on the inner product between the empirical conditional mean embedding and another empirically estimate function, instead of between the empirical conditional mean embedding and a known function like the decision probability estimate, it is not immediately clear that such an estimator converges. Nevertheless, intuition tells us that the inner product between two converging quantities should converge. We proceed to show that this intuition is correct.
		
		\begin{theorem}[Convergence of Empirical Information Entropy Function]
			\label{thm:entropy_convergence}
			Assuming that $k(x, \cdot)$ is in the image of $C_{XX}$, the empirical information entropy function $\hat{h} : \mathcal{X} \to \mathbb{R}$ \eqref{eq:empirical_information_entropy} converges pointwise to the true information entropy function $h : \mathcal{X} \to [0, \infty)$ at a stochastic rate of at least $O_{p}((n \lambda)^{-\frac{1}{2}} + \lambda^{\frac{1}{2}})$.
		\end{theorem}
		
		\begin{proof}
			Since we are interested in the asymptotic properties of our estimators when $n \to \infty$, and we have proved that the empirical decision probabilities converges to the true probabilities (\cref{thm:probability_convergence}), the condition $\hat{p}_{c}(x) > 0$ holds for large $n$ such that we simply have $\hat{u}_{x}(c) = - \log{\hat{p}_{c}(x)}$. That is, the effects of clipping for the information estimate \eqref{eq:empirical_information} vanishes.
			
			Consider the pointwise absolute difference between the empirical and true information entropy,
			\begin{equation}
			\begin{aligned}
				| \hat{h}(x) - h(x) | &= | \langle \hat{\mu}_{Y | X = x}, \hat{u}_{x} \rangle_{\mathcal{H}_{\delta}} - \langle \mu_{Y | X = x}, u_{x} \rangle_{\mathcal{H}_{\delta}} | \\
				&= | \langle \hat{\mu}_{Y | X = x}, \hat{u}_{x} \rangle_{\mathcal{H}_{\delta}} - \langle \hat{\mu}_{Y | X = x}, u_{x} \rangle_{\mathcal{H}_{\delta}} \\
				& \qquad + \langle \hat{\mu}_{Y | X = x}, u_{x} \rangle_{\mathcal{H}_{\delta}} - \langle \mu_{Y | X = x}, u_{x} \rangle_{\mathcal{H}_{\delta}} | \\
				&\leq | \langle \hat{\mu}_{Y | X = x}, \hat{u}_{x} \rangle_{\mathcal{H}_{\delta}} - \langle \hat{\mu}_{Y | X = x}, u_{x} \rangle_{\mathcal{H}_{\delta}} | \\
				& \qquad + | \langle \hat{\mu}_{Y | X = x}, u_{x} \rangle_{\mathcal{H}_{\delta}} - \langle \mu_{Y | X = x}, u_{x} \rangle_{\mathcal{H}_{\delta}} | \\
				&= | \langle \hat{\mu}_{Y | X = x}, \hat{u}_{x} - u_{x} \rangle_{\mathcal{H}_{\delta}} | + | \langle \hat{\mu}_{Y | X = x} - \mu_{Y | X = x}, u_{x} \rangle_{\mathcal{H}_{\delta}} | \\
				&\leq \| \hat{\mu}_{Y | X = x} \|_{\mathcal{H}_{\delta}} \| \hat{u}_{x} - u_{x} \|_{\mathcal{H}_{\delta}} + \| \hat{\mu}_{Y | X = x} - \mu_{Y | X = x} \|_{\mathcal{H}_{\delta}} \| u_{x} \|_{\mathcal{H}_{\delta}},
			\label{eq:information_entropy_bound}
			\end{aligned}
			\end{equation}
			where the we used the triangle inequality and Cauchy Schwarz inequality in a Hilbert space respectively. Since the kernel $l = \delta$ is bounded, so is $\hat{\mu}_{Y | X = x}(c) = \sum_{i = 1}^{n} w_{i} \delta(y_{i}, c)$ for some embedding weights $w_{i}$ and all $c \in \mathbb{N}_{m}$, and thus its \gls{RKHS} norm is finite for all $n \in \mathbb{N}_{n}$. Similarly, assuming that $p_{c}(x)$ is never exactly zero, $u_{x}(c)$ is also finite for all $c \in \mathbb{N}_{m}$ and thus so is its \gls{RKHS} norm. We already know that $\| \hat{\mu}_{Y | X = x} - \mu_{Y | X = x} \|_{\mathcal{H}_{\delta}}$ stochastically converges to zero at the rate $O_{p}((n \lambda)^{-\frac{1}{2}} + \lambda^{\frac{1}{2}})$ \eqref{eq:empirical_conditional_embedding_stochastic_convergence}. Thus, it remains to bound $\| \hat{u}_{x} - u_{x} \|_{\mathcal{H}_{\delta}}$ by a multiple of $\| \hat{\mu}_{Y | X = x} - \mu_{Y | X = x} \|_{\mathcal{H}_{\delta}}$.
			
			To this end, we first use \cref{thm:probability_is_embedding} and \cref{thm:empirical_probability_is_embedding} and to express the theoretical and empirical information as the negative log of the embedding, so that it is explicitly written as a function of $c \in \mathcal{Y}$ in $\mathcal{H}_{\delta}$ indexed by $x \in \mathcal{X}$,
			\begin{equation}
				\begin{aligned}
				u_{x}(c) &= - \log{p_{c}(x)} = -\log{\mu_{Y | X = x}(c)}, \\
				\hat{u}_{x}(c) &= - \log{\hat{p}_{c}(x)} = -\log{\hat{\mu}_{Y | X = x}(c)}.
			\end{aligned}
			\end{equation}
			
			Since $\log$ is a concave function, we have the property that $\log{a} - \log{b} \leq \frac{1}{b} (a - b)$. This allows us to bound $| \hat{u}_{x}(c) - u_{x}(c) |$ by $| \hat{\mu}_{Y | X = x}(c) - \mu_{Y | X = x}(c) |$ for all $c \in \mathbb{N}_{m}$,
			\begin{equation}
			\begin{aligned}
				| \hat{u}_{x}(c) - u_{x}(c) | &= | \log{\hat{\mu}_{Y | X = x}(c)} - \log{\mu_{Y | X = x}(c)} | \\
				&\leq \frac{1}{| \mu_{Y | X = x}(c) |} | \hat{\mu}_{Y | X = x}(c) - \mu_{Y | X = x}(c) | \\
				&\leq \alpha_{x} | \hat{\mu}_{Y | X = x}(c) - \mu_{Y | X = x}(c) |,
			\end{aligned}
			\end{equation}
			where we define $\alpha_{x} := \max_{c \in \mathbb{N}_{m}} \frac{1}{| \mu_{Y | X = x}(c) |}$, which is well defined as the conditional mean embedding is bounded. Since the \gls{RKHS} norm of bounded functions in $\mathcal{H}_{\delta}$ is simply the $\ell_{2}$-norm of their vector representations \eqref{eq:RKHS_norm_is_l2_norm}, we have
			\begin{equation}
			\begin{aligned}
				\| \hat{u}_{x} - u_{x} \|_{\mathcal{H}_{\delta}}^{2} &= \| \hat{\bvec{u}}_{x} - \bvec{u}_{x} \|_{\ell_{2}}^{2} \\
				&= \sum_{c = 1}^{m} | \hat{u}_{x}(c) - u_{x}(c) |^{2} \\
				&\leq \sum_{c = 1}^{m} \alpha_{x}^{2} | \hat{\mu}_{Y | X = x}(c) - \mu_{Y | X = x}(c) |^{2} \\
				&\leq \alpha_{x}^{2} \sum_{c = 1}^{m} | \hat{\mu}_{Y | X = x}(c) - \mu_{Y | X = x}(c) |^{2} \\
				&\leq \alpha_{x}^{2} \| \hat{\bm{\mu}}_{Y | X = x} - \bm{\mu}_{Y | X = x}  \|_{\ell_{2}}^{2} \\
				&\leq \alpha_{x}^{2} \| \hat{\mu}_{Y | X = x} - \mu_{Y | X = x} \|_{\mathcal{H}_{\delta}}^{2}.
			\end{aligned}
			\end{equation}
			
			Therefore, $\| \hat{u}_{x} - u_{x} \|_{\mathcal{H}_{\delta}} \leq \alpha_{x} \| \hat{\mu}_{Y | X = x} - \mu_{Y | X = x} \|_{\mathcal{H}_{\delta}}$, and \eqref{eq:information_entropy_bound} becomes
			\begin{equation}
			\begin{aligned}
				| \hat{h}(x) - h(x) | &\leq \| \hat{\mu}_{Y | X = x} \|_{\mathcal{H}_{\delta}} \| \hat{u}_{x} - u_{x} \|_{\mathcal{H}_{\delta}} + \| \hat{\mu}_{Y | X = x} - \mu_{Y | X = x} \|_{\mathcal{H}_{\delta}} \| u_{x} \|_{\mathcal{H}_{\delta}} \\
				&= \alpha_{x} \| \hat{\mu}_{Y | X = x} \|_{\mathcal{H}_{\delta}} \| \hat{\mu}_{Y | X = x} - \mu_{Y | X = x} \|_{\mathcal{H}_{\delta}} \\
				& \qquad + \| \hat{\mu}_{Y | X = x} - \mu_{Y | X = x} \|_{\mathcal{H}_{\delta}} \| u_{x} \|_{\mathcal{H}_{\delta}} \\
				&= ( \alpha_{x} \| \hat{\mu}_{Y | X = x} \|_{\mathcal{H}_{\delta}} + \| u_{x} \|_{\mathcal{H}_{\delta}} ) \| \hat{\mu}_{Y | X = x} - \mu_{Y | X = x} \|_{\mathcal{H}_{\delta}}.
			\end{aligned}
			\end{equation}
			
			Hence, with $\gamma(x) = \alpha_{x} \| \hat{\mu}_{Y | X = x} \|_{\mathcal{H}_{\delta}} + \| u_{x} \|_{\mathcal{H}_{\delta}}$, \cref{thm:pointwise_uniform_convergence} implies that $\hat{h}$ converges pointwise to $h$ at a stochastic rate of at least $O_{p}((n \lambda)^{-\frac{1}{2}} + \lambda^{\frac{1}{2}})$.
			\qed
		\end{proof}
	
	\newpage
	\section{Learning Theoretic Bounds}
	\label{app:learning_theoretic_bounds}
	
		In this section we derive \glspl{RCB} for \glspl{MCE}, and show that it can be used in conjunction with cross entropy loss to bound the expected risk with high probability.
		
		\subsection{Rademacher Complexity Bounds}
		\label{app:rademacher_complexity_theorems}
			
			Suppose a set of training data $\{x_{i}, y_{i}\}_{i = 1}^{n}$ is drawn from $\mathbb{P}_{X Y}$ in an \textit{iid} fashion. We denote the one hot encoded target labels of $\{y_{i}\}_{i = 1}^{n}$ by $\bvec{y}_{i} := \{\mathbb{1}_{c}(y_{i})\}_{c = 1}^{m} \in \{0, 1\}^{m}$ and $\bvec{Y} := \begin{bmatrix} \bvec{y}_{1} & \bvec{y}_{2} & \cdots & \bvec{y}_{n} \end{bmatrix}^{T} \in \{0, 1\}^{n \times m}$. Similarly, let $\bvec{y} \in \{0, 1\}^{m}$ denote the one hot encoded target labels for a generic label $y \in \mathcal{Y}$. Let $k_{\theta} : \mathcal{X} \times \mathcal{X} \to [0, \infty)$ be a family of positive definite kernels indexed by $\theta \in \Theta$. As before, we define the shorthand notation for the gram matrices $K_{\theta} := \{k_{\theta}(x_{i}, x_{j}) : i \in \mathbb{N}_{n}, j \in \mathbb{N}_{n}\}$ and $\bvec{k}_{\theta}(x) := \{k_{\theta}(x_{i}, x) : i \in \mathbb{N}_{n}\}$, and $\lambda$ denotes the regularization hyperparameter of the conditional mean embedding \eqref{eq:empirical_conditional_embedding}. The \gls{MCE} has a predictor form $\hat{\bvec{p}}(x) = \bvec{f}_{\theta, \lambda}(x)$ \eqref{eq:empirical_decision_probability_vector} defined by
			\begin{equation}
				\bvec{f}_{\theta, \lambda}(x) := \bvec{Y}^{T} (K_{\theta} + n \lambda I)^{-1} \bvec{k}_{\theta}(x),
			\label{eq:predictor}
			\end{equation}
			where each entry of the predictor $\bvec{f}_{\theta, \lambda}(x)$ is the decision probability estimate for $p_{c}(x)$. This defines the function class of the predictor over the kernel family and a set of regularization hyperparameters for any set of training observations $\{x_{i}, y_{i}\}_{i = 1}^{n}$,
			\begin{equation}
				F_{n}(\Theta, \Lambda) := \{ \bvec{f}_{\theta, \lambda}(x) : \theta \in \Theta, \lambda \in \Lambda \}.
			\label{eq:predictor_class}
			\end{equation}
			
			The predictor form \eqref{eq:predictor} is linear in the reproducing kernel Hilbert space $\mathcal{H}_{k_{\theta}}$ induced by $k_{\theta}$ in the sense that
			\begin{equation}
			\begin{aligned}
				\bvec{f}_{\theta, \lambda}(x) &:= W_{\theta, \lambda}^{T} \phi_{\theta}(x), \\
				W_{\theta, \lambda} &:= \Phi_{\theta} (K_{\theta} + n \lambda I)^{-1} \bvec{Y},
			\end{aligned}
			\label{eq:linear_predictor}
			\end{equation}
			where we decompose $\bvec{k}_{\theta}(x) = \Phi_{\theta}^{T} \phi_{\theta}(x)$ by the reproducing property. By \cref{thm:empirical_probability_is_embedding}, $\bvec{f}_{\theta, \lambda}(x) = \hat{\bvec{p}}_{\theta, \lambda}(x) = \hat{\mu}^{(\theta, \lambda)}_{Y | X = x} = \hat{\mathcal{U}}^{(\theta, \lambda)}_{Y | X} \phi_{\theta}(x)$. Therefore, we have that $\hat{\mathcal{U}}^{(\theta, \lambda)}_{Y | X} \equiv W_{\theta, \lambda}^{T}$. Throughout this paper, inner products are defined in the Hilbert-Schmidt sense, which induces the Hilbert-Schmidt norm $\| \cdot \|_{HS}$ and generalises the Frobenius inner product with induced norm $\| \cdot \|_{\mathrm{tr}}$ for finite dimensional operators. Nevertheless, while they refer to the same quantity, we will use the standard notations $ \| \hat{\mathcal{U}}^{\theta, \lambda}_{Y | X} \|_{HS}$ as per the literature in Hilbert space embeddings and $\| W_{\theta, \lambda} \|_{\mathrm{tr}}$ as per the literature for linear classifiers.
			
			\begin{theorem}[\gls{MCE} Rademacher Complexity Bound]
				\label{thm:rademacher_complexity_bound}
				Suppose that the trace norm $\| W_{\theta, \lambda} \|_{\mathrm{tr}} \leq \rho$ is bounded for all $\theta \in \Theta, \lambda \in \Lambda$. Further suppose that the canonical feature map is bounded in \gls{RKHS} norm $\| \phi_{\theta}(x) \|_{\mathcal{H}_{k_{\theta}}}^{2} = k_{\theta}(x, x) \leq \alpha^{2}$, $\alpha > 0$, for all $x \in \mathcal{X}, \theta \in \Theta$. For any set of training observations $\{x_{i}, y_{i}\}_{i = 1}^{n}$, the Rademacher complexity of the class of \glspl{MCE} $F_{n}(\Theta, \Lambda)$ \eqref{eq:predictor_class} defined over $\theta \in \Theta, \lambda \in \Lambda$ is bounded by
				\begin{equation}
					\mathcal{R}_{n}(F_{n}(\Theta, \Lambda)) \leq 2 \alpha \rho.
				\label{eq:rademacher_complexity_bound}
				\end{equation}
			\end{theorem}
	
			\begin{proof}
				The Rademacher complexity \citep[Definition 2]{bartlett2002rademacher} of the function class $F_{n}(\Theta, \Lambda)$ is 
				\begin{equation}
				\begin{aligned}
					\mathcal{R}_{n}(F_{n}(\Theta, \Lambda)) :=& \mathbb{E}\bigg[\sup_{\theta \in \Theta, \lambda \in \Lambda} \Big\| \frac{2}{n} \sum_{i = 1}^{n} \sigma_{i}\bvec{f}_{\theta, \lambda}(X_{i}) \Big\|\bigg] \\
					=& \frac{2}{n} \mathbb{E}\bigg[\sup_{\theta \in \Theta, \lambda \in \Lambda} \Big\| \sum_{i = 1}^{n} \sigma_{i}\bvec{f}_{\theta, \lambda}(X_{i}) \Big\|\bigg],
				\end{aligned}
				\end{equation}
				where $\sigma_{i}$ are \textit{iid} Rademacher random variables, taking values in $\{-1, 1\}$ with equal probability, and $X_{i}$ are \textit{iid} random variables from the same distribution $\mathbb{P}_{X}$ as our training data. We further define $\bm{\sigma} := \{\sigma_{i}\}_{i = 1}^{n}$.
				
				We first bound the term inside the suprenum using the Cauchy Schwarz inequality,
				\begin{equation}
				\begin{aligned}
					\Big\| \sum_{i = 1}^{n} \sigma_{i}\bvec{f}_{\theta, \lambda}(X_{i}) \Big\| &= \Big\| \sum_{i = 1}^{n} \sigma_{i} W_{\theta, \lambda}^{T} \phi_{\theta}(X_{i}) \Big\| \\
					&= \Big\| W_{\theta, \lambda}^{T} \bm{\Phi}_{\theta} \bm{\sigma} \Big\| \\
					&\leq \| W_{\theta, \lambda} \|_{\mathrm{tr}} \| \| \bm{\Phi}_{\theta} \bm{\sigma} \| \\
					&\leq \| W_{\theta, \lambda} \|_{\mathrm{tr}} \| \| \bm{\Phi}_{\theta}^{T} \|_{\mathrm{tr}} \| \bm{\sigma} \| \\
					&= \| W_{\theta, \lambda} \|_{\mathrm{tr}} \| \| \bm{\Phi}_{\theta} \|_{\mathrm{tr}} \| \bm{\sigma} \|,
				\end{aligned}
				\end{equation}
				where we define the random operator $\bm{\Phi}_{\theta} := \begin{bmatrix} \phi(X_{1}) & \phi(X_{2}) & \cdots & \phi(X_{n}) \end{bmatrix}$. Note that this is distinct from $\Phi_{\theta}$, whose columns are the canonical \gls{RKHS} features at the training observations and is not random. Now, random or not, entries of $\bm{\sigma} := \{\sigma_{i}\}_{i = 1}^{n}$ are either $-1$ or $1$, so its norm is simply $\| \bm{\sigma} \| = \sqrt{n}$. We can then also compute the trace norm of the other random component $\bm{\Phi}_{\theta}$,
				\begin{equation}
				\begin{aligned}
					\| \bm{\Phi}_{\theta} \|_{\mathrm{tr}} :=& \sqrt{\mathrm{trace}(\bm{\Phi}_{\theta}^{T} \bm{\Phi}_{\theta})} \\
					=& \sqrt{\mathrm{trace}(\bvec{K}_{\theta})} \\
					=& \sqrt{\sum_{i = 1}^{n} k_{\theta}(X_{i}, X_{i})} \\
					=& \sqrt{n} \sqrt{ \frac{1}{n} \sum_{i = 1}^{n} k_{\theta}(X_{i}, X_{i})} \\
					\leq& \sqrt{n} \sqrt{ \frac{1}{n} \sum_{i = 1}^{n} \alpha^{2}} \\
					=& \sqrt{n} \alpha,
				\end{aligned}
				\end{equation}
				where the inequality comes from the assertion that $k_{\theta}(x, x) \leq \alpha^{2}$ for all $x \in \mathcal{X}, \theta \in \Theta$. This bounds all the random components in the expectation by a constant, so that later the expectation can vanish. 
				
				Using the assertion that $\| W_{\theta, \lambda} \|_{\mathrm{tr}} \leq \rho$ for all $\theta \in \Theta, \lambda \in \Lambda$, we can now bound the Rademacher complexity,
				\begin{equation}
				\begin{aligned}
					\mathcal{R}_{n}(F_{n}(\Theta, \Lambda)) &= \frac{2}{n} \mathbb{E}\bigg[\sup_{\theta \in \Theta, \lambda \in \Lambda} \Big\| \sum_{i = 1}^{n} \sigma_{i}\bvec{f}_{\theta, \lambda}(X_{i}) \Big\|\bigg] \\
					&\leq \frac{2}{n} \mathbb{E}\bigg[\sup_{\theta \in \Theta, \lambda \in \Lambda} \|  W_{\theta, \lambda} \|_{\mathrm{tr}} \| \| \bm{\Phi}_{\theta} \|_{\mathrm{tr}} \| \bm{\sigma} \| \bigg] \\
					&= \frac{2}{n} \sqrt{n} \mathbb{E}\bigg[\sup_{\theta \in \Theta, \lambda \in \Lambda} \|  W_{\theta, \lambda} \|_{\mathrm{tr}} \| \| \bm{\Phi}_{\theta} \|_{\mathrm{tr}} \bigg] \\
					&\leq \frac{2}{n} \sqrt{n} \sqrt{n} \alpha \mathbb{E}\bigg[\sup_{\theta \in \Theta, \lambda \in \Lambda} \|  W_{\theta, \lambda} \|_{\mathrm{tr}} \bigg] \\
					&\leq 2 \alpha \mathbb{E}\bigg[\sup_{\theta \in \Theta, \lambda \in \Lambda} \|  W_{\theta, \lambda} \|_{\mathrm{tr}} \bigg] \\
					&= 2 \alpha \sup_{\theta \in \Theta, \lambda \in \Lambda} \|  W_{\theta, \lambda} \|_{\mathrm{tr}} \\
					&\leq 2 \alpha \rho.
				\end{aligned}
				\end{equation}
				\qed
			\end{proof}
	
			\Cref{thm:rademacher_complexity_bound} provides a generic Rademacher complexity bound for any type of \gls{MCE} with a bounded positive definite kernel and bounded trace norm. One of the most widely used kernels in practice are the family of stationary kernels. We provide a more specific bound for the case of stationary kernels below.
			
			\begin{corollary}[Rademacher Complexity Bound for Stationary Kernels]
				\label{thm:rademacher_complexity_stationary_kernels_bound}
				Suppose that the trace norm $\| W_{\theta, \lambda} \|_{\mathrm{tr}} \leq \rho$ is bounded for all $\theta \in \Theta, \lambda \in \Lambda$. Suppose that $k_{\theta}$ is a family of positive definite stationary kernels. That is, $k_{\theta} (x, x') = \tilde{k}_{\theta}( \| x - x' \| )$ for some real-valued function $\tilde{k} : [0, \infty) \to [0, \infty)$. Select $\tilde{\theta} \in \Theta$ and define $\Theta(\tilde{\theta})$ such that $k_{\theta}(0, 0) \leq k_{\tilde{\theta}}(0, 0)$ for all $\theta \in \Theta(\tilde{\theta})$. For any $\tilde{\theta} \in \Theta$ and set of training observations $\{x_{i}, y_{i}\}_{i = 1}^{n}$, the Rademacher complexity of the resulting class of \glspl{MCE} $F_{n}(\Theta(\tilde{\theta}), \Lambda)$ defined over $\theta \in \Theta(\tilde{\theta}), \lambda \in \Lambda$ is bounded by
				\begin{equation}
					\mathcal{R}_{n}(F_{n}(\Theta(\tilde{\theta}), \Lambda)) \leq 2 \rho \sqrt{k_{\tilde{\theta}}(0, 0)}.
				\end{equation}
			\end{corollary}
	
			\begin{proof}
				Observe that $k_{\tilde{\theta}}(0, 0)$ is an upper bound for $k_{\theta}(x, x)$ for all $x \in \mathcal{X}$ and $\theta \in \Theta$,
				\begin{equation}
				\begin{aligned}
					k_{\theta}(x, x) = \tilde{k}_{\theta}( \| x - x \| ) = \tilde{k}_{\theta}( \| 0 \| ) = k_{\theta}(0, 0) \leq k_{\tilde{\theta}}(0, 0).
				\end{aligned}
				\end{equation}
				
				We simply choose $\alpha^{2} = k_{\tilde{\theta}}(0, 0)$ in \cref{thm:rademacher_complexity_bound}.
				\qed
			\end{proof}

			\Cref{thm:rademacher_complexity_stationary_kernels_bound} motivates the choice $\alpha^{2}(\theta) = k_{\theta}(0, 0) = \sigma_{f}^{2}$ for stationary radial basis type kernels such as the Gaussian or Mat\'{e}rn kernels, where $\sigma_{f}$ is the sensitivity \citep{rasmussen2006gaussian} of the stationary kernel, which we employ in our learning algorithm when the kernel is stationary.
	
		\subsection{Expected Risk Bounds}
		\label{app:expected_risk_bounds}
		
			In order to quantify the performance of the \gls{MCE}, we specify a loss function $\mathcal{L} : \mathcal{Y} \times \mathcal{A} \to [0, \infty)$, where $\mathcal{L}(y, f(x))$ measures the loss of a decision function $f : \mathcal{X} \to \mathcal{A}$ on a paired example $x \in \mathcal{X}$ and label $y \in \mathcal{Y}$. In the \gls{MCE} context, the decision function is $\bvec{f}_{\theta, \lambda} : \mathcal{X} \to \mathbb{R}^{m}$, with $\mathcal{A} = \mathbb{R}^{m}$ and $\mathcal{Y} = \mathbb{N}_{m}$. The loss function is to capture the desire for $\bvec{y}^{T} \bvec{f}_{\theta, \lambda}(x) = f_{y}^{(\theta, \lambda)}(x)$ to be high for all likely test points $x \in \mathcal{X}$ and $y \in \mathcal{Y}$.
			
			A suitable choice of the loss function in the probabilistic multiclass classification context is the cross entropy loss,
			\begin{equation}
				\mathcal{L}(y, \bvec{f}(x)) := - \log{\bvec{y}^{T} \bvec{f}(x)} = - \log{f_{y}(x)},
			\end{equation}
			where $\bvec{f}(x)$ are the inferred decision probability estimates of each class for the example $x \in \mathcal{X}$. Since logarithms explode at zero, in practice the probability estimate is often clipped from below at a predetermined threshold $\epsilon \in (0, 1)$. Furthermore, it is also convenient to clip the probability estimate from above at one to avoid negative losses. Consequently, with the notation $[\;\cdot\;]_{\epsilon}^{1} := \min\{\max\{\;\cdot\;, \epsilon\}, 1\}$, we define the effective cross entropy loss as
			\begin{equation}
				\mathcal{L}_{\epsilon}(y, \bvec{f}(x)) := - \log{ [\bvec{y}^{T} \bvec{f}(x)]_{\epsilon}^{1} } = - \log{ [f_{y}(x)]_{\epsilon}^{1} }.
			\label{eq:cross_entropy_loss}
			\end{equation}
			
			In this way, our cross entropy loss \eqref{eq:cross_entropy_loss} is both bounded and positive. In our subsequent analysis, we require that our loss function has an image in $[0, 1]$. To do this, we simply rescale the loss function by dividing it by its largest value,
			\begin{equation}
			\begin{aligned}
				\bar{\mathcal{L}}_{\epsilon}(y, \bvec{f}(x)) &:= \frac{1}{M_{\epsilon}} \mathcal{L}_{\epsilon}(y, \bvec{f}(x)) = - \frac{1}{M_{\epsilon}} \log{ [f_{y}(x)]_{\epsilon}^{1} }, \\
				M_{\epsilon} &:= - \log{\epsilon}.
			\end{aligned}
			\label{eq:normalized_cross_entropy_loss}
			\end{equation}
			
			We will refer to \eqref{eq:normalized_cross_entropy_loss} as the normalized cross entropy loss. We then further define the centered normalized cross entropy loss,
			\begin{equation}
				\tilde{\mathcal{L}}_{\epsilon}(y, \bvec{f}(x)) := \bar{\mathcal{L}}_{\epsilon}(y, \bvec{f}(x)) - \bar{\mathcal{L}}_{\epsilon}(y, \bvec{0}) = - \frac{1}{M_{\epsilon}} \log{ [f_{y}(x)]_{\epsilon}^{1} } - 1.
			\label{eq:centered_normalized_cross_entropy_loss}
			\end{equation}
			
			With the normalized cross entropy loss \eqref{eq:normalized_cross_entropy_loss} as our loss function, we now employ Theorem 8 of \cite{bartlett2002rademacher} for this loss and provide a bound for the expected normalized cross entropy loss for an unseen test example.
			
			\begin{lemma}[Expected Risk Bound]
				\label{thm:expected_normalized_cross_entropy_loss_bound}
				For any integer $n \in \mathbb{N}_{+}$ and any set of training observations $\{x_{i}, y_{i}\}_{i = 1}^{n}$, with probability $1 - \beta$ over \textit{iid} samples $\{X_{i}, Y_{i}\}_{i = 1}^{n}$ of length $n$ from $\mathbb{P}_{X Y}$, every $f \in F_{n}(\Theta, \Lambda)$ satisfies
				\begin{equation}
					\frac{1}{M_{\epsilon}} \mathbb{E}[\mathcal{L}_{\epsilon}(Y, f(X))] \leq \frac{1}{n M_{\epsilon}} \sum_{i = 1}^{n} \mathcal{L}_{\epsilon}(Y_{i}, f(X_{i})) + \mathcal{R}_{n}(\tilde{\mathcal{L}}_{\epsilon} \circ F_{n}(\Theta, \Lambda)) + \sqrt{\frac{8}{n} \log{\frac{2}{\beta}}}.
				\label{eq:expected_loss_bound}
				\end{equation}
			\end{lemma}
	
			\begin{proof}
				Since $\bar{\mathcal{L}}_{\epsilon} : \mathcal{Y} \times \mathcal{A} \to [0, 1]$ has a unit range and dominates itself, $\bar{\mathcal{L}}_{\epsilon}(y, f(x)) \leq \bar{\mathcal{L}}_{\epsilon}(y, f(x))$, the result follows directly from Theorem 8 of \cite{bartlett2002rademacher}. We then use the definition \eqref{eq:normalized_cross_entropy_loss} for the normalized cross entropy loss.
				\qed
			\end{proof}
			
			Equivalently, by definition \eqref{eq:predictor_class}, this result holds for $f = \bvec{f}_{\theta, \lambda}(x)$ for every $\theta \in \Theta, \lambda \in \Lambda$. The bound \eqref{eq:expected_loss_bound} involves the Rademacher complexity $\mathcal{R}_{n}(\tilde{\mathcal{L}}_{\epsilon} \circ F_{n}(\Theta, \Lambda))$ of the centered normalized cross entropy loss applied onto the class of functions $F_{n}(\Theta, \Lambda)$, and not just the Rademacher complexity $\mathcal{R}_{n}(F_{n}(\Theta, \Lambda))$ of the class of functions $F_{n}(\Theta, \Lambda)$ itself. In \cref{thm:rademacher_complexity_bound}, we have bounded the latter. We now proceed to bound the former with the latter \eqref{eq:rademacher_complexity_bound}, so that the upper bound in \cref{thm:expected_normalized_cross_entropy_loss_bound} can be written in terms of the latter.
			
			\begin{lemma}[Rademacher Complexity Bound with Cross Entropy Loss]
				\label{thm:rademacher_complexity_bound_with_cross_entropy_loss}
				For any integer $n \in \mathbb{N}_{+}$ and any set of training observations $\{x_{i}, y_{i}\}_{i = 1}^{n}$, the Rademacher complexity of the class of cross entropy loss applied onto the \gls{MCE} is bounded by
				\begin{equation}
					\mathcal{R}_{n}(\tilde{\mathcal{L}}_{\epsilon} \circ F_{n}(\Theta, \Lambda)) \leq 2 \frac{1}{\epsilon \log{\frac{1}{\epsilon}}} \mathcal{R}_{n}(F_{n}(\Theta, \Lambda)),
				\label{eq:composed_rademacher_complexity_bound}
				\end{equation}
				where $\tilde{\mathcal{L}}_{\epsilon} \circ F_{n}(\Theta, \Lambda) := \{(x, y) \mapsto \tilde{\mathcal{L}}_{\epsilon}(y, \bvec{f}_{\theta, \lambda}(x)) : \theta \in \Theta, \lambda \in \Lambda\}$.
			\end{lemma}
	
			\begin{proof}
				Let $\tilde{\psi}(z) := - \frac{1}{M_{\epsilon}} \log{[z]_{\epsilon}^{1}} - 1$ so that $\tilde{\psi} : \mathbb{R} \to \mathbb{R}$ satisfies $\tilde{\psi}(0) = 0$. Then, the centered normalized cross entropy loss can be written as $\tilde{\mathcal{L}}_{\epsilon}(y, \bvec{f}(x)) = \tilde{\psi}(f_{y}(x))$. In particular, $\tilde{\psi}(z)$ is piecewise differentiable. We proceed to show that $\tilde{\psi}$ is Lipschitz by showing that the suprenum of its absolute derivative over all piecewise regions is finite, and thus infer its Lipschitz constant.
				
				The real-valued function $\tilde{\psi}$ can be split into three piecewise regions over the real domain,
				\begin{equation}
					\tilde{\psi}(z) = \begin{cases}
					0, \qquad &z \in (-\infty, \epsilon], \\
					- \frac{1}{M_{\epsilon}}\log{z} - 1, \qquad & z \in (\epsilon, 1), \\
					-1, & z \in [1, \infty).
				\end{cases}
				\end{equation}
				
				The derivative over the regions $z \in (-\infty, \epsilon]$ and $z \in [1, \infty)$ is thus $0$ and the local Lipschitz constant over that region is thus $0$. We then focus on the other region,
				\begin{equation}
				\begin{aligned}
					\sup_{z \in (\epsilon, 1)} | \tilde{\psi}'(z) | = \sup_{z \in (\epsilon, 1)} \bigg| - \frac{1}{z M_{\epsilon}} \bigg| = \sup_{z \in (\epsilon, 1)} \frac{1}{z M_{\epsilon}} = \frac{1}{\epsilon M_{\epsilon}} = \frac{1}{\epsilon \log{\frac{1}{\epsilon}}}.
				\end{aligned}
				\end{equation}
				
				Thus, $\tilde{\psi}$ is Lipschitz with a Lipschitz constant of $L_{\tilde{\psi}} = \frac{1}{\epsilon \log{\frac{1}{\epsilon}}}$.
				
				For a given general loss function $\mathcal{L}$, \citet[Corollary 3.17]{ledoux2013probability} proved that if there exists a Lipschitz real-valued function $\psi : \mathbb{R} \to \mathbb{R}$, $\psi(0) = 0$, with constant $L_{\psi}$ such that $\mathcal{L}(y, f(x)) = \psi(f_{y}(x))$, then $\mathcal{R}_{n}(\mathcal{L} \circ F) \leq 2 L_{\psi} \mathcal{R}_{n}(F)$ for any class of functions $F$. This result is also described in \citet[Theorem 12.4]{bartlett2002rademacher}.
				
				Applying this result to our loss function with $\mathcal{L} = \tilde{\mathcal{L}}_{\epsilon}$ with $\psi = \tilde{\psi}$ and $F = F_{n}(\Theta, \Lambda)$, we have $\mathcal{R}_{n}(\tilde{\mathcal{L}}_{\epsilon} \circ F_{n}(\Theta, \Lambda)) \leq 2 L_{\tilde{\psi}} \mathcal{R}_{n}(F_{n}(\Theta, \Lambda))$, which proves the claim.
				\qed
			\end{proof}
		
			The bound \eqref{eq:composed_rademacher_complexity_bound} in \cref{thm:rademacher_complexity_bound_with_cross_entropy_loss} will be the bridge that relates the expected cross entropy loss over our function class to the Rademacher complexity of our function class. We now proceed to state the main theorem which forms the backbone of our learning algorithm for the \gls{MCE}.
			
			\begin{lemma}[\gls{MCE} $\epsilon$-General Expected Risk Bound]
				\label{thm:expected_loss_bound_for_multiclass_conditional_embedding}
				Suppose that the trace norm $\| W_{\theta, \lambda} \|_{\mathrm{tr}} \leq \rho$ is bounded for all $\theta \in \Theta, \lambda \in \Lambda$. Further suppose that the canonical feature map $\| \phi_{\theta}(x) \|_{\mathcal{H}_{k_{\theta}}}^{2} = k_{\theta}(x, x) \leq \alpha^{2}$, $\alpha > 0$, is bounded in \gls{RKHS} norm for all $x \in \mathcal{X}, \theta \in \Theta$. For any integer $n \in \mathbb{N}_{+}$ and any set of training observations $\{x_{i}, y_{i}\}_{i = 1}^{n}$, with probability of at least $1 - \beta$ over \textit{iid} samples $\{X_{i}, Y_{i}\}_{i = 1}^{n}$ of length $n$ from $\mathbb{P}_{X Y}$, every $f \in F_{n}(\Theta, \Lambda)$ satisfies
				\begin{equation}
					\frac{1}{M_{\epsilon}} \mathbb{E}[\mathcal{L}_{\epsilon}(Y, f(X))] \leq \frac{1}{n M_{\epsilon}} \sum_{i = 1}^{n} \mathcal{L}_{\epsilon}(Y_{i}, f(X_{i})) + 4 \frac{1}{\epsilon \log{\frac{1}{\epsilon}}} \alpha \rho + \sqrt{\frac{8}{n} \log{\frac{2}{\beta}}},
				\label{eq:expected_loss_bound_for_multiclass_conditional_embedding}
				\end{equation}
				for any $\epsilon \in (0, 1)$. Equivalently, the bound \eqref{eq:expected_loss_bound_for_multiclass_conditional_embedding} holds for $f = \bvec{f}_{\theta, \lambda}(x)$ for every $\theta \in \Theta, \lambda \in \Lambda$.
			\end{lemma}
		
			\begin{proof}
				From \cref{thm:rademacher_complexity_bound}, we have $\mathcal{R}_{n}(F_{n}(\Theta, \Lambda)) \leq 2 \alpha \rho$. Further, from \cref{thm:rademacher_complexity_bound_with_cross_entropy_loss}, we have $\mathcal{R}_{n}(\tilde{\mathcal{L}}_{\epsilon} \circ F_{n}(\Theta, \Lambda)) \leq 2 \frac{1}{\epsilon \log{\frac{1}{\epsilon}}} \mathcal{R}_{n}(F_{n}(\Theta, \Lambda))$. These are both deterministic inequalities, leading to $\mathcal{R}_{n}(\tilde{\mathcal{L}}_{\epsilon} \circ F_{n}(\Theta, \Lambda)) \leq 4 \frac{1}{\epsilon \log{\frac{1}{\epsilon}}} \alpha \rho$. We then apply this inequality to \cref{thm:expected_normalized_cross_entropy_loss_bound}, which proves the claim.
				\qed
			\end{proof}
	
			Similar to many learning theoretic bounds, the expected risk bound \eqref{eq:expected_loss_bound_for_multiclass_conditional_embedding} is composed of three qualitatively different terms. The first term is a training loss or data fit term, which is a measure of how poorly the decision function $f$ is performing on a given training dataset. The second term is a model complexity or regularization term, which measures how complicated the model is. In this case, the model complexity is measured by the Rademacher complexity, which captures the expressiveness of the function class by quantifying how well the function class is able to shatter noise. The third term is a statistical constant which plays no specific role to the function class.
			
			We will eventually be minimizing the first two terms over some class of functions $f \in F_{n}(\Theta, \Lambda)$ with some approach, as a proxy to minimizing the actual expected risk. It would be fruitful to develop an intuition for the tightness of the bound from the contributions of the training loss term and the model complexity term. Since, like the expected loss, the training loss term is always in the unit range $[0, 1]$, we focus on understanding the tightness of the bound contributed from the complexity term.
			
			Consider a clipped cross entropy loss with either a very small clipping factor $\epsilon \approx 0$, or a very large clipping factor $\epsilon \approx 1$. In these scenarios, $\epsilon \log{\frac{1}{\epsilon}}$ would be very small, so that the coefficient on the complexity term would then be very large, regardless of what the complexity bound factors $\alpha$ and $\rho$ are. As a result, intuitively, this bound is unlikely to be tight due to the large coefficient on the complexity term.
			
			Consequently, it would then be natural to consider a middle-ground choice of the cross entropy loss where this bound is the most tight by varying $\epsilon \in (0, 1)$. Since $\epsilon \log{\frac{1}{\epsilon}}$ is maximized at $\epsilon = \frac{1}{e}$ for a maximal value of $\frac{1}{e}$, such a choice in the clipping factor would indeed yield the tightest bound for the complexity bound in terms of the bounding slack of the result stated in \cref{thm:rademacher_complexity_bound_with_cross_entropy_loss}.
			
			This is great news for the complexity term. What about the training loss term? Intuition tells us that, with a clipping factor of $\epsilon = e^{-1}$ that is slightly more than a third of the way into the interval $(0, 1)$ from zero, the classifier is not being penalised as strongly for assigning probabilities smaller than $e^{-1}$ to observed classes as compared to very small values of $\epsilon$. Furthermore, beyond the clipping point, assigning even lower probabilities to the observations does not result in a higher loss. In practice, the cross entropy loss is renowned for its rapidly growing penalty as the probability assignment gets lower, which is advantageous when using a gradient based optimization scheme. In this case, the gradients are large in magnitude and the classifier can adjust and fix these assignment errors relatively quickly. In other words, by using a slightly larger clipping factor than usual, we have seemingly lost the faster convergence properties from using a cross entropy loss.
			
			Nevertheless, observe that for such a clipping factor $\epsilon = e^{-1}$, the normalization constant becomes $M_{e^{-1}} = - \log{\frac{1}{e}} = 1$, so that it is effectively removed. Furthermore, we also have the following simple upper bound for the cross entropy loss clipped at $\epsilon = e^{-1}$,
			\begin{equation}
				\bar{\mathcal{L}}_{e^{-1}}(y, f(x)) = \mathcal{L}_{e^{-1}}(y, f(x)) \leq \mathcal{L}_{\epsilon}(y, f(x)) \quad \forall \epsilon \in (0, e^{-1}), x \in \mathcal{X}, y \in \mathcal{Y}.
			\label{eq:loss_inequality}
			\end{equation}
			
			To see why inequality \eqref{eq:loss_inequality} holds, note that $[f_{y}(x)]_{\epsilon}^{1} \leq [f_{y}(x)]_{e^{-1}}^{1}$ holds for all $\epsilon \in (0, e^{-1}), x \in \mathcal{X}, y \in \mathcal{Y}$. Applying negative log to both sides yields the inequality from definition \eqref{eq:cross_entropy_loss}.
			
			Therefore, we propose to choose $\epsilon = e^{-1}$, and then replace replace $\mathcal{L}_{e^{-1}}$ with $\mathcal{L}_{\epsilon}$ for some new generic $\epsilon \in (0, e^{-1})$ much smaller than $e^{-1}$ on the training loss terms. In this way, we still maintain an upper bound for the training loss term. While this bound would not necessarily be tight for high training losses, the gradients from the high training loss would drive the system to a lower training loss, where the bound would become tight again as equality holds in \eqref{eq:loss_inequality} whenever $f_{y}(x) \geq e^{-1}$.
			
			The above intuition motivates the result in the following theorem.
			
			\begin{theorem}[\gls{MCE} $\epsilon$-Specific Expected Risk Bound]
				\label{thm:specific_expected_loss_bound_for_multiclass_conditional_embedding}
				Suppose that the trace norm $\| W_{\theta, \lambda} \|_{\mathrm{tr}} \leq \rho$ is bounded for all $\theta \in \Theta, \lambda \in \Lambda$. Further suppose that the canonical feature map $\| \phi_{\theta}(x) \|_{\mathcal{H}_{k_{\theta}}}^{2} = k_{\theta}(x, x) \leq \alpha^{2}$, $\alpha > 0$, is bounded in \gls{RKHS} norm for all $x \in \mathcal{X}, \theta \in \Theta$. For any integer $n \in \mathbb{N}_{+}$ and any set of training observations $\{x_{i}, y_{i}\}_{i = 1}^{n}$, with probability of at least $1 - \beta$ over \textit{iid} samples $\{X_{i}, Y_{i}\}_{i = 1}^{n}$ of length $n$ from $\mathbb{P}_{X Y}$, every $f \in F_{n}(\Theta, \Lambda)$ satisfies
				\begin{equation}
					\mathbb{E}[\mathcal{L}_{e^{-1}}(Y, f(X))] \leq \frac{1}{n} \sum_{i = 1}^{n} \mathcal{L}_{\epsilon}(Y_{i}, f(X_{i})) + 4 e \; \alpha \rho + \sqrt{\frac{8}{n} \log{\frac{2}{\beta}}},
				\label{eq:specific_expected_loss_bound_for_multiclass_conditional_embedding}
				\end{equation}
				for any $\epsilon \in (0, e^{-1})$. Equivalently, the bound \eqref{eq:specific_expected_loss_bound_for_multiclass_conditional_embedding} holds for $f = \bvec{f}_{\theta, \lambda}(x)$ for every $\theta \in \Theta, \lambda \in \Lambda$.
			\end{theorem}
	
			\begin{proof}
				We first apply \cref{thm:expected_loss_bound_for_multiclass_conditional_embedding} with $\epsilon = e^{-1}$,
				\begin{equation}
					\mathbb{E}[\mathcal{L}_{e^{-1}}(Y, f(X))] \leq \frac{1}{n} \sum_{i = 1}^{n} \mathcal{L}_{e^{-1}}(Y_{i}, f(X_{i})) + 4 e \; \alpha \rho + \sqrt{\frac{8}{n} \log{\frac{2}{\beta}}}.
				\end{equation}
				
				For any $\epsilon \in (0, e^{-1})$, the inequality $\mathcal{L}_{e^{-1}}(Y_{i}, f(X_{i})) \leq \mathcal{L}_{\epsilon}(Y_{i}, f(X_{i}))$ holds almost surely (a.s.) due to the deterministic inequality \eqref{eq:loss_inequality}. These sets of inequalities together proves the claim.
				\qed
			\end{proof}
	
		\subsection{Expected Risk Bounds for Hyperparameter Learning}
		\label{app:expected_risk_bounds_for_hyperparameter_learning}

			We are now ready to use the result of \cref{thm:specific_expected_loss_bound_for_multiclass_conditional_embedding} to derive a specific expected risk bound for a given choice of hyperparameters $\theta \in \Theta$ and $\lambda \in \Lambda$ of the \gls{MCE}, and not just for a general set of hyperparameters. We focus on kernels $k_{\theta}$ that are bounded over the domain $\mathcal{X}$ in the sense that for each $\theta \in \Theta$, $k_{\theta}(x, x) < \infty$ for all $x \in \mathcal{X}$.
			
			For some kernel hyperparameters $\tilde{\theta} \in \Theta$ and regularization hyperparameter $\tilde{\lambda} \in \Lambda$, we construct a subset of hyperparameters (kernel hyperparameters and regularization hyperparameters) $\Xi(\tilde{\theta}, \tilde{\lambda}) \subseteq \Theta \times \Lambda$ such that
			\begin{equation}
			\begin{aligned}
				\Xi(\tilde{\theta}, \tilde{\lambda}) := \{ (\theta, \lambda) \in \Theta \times \Lambda :\; &\| W_{\theta, \lambda} \|_{\mathrm{tr}} \leq \| W_{\tilde{\theta}, \tilde{\lambda}} \|_{\mathrm{tr}}, \\
				&\sup_{x \in \mathcal{X}} k_{\theta}(x, x) \leq \alpha^{2}(\tilde{\theta}) := \sup_{x \in \mathcal{X}} k_{\tilde{\theta}}(x, x) \}.
			\end{aligned}
			\label{eq:hyperparameter_subset_definition}
			\end{equation}
			
			Clearly, this subset is non-empty, since $(\tilde{\theta}, \tilde{\lambda}) \in \Xi(\tilde{\theta}, \tilde{\lambda})$ is itself an element of this subset. Note that $\alpha : \Theta \to \mathbb{R}_{+}$ must necessarily exist as the kernel family $k_{\theta}$ is assumed to be bounded over the domain $\mathcal{X}$. The class of \glspl{MCE} over this subset of hyperparameters is 
			\begin{equation}
				F_{n}(\Xi(\tilde{\theta}, \tilde{\lambda})) := \{\bvec{f}_{\theta, \lambda}(x) : (\theta, \lambda) \in \Xi(\tilde{\theta}, \tilde{\lambda}) \}.
			\end{equation}
			
			Thus, we can assert that the trace norm $\| W_{\theta, \lambda} \|_{\mathrm{tr}} \leq \rho = \| W_{\tilde{\theta}, \tilde{\lambda}} \|_{\mathrm{tr}}$ is bounded for all $(\theta, \lambda) \in \Xi(\tilde{\theta}, \tilde{\lambda})$, and that the canonical feature map $\| \phi_{\theta}(x) \|_{\mathcal{H}_{k_{\theta}}}^{2} = k_{\theta}(x, x) \leq \alpha^{2} = \sup_{x \in \mathcal{X}} k_{\tilde{\theta}}(x, x)$ is bounded in \gls{RKHS} norm for all $x \in \mathcal{X}, (\theta, \lambda) \in \Xi(\tilde{\theta}, \tilde{\lambda})$. By \cref{thm:specific_expected_loss_bound_for_multiclass_conditional_embedding}, we can now claim the following.
			
			\begin{lemma}[\gls{MCE} Expected Risk Bound for Hyperparameter Sets]
				\label{thm:general_expected_risk_bound_hyperparameter_learning}
				For any integer $n \in \mathbb{N}_{+}$ and any set of training observations $\{x_{i}, y_{i}\}_{i = 1}^{n}$, with probability $1 - \beta$ over \textit{iid} samples $\{X_{i}, Y_{i}\}_{i = 1}^{n}$ of length $n$ from $\mathbb{P}_{X Y}$, every $(\theta, \lambda) \in \Xi(\tilde{\theta}, \tilde{\lambda})$ satisfies
				\begin{equation}
				\begin{aligned}
					\mathbb{E}[\mathcal{L}_{e^{-1}}(Y, \bvec{f}_{\theta, \lambda}(X))] &\leq \frac{1}{n} \sum_{i = 1}^{n} \mathcal{L}_{\epsilon}(Y_{i}, \bvec{f}_{\theta, \lambda}(X_{i})) \\
					&+ 4 e \sqrt{\sup_{x \in \mathcal{X}} k_{\tilde{\theta}}(x, x)} \| W_{\tilde{\theta}, \tilde{\lambda}} \|_{\mathrm{tr}} + \sqrt{\frac{8}{n} \log{\frac{2}{\beta}}},
				\label{eq:general_expected_risk_bound_hyperparameter_learning}
				\end{aligned}
				\end{equation}
				for every $\epsilon \in (0, e^{-1})$, where 
				\begin{equation}
				\begin{aligned}
					\bvec{f}_{\theta, \lambda}(x) :=& \bvec{Y}^{T} (K_{\theta} + n \lambda I)^{-1} \bvec{k}_{\theta}(x), \\
					\| W_{\tilde{\theta}, \tilde{\lambda}} \|_{\mathrm{tr}} =& \sqrt{\mathrm{trace}\bigg(\bvec{Y}^{T} (K_{\tilde{\theta}} + n \tilde{\lambda} I)^{-1} K_{\tilde{\theta}} (K_{\tilde{\theta}} + n \tilde{\lambda} I)^{-1} \bvec{Y}\bigg)}.
				\end{aligned}
				\end{equation}
			\end{lemma}
	
			\begin{proof}
				We first apply \cref{thm:specific_expected_loss_bound_for_multiclass_conditional_embedding} with the choice of $\rho = \| W_{\tilde{\theta}, \tilde{\lambda}} \|_{\mathrm{tr}}$ and $\alpha^{2} = \sup_{x \in \mathcal{X}} k_{\tilde{\theta}}(x, x)$. The inequality \eqref{eq:specific_expected_loss_bound_for_multiclass_conditional_embedding} then only holds for a subset of kernel hyperparameters and regularizations $(\theta, \lambda) \in \Xi(\tilde{\theta}, \tilde{\lambda})$ as defined by \eqref{eq:hyperparameter_subset_definition}.
				\qed
			\end{proof}
		
			Since inequality \eqref{eq:general_expected_risk_bound_hyperparameter_learning} holds for any $(\theta, \lambda) \in \Xi(\tilde{\theta}, \tilde{\lambda})$ and we know that $(\tilde{\theta}, \tilde{\lambda}) \in \Xi(\tilde{\theta}, \tilde{\lambda})$, we choose $\theta = \tilde{\theta}$ and $\lambda = \tilde{\lambda}$.  We now arrive at our final result from which we can bound the expected risk for a specific choice of hyperparameters $\theta \in \Theta$ and $\lambda \in \Lambda$.
			
			\begin{theorem}[\gls{MCE} Expected Risk Bound for Hyperparameters]
				\label{thm:expected_risk_bound_hyperparameter_learning}
				For any integer $n \in \mathbb{N}_{+}$ and any set of training observations $\{x_{i}, y_{i}\}_{i = 1}^{n}$, with probability $1 - \beta$ over \textit{iid} samples $\{X_{i}, Y_{i}\}_{i = 1}^{n}$ of length $n$ from $\mathbb{P}_{X Y}$, every $\theta \in \Theta$ and $\lambda \in \Lambda$ satisfies
				\begin{equation}
					\mathbb{E}[\mathcal{L}_{e^{-1}}(Y, \bvec{f}_{\theta, \lambda}(X))] \leq \frac{1}{n} \sum_{i = 1}^{n} \mathcal{L}_{\epsilon}(Y_{i}, \bvec{f}_{\theta, \lambda}(X_{i})) + 4 e \; r(\theta, \lambda) + \sqrt{\frac{8}{n} \log{\frac{2}{\beta}}},
				\label{eq:expected_risk_bound_hyperparameter_learning}
				\end{equation}
				for every $\epsilon \in (0, e^{-1})$, where 
				\begin{equation}
				\begin{aligned}
					\bvec{f}_{\theta, \lambda}(x) &:= \bvec{Y}^{T} (K_{\theta} + n \lambda I)^{-1} \bvec{k}_{\theta}(x), \\
					r(\theta, \lambda) &:= \sqrt{\mathrm{trace}\bigg(\bvec{Y}^{T} (K_{\theta} + n \lambda I)^{-1} K_{\theta} (K_{\theta} + n \lambda I)^{-1} \bvec{Y}\bigg) \sup_{x \in \mathcal{X}} k_{\theta}(x, x)}.
				\end{aligned}
				\end{equation}
			\end{theorem}
			
			\begin{proof}
				We first apply \cref{thm:general_expected_risk_bound_hyperparameter_learning} with the choice of $\theta = \tilde{\theta}$ and $\lambda = \tilde{\lambda}$. We then replace the notation $\tilde{\theta} \rightarrow \theta$ and $\tilde{\lambda} \rightarrow \lambda$ back to avoid cluttered notation. Note that this should not be confused with the general $\theta$ and $\lambda$ from earlier theorems.
				\qed
			\end{proof}
	
	\newpage
	\section{Special Cases and Model Architectures}
	\label{app:special_cases}
	
		For \glspl{MCE}, the modelling lies in the choice of the kernel family $k_{\theta} : \mathcal{X} \times \mathcal{X} \to \mathbb{R}$ over the input space $\mathcal{X}$. The only requirement for the kernel $k$ is that it is symmetric and positive definite, and thus we may construct richer and more expressive kernel families in any way subject to such requirements. Once such a kernel family is constructed, the kernel hyperparameters $\theta$, as well as the regularization hyperparameter $\lambda$, can be learned effectively using algorithm 1.
		
		One way to construct richer and more expressive kernels is to compose them from simpler kernels. For example, we can construct new kernels through convex combinations or products of multiple simpler kernels \citep{genton2001classes}. Any new parameters, such as coefficients for linear combinations of simpler kernels, can be included into the kernel hyperparameters $\theta$ and learned in the same way as before. Alternatively, there may be domain specific structures or representations within the data that can be exploited. We can then construct the kernel family by incorporating such structural representations into the kernel. Even better, we can construct the kernel family so that it is capable or learning such structural representations by itself, by parameterizing such representations into the kernel.
		
		In this section, we focus on special cases of the \gls{MCE} where the kernel family is constructed through explicit feature maps. This construction allows the incorporation of trainable domain specific structures and enables scalability to larger datasets. We first begin by introducing the explicit \gls{MCE} in \cref{app:explicit_multiclass_conditional_embedding}, where explicit feature maps can be learned while enabling scalability to larger datasets. We then construct the \gls{CEN} in \cref{app:conditional_embedding_network}, where the kernel family is formed from multiple layers of learned representations before a simpler kernel encodes their similarity for inference. Finally, we marry both constructions into the explicit \gls{CEN} in \cref{app:explicit_conditional_embedding_network}, which provides a scalable and more applicable version of the deep \gls{CEN} by placing a linear kernel on the network features.
		
		In essence, we can categorise the \gls{MCE} using two properties: the model width and the model depth. The model width represents the dimensionality of the feature space used to construct the linear decision boundaries. The model depth represents the number of transformations used to map examples from the input space to the feature space. By implicitly defining a high dimensional feature space through simple transformations, typical nonlinear kernels produce classifiers that have a shallow but wide architecture. In contrast, the three \gls{MCE} variants to be introduced in this section form other combinations of model architecture in both depth and width. Of course, this characterization of architecture is not mutually exclusive. For example, a polynomial kernel can be seen as a nonlinear kernel where higher order polynomial features are implicitly defined, or as a linear kernel on explicit polynomial features. We summarize those architectures in \cref{tab:multiclass_conditional_embedding_variants}.
		
		\begin{table}[h]
			\caption{Properties of \gls{MCE} architectures}
			\label{tab:multiclass_conditional_embedding_variants}
			\centering
			\begin{tabular}{lccccc}
				\gls{MCE} Variant & Width & Depth & Scalability & Flexibility & Typical Datasets  \\
				\midrule
				Implicit \gls{MCE} & Wide & Shallow & Low & High & High or Low $d$, Low $n$ \\
				Explicit \gls{MCE} & Narrow & Shallow & High & Low & Low $d$, High $n$ \\
				Implicit \gls{CEN} &  Wide & Deep & Low & High & Structured $d$, Low $n$ \\
				Explicit \gls{CEN} & Narrow & Deep & High & High & Structured $d$, High $n$ \\ 
			\end{tabular}
		\end{table}
	
		\subsection{Explicit Multiclass Conditional Embedding}
		\label{app:explicit_multiclass_conditional_embedding}
	
			The advantage of using a kernel-based classifier is that the kernel $k$ allows us to express nonlinearities in a simple way. It does this by implicitly mapping the input space $\mathcal{X}$ to a high dimensional feature space $\mathcal{H}_{k}$ of non-linear basis functions such that decision boundaries become linear in that space. For many kernels, such as the Gaussian kernel defined over the Euclidean space, the feature space $\mathcal{H}_{k}$ has dimensionality that is uncountably infinite. Nevertheless, by virtue of the Representer Theorem \citep{kimeldorf1971some}, the resulting decision functions can be represented by a finite linear combination of kernels centered at the training data, and the \gls{MCE} is no exception. This elegant and convenient result enables exact inference to be performed while only requiring a finite kernel gram matrix of the size of the dataset ($n \times n$) to be computed. In this way, the capacity of the model grows with the size of the dataset, which makes kernel methods nonparametric and very flexible, as it can adapt to the complexity of a dataset even with relatively simple kernels. 
			
			However, this elegant property is also the very reason that prevents kernel-based methods from scaling to larger datasets, as the size of such a gram matrix grows very quickly by $O(n^{2})$. Many kernel-based methods also require the inversion of a regularized gram matrix, which has a time complexity of $O(n^{3})$, and cannot be easily parallelized like standard matrix multiplications. As such, inference on datasets beyond tens of thousands of observations quickly becomes impractical to perform with kernel-based techniques.
			
			In order to scale to big datasets, instead of placing a kernel over the input space directly and let it implicitly define the feature space, we explicitly define a finite dimensional feature space $\mathcal{Z} \subseteq \mathbb{R}^{p}$ of lower dimension $p$, where $p < n$, and place a linear kernel over it. That is, we specify a family of explicit features maps $\varphi_{\theta} : \mathcal{X} \to \mathcal{Z}$, and place a linear kernel on top of these explicit features,
			\begin{equation}
			k_{\theta}(x, x') = \varphi_{\theta}(x)^{T} \varphi_{\theta}(x').
			\end{equation}
			
			By explicitly defining a finite dimensional feature space, the matrix to be inverted during both learning and inference in the \gls{MCE} can be reduced from size $n \times n$ to size $p \times p$ by using the Woodbury matrix inversion identity \citep{higham2002accuracy}. We use this identity to modify algorithm 1 to \cref{alg:explicit_multiclass_conditional_embedding_training} to exploit this computational speed up.
			
			However, with a fixed and finite amount of feature basis, the model becomes parametric and its flexibility is compromised. In other words, the model is narrow in the number of feature representations. We therefore turn to multi-layered feature compositions, where the flexibility of a model comes from the deep architecture instead of implicit high dimensional features.
			
			\begin{algorithm}[tb]
				\caption{\gls{MCE} Hyperparameter Learning with Batch Stochastic Gradient Updates for Explicit Features}
				\label{alg:explicit_multiclass_conditional_embedding_training}
				\begin{algorithmic}[1]
					\STATE {\bfseries Input:} feature family $\varphi_{\theta} : \mathcal{X} \to \mathcal{Z} \subseteq \mathbb{R}^{p}$, dataset $\{x_{i}, y_{i}\}_{i = 1}^{n}$, feature parameters $\theta_{0}$, regularization hyperparameters $\lambda_{0}$, learning rate $\eta$, batch size $n_{b}$
					\STATE $\theta \leftarrow \theta_{0}$, $\lambda \leftarrow \lambda_{0}$
					\REPEAT
					\STATE Sample the next batch $\mathcal{I}_{b} \subseteq \mathbb{N}_{n}$, s.t. $| \mathcal{I}_{b} | = n_{b}$ \hspace{\fill} 
					\STATE $Y \leftarrow \{\delta(y_{i}, c) : i \in \mathcal{I}_{b}, c \in \mathbb{N}_{m}\} \hspace{\fill} \in \{0, 1\}^{n_{b} \times m}$
					\STATE $Z_{\theta} \leftarrow \{\varphi_{\theta}(x_{i}) : i \in \mathcal{I}_{b}\} \hspace{\fill} \in \mathbb{R}^{n_{b} \times p}$
					\STATE $L_{\theta, \lambda} \leftarrow \mathrm{cholesky}(Z_{\theta}^{T} Z_{\theta} + n_{b} \lambda I_{p}) \hspace{\fill} \in \mathbb{R}^{p \times p}$
					\STATE $W_{\theta, \lambda} \leftarrow L_{\theta, \lambda}^{T} \backslash (L_{\theta, \lambda} \backslash Z_{\theta}^{T} Y) \hspace{\fill} \in \mathbb{R}^{p \times m}$
					\STATE $P_{\theta, \lambda} \leftarrow Z_{\theta} W_{\theta, \lambda} \hspace{\fill} \in \mathbb{R}^{n_{b} \times m}$
					\STATE $r(\theta, \lambda) = \alpha(\theta) \sqrt{\sum_{c = 1}^{m} \sum_{j = 1}^{p} (W_{\theta, \lambda})_{j, c}^{2}}$
					\STATE $q(\theta, \lambda) \leftarrow \frac{1}{n_{b}} \sum_{i = 1}^{n_{b}} \mathcal{L}_{\epsilon}((Y)_{i}, (P_{\theta, \lambda})_{i}) + 4 e \; r(\theta, \lambda)$
					\STATE $(\theta, \lambda) \leftarrow \mathrm{GradientBasedUpdate}(q, \theta, \lambda; \eta)$ 
					\UNTIL{maximum iterations reached} 
					\STATE {\bfseries Output:} kernel hyperparameters $\theta$, regularization hyperparameter $\lambda$
				\end{algorithmic}
			\end{algorithm}
	
		\subsection{Conditional Embedding Network}
		\label{app:conditional_embedding_network}
	
			For many application domains, there are natural structures in the data. For example, in image recognition, pixel dimensions are spatially correlated: nearby pixels are more related, and ordering between the pixel dimensions matter. One would expect convolutional features \citep{lecun1998gradient} to be natural in this domain, and provide a performance boost to our classifier should it be included. In this way, we can often benefit by including domain specific structures and features into our model.
			
			In this section, we focus on constructing kernels for which inputs $x, x' \in \mathcal{X}$ is to undergo various stages of feature transformations before such it is passed into a simpler kernel $\kappa$ that captures the similarity between the representations. Specifically, we pay particular attention to feature transformations in the form of a perceptron, so that the cumulative stages of feature transformation become the (feed-forward) multi-layer perceptron that is familiar within the neural network literature.
			
			Formally, let $\mathcal{F}_{0} := \mathcal{X}$ be the original input space. The $j^{\mathrm{th}}$ layer of the network $\varphi^{(j)}_{\theta_{j}} : \mathcal{F}_{j - 1} \to \mathcal{F}_{j}, j = 1, 2, \dots, L$ is to transform features from the previous layer to features in the current layer, where $L$ is the total number of such feature transformation layers, and $\theta_{j} \in \Theta_{j}$ parametrizes each of those transformations.
			
			For example, in a typical multi-layer perceptron context, each layer can be written as $\varphi^{(j)}_{\theta_{j}}(x) = \sigma(W_{j} x + b_{j})$, where $W_{j}$ and $b_{j}$ are the weight and bias parameters of the layer, and $\sigma$ is an element-wise activation function, typically the rectified linear unit (ReLU) or the sigmoid. In this case, the layer is parametrized by $\theta_{j} = \{W_{j}, b_{j}\}$.
			
			Let $\kappa_{\theta_{0}} : \mathcal{F}_{p} \times \mathcal{F}_{p} \to \mathbb{R}$ be parametrized by $\theta_{0} \in \Theta_{0}$. We will construct our kernel network $k$ by
			\begin{equation}
			\begin{aligned}
				k_{\theta}(x, x') := \kappa_{\theta_{0}}\Bigg(& \varphi^{(L)}_{\theta_{L}}\bigg(\varphi^{(L - 1)}_{\theta_{L - 1}}\Big(\dots\varphi^{(2)}_{\theta_{2}}\big(\varphi^{(1)}_{\theta_{1}}(x)\big)\Big)\bigg), \\
				&\varphi^{(L)}_{\theta_{L}}\bigg(\varphi^{(L - 1)}_{\theta_{L - 1}}\Big(\dots\varphi^{(2)}_{\theta_{2}}\big(\varphi^{(1)}_{\theta_{1}}(x')\big)\Big)\bigg) \Bigg),
			\label{eq:deep_conditional_embedding_network}
			\end{aligned}
			\end{equation}
			where $\theta = (\theta_{1}, \theta_{2}, \dots, \theta_{L -1}, \theta_{L}, \theta_{0}) \in \Theta = \Theta_{1} \otimes \Theta_{2} \otimes \dots \otimes \Theta_{L - 1} \otimes \Theta_{L} \otimes \Theta_{0}$ are the collection of all parameters of each layer and the kernel $\kappa$.
			
			In order to train the multi-layered representations in an end-to-end fashion, we employ algorithm 1. With a deep architecture, the feature representations the \gls{CEN} can learn are very flexible, and can work very well for structured data by employing suitable network architectures.
			
			If we choose to employ nonlinear kernels $\kappa$, the model architecture is also wide in that an even higher dimensional feature space is implicitly defined on top of the feature space of the last network layer. Despite its supreme flexibility, this again prevents the model from being scalable. We therefore turn to the specific case where we employ a linear kernel $\kappa$ on top of the multi-layered features.
			
		\subsection{Explicit Conditional Embedding Network}
		\label{app:explicit_conditional_embedding_network}
			
			The explicit \gls{CEN} is simply a special case at the intersection of the explicit \gls{MCE} and the \gls{CEN}. From the explicit \gls{MCE} perspective, we simply choose the feature map $\varphi_{\theta}(x) = \varphi^{(L)}_{\theta_{L}}\big(\varphi^{(L - 1)}_{\theta_{L - 1}}(\dots\varphi^{(2)}_{\theta_{2}}(\varphi^{(1)}_{\theta_{1}}(x)))\big)$. From the \gls{CEN} perspective, we simply choose $\kappa(z, z') = z^{T} z'$ to be a linear kernel.
			
			This model architecture is a very practical and powerful form of the \gls{MCE}. By having a deep architecture, the classifier is still capable of learning flexible representations on structured data, while being able to scale to larger datasets due to the linear kernel at the output layer, provided that the dimensionality of the last layer is relatively small compared to the size of the dataset.
			
			As a subclass of explicit \gls{MCE}, we can employ \cref{alg:explicit_multiclass_conditional_embedding_training} to learn the multi-layered features effectively. In fact, by not mapping the multi-layered features into a nonlinear kernel, the gradients for each network weight and bias are usually more pronounced, and learning is usually faster in comparison. This approach was used to train the neural network features in our experiments.

\end{document}